\documentclass[twoside,11pt]{article}

\usepackage[preprint]{jmlr2e}

\usepackage{newtxmath} 

\usepackage{tikz}
\usetikzlibrary{positioning, fit, arrows.meta}
\usepackage{bm}

\usepackage[ruled,vlined]{algorithm2e}

\usepackage{booktabs}

\usepackage{subcaption}
\usepackage{amsmath}

\usepackage{float}
\usepackage{xurl}

\newcommand{\mby}{\mathbf{y}}
\newcommand{\mbz}{\mathbf{z}}

\newcommand{\mbw}{\mathbf{w}}

\newcommand{\mbalpha}{\bm{\alpha}}
\newcommand{\mbgamma}{\bm{\gamma}}
\newcommand{\mbtheta}{\bm{\theta}}
\newcommand{\mbbeta}{\bm{\beta}}
\newcommand{\mbeta}{\bm{\eta}}
\newcommand{\mbnu}{\bm{\nu}}
\newcommand{\mbmu}{\bm{\mu}}
\newcommand{\mbsigma}{\bm{\sigma}}
\newcommand{\mbtau}{\bm{\tau}}

\newcommand{\mbphi}{\bm{\phi}}

\newcommand{\s}{\,;\,}
\newcommand{\rmp}{\mathrm{p}}
\newcommand{\E}[1]{\mathbb{E}\left[#1\right]}
\newcommand{\EE}[2]{\mathbb{E}_{#1}\left[#2\right]}

\DeclareMathOperator{\rank}{rank}

\newtheorem{assumption}{Assumption}

\usepackage{lastpage}
\jmlrheading{23}{2026}{1-\pageref{LastPage}}{1/21; Revised 5/22}{9/22}{21-0000}{Hector Rodriguez-Deniz and David M. Blei}

\ShortHeadings{Exponential Family Synthetic Controls}{Rodriguez-Deniz and Blei}
\firstpageno{1}

\begin{document}

\title{Exponential Family Synthetic Controls}
\author{\name Hector Rodriguez-Deniz \email hector.rodriguez@columbia.edu \\
       \addr Data Science Institute\\
       Columbia University\\
       New York, NY 10027, USA
       \AND
       \name David M. Blei \email david.blei@columbia.edu \\
       \addr Department of Computer Science and Department of Statistics\\
       Columbia University\\
       New York, NY 10027, USA}
\editor{My editor}
\maketitle

\begin{abstract}
We develop \emph{exponential family synthetic controls} (EFSC), a distributional version of synthetic controls for a panel of datasets. Each cell of the panel corresponds to a dataset drawn from an exponential family whose natural parameters factorize probabilistically across units and times. We estimate the latent factors using black-box variational inference. This replaces the usual weighted-average view of synthetic controls with a flexible probabilistic model that operates on full distributions. We propose causal estimands based on divergences between pre- and post-intervention distributions induced by the posterior of the natural parameters, together with distributional placebo tests to support causal inference and assess the significance of the estimated effects. We validate the proposed framework on synthetic and real data. Across a variety of exponential-family distributions, EFSC accurately recovers causal effects induced by exponential tilts, together with the corresponding divergences between treated and counterfactual distributions. The framework also captures effects induced by structural perturbations of the latent factors and by heavy-tailed noise contamination. Finally, we apply EFSC to study the expansion of Medicaid under the Affordable Care Act (ACA) and its impact on the distribution of health insurance coverage across U.S. states. Code is available at \url{https://github.com/blei-lab/efsc}.
\end{abstract}


\section{Introduction}

The method of \emph{synthetic controls} (SC) uses panel data to estimate the causal effect of an intervention. Consider the study of California's tobacco-control program in \citet{abadie2010synthetic}. The rows of the panel are states, the columns are years, and each cell contains annual per-capita cigarette sales. Beginning in 1989, California implemented Proposition~99, which raised its tax on cigarettes. SC helps answer the question: How did this policy change cigarette consumption in California?

The dataset contains California's cigarette sales under Proposition~99, but it does not contain the sales that would have occurred without it. The idea behind SC is to estimate this missing counterfactual from the other states. Specifically, it models each year of California's (untreated) cigarette sales as a weighted average of the other states' sales. It then uses this fitted model to provide a ``synthetic counterfactual California'' during the treated years.  The difference between the observed sales and synthetic sales estimates the policy's effect.

SC began as a method for comparative policy evaluation, with applications to conflict in the Basque Country and tobacco policy in California \citep{abadie2003economic,abadie2010synthetic}. It is now used across economics and the social sciences to study interventions that affect one or a few aggregate units~\citep{abadie2021using}.

In its usual form, SC is designed to analyze a panel with one measurement in each cell, such as the cigarette sales of each state in each year. In this paper, we consider how to implement SC analysis on panels of \textit{datasets}. Rather than observing aggregated cigarette sales in each cell of the panel, suppose we observed a sample of individuals from that state and year. These data might include how many cigarettes they bought as well as demographic covariates.  (Indeed, many SC applications are attached to summaries of such data.)

With a panel of datasets, we assume that each cell contains a sample from a \textit{distribution} of observations, e.g., a distribution of cigarette sales. For California's cells after 1989, those distributions are formed under treatment, i.e., the increase in the tax. For the other cells---in California before the tax and in other states throughout the period---these distributions are under the control, i.e., no tax increase. We now ask: What is the effect of the tax increase on the distribution of cigarette sales in California?  Figure~\ref{fig:data_panel_and_tilt} illustrates this setting.

We develop a new SC method to answer this question. In general, for unit $i$ and time $j$, the cell $(i,j)$ contains a dataset $\mby_{ij}=\{y_{ijk}\}_{k=1}^{m_{ij}}$, a sample from a cell-specific distribution. In the treated post-intervention cells, we observe the distribution under treatment; in the other cells we observe the distribution under the untreated distributions. The causal question is how the treatment changed the distribution.

To solve the problem, we develop \emph{exponential family synthetic controls} (EFSC). EFSC models each cell as a sample from an exponential-family distribution with natural parameter $\eta_{ij}$. In a treated post-intervention cell, the observations inform the treated parameter $\eta_{ij}^{\mathrm{treat}}$. The corresponding untreated parameter $\eta_{ij}^{\mathrm{ctrl}}$ is missing.

EFSC takes the matrix completion view of synthetic controls~\citep{athey2021matrix}, where the goal is to complete the matrix of untreated natural parameters. Specifically, EFSC completes the matrix through a latent factor model of the distributions. We assume the untreated natural parameter can be written as a linear function of per-row and per-column latent variables,
\begin{align}
  \eta_{ij}^{\mathrm{ctrl}}
  =
  \alpha_i+\gamma_j+\mbtheta_i^\top\mbbeta_j.
\end{align}
First, the data in the untreated cells help us estimate these per-row and per-column variables. Then, their fitted values determine the missing $\eta_{ij}^{\mathrm{ctrl}}$ for each treated unit and post-treatment time.

To fit the model, we take a fully probabilistic approach where we place Gaussian priors on the variables in the factorization and then approximate the posterior. Specifically, we approximate it with a mean-field Gaussian variational distribution and fit the variational parameters with black-box variational inference \citep{ranganath2014black,blei2017variational}. The fitted variational distribution induces an approximate posterior distribution over each missing counterfactual natural parameter.

With this setup, the rest of the paper proceeds as follows. Section~2 situates this work in the related literature. Section~3 develops the exponential family synthetic controls method. It defines the causal estimands, develops the latent factor model, and develops its variational inference algorithm. It further establishes the assumptions that support a causal interpretation, and it constructs distributional placebo tests to help assess and criticize an EFSC model.

With the method in hand, Section~4 studies EFSC on simulated and real datasets. In simulation, we confirm that EFSC can recover true causal effects with different types of interventions on the distribution. On real data, we evaluate the Medicaid expansion under the Affordable Care Act \citep{patientprotection2010}. In this study, each state-year cell contains the distribution of low-income adults across five insurance categories. EFSC estimates the causal effect of expanding Medicaid on the distribution of insurance coverage. Section~5 concludes the paper.


\section{Related Work}
\begin{figure}[t]
    \centering
    \includegraphics[width=0.85\linewidth]{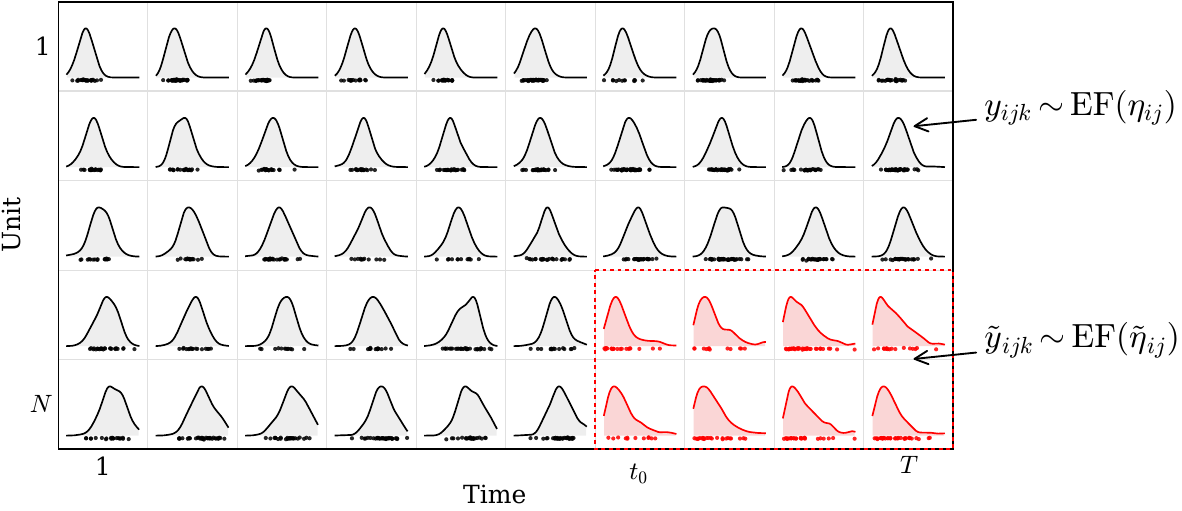}
    \caption{Panel of datasets with $N$ units and $T$ time steps; each cell $(i,j)$ contains a sample of size $m_{ij}$, drawn from a cell-specific exponential-family distribution with natural parameter $\eta_{ij}$. The data in the last two units are affected by a policy that perturbs the generating process after time $t_0$.}\label{fig:data_panel_and_tilt}
\end{figure}
\textbf{Synthetic controls on disaggregated data.}
Recent contributions on synthetic controls have focused on leveraging disaggregated observations within panel data. \citet{abadie2021penalized} propose a penalized estimator that operates at the unit level with the objective of reducing interpolation bias. \citet{shi2022assumptions} reframe potential outcomes in a fine-grained fashion to evaluate some of the assumptions underlying SC. In particular, they analyze identification under linear factor models, linking them to invariance principles, and provide guidance for donor and covariate selection. Also building on this framework, \citet{nazaret2024misspecification} assess the robustness of the SC estimator to misspecification of the linear assumption and derive corresponding error bounds. \cite{rho2025cluster} tackle the problem of disaggregated observations by selecting a subset of donors via clustering methods, which allows them to exploit fine-grained information while controlling the size of the effective donor set and reducing the variance. Although not directly related to SC, recent developments on hierarchical causal models such as \cite{weinstein2026hierarchical} also highlight the importance of modeling unit-level heterogeneity in panel data. In contrast to previous approaches based on e.g., regularization or clustering, our EFSC summarizes the distributional properties of the disaggregated observations through the corresponding sufficient statistics within the exponential family. We then use a latent factorization to model the natural parameters, thereby modeling the data-generating process directly.

\textbf{Flexible low-dimensional representations.}
A parallel line of work introduces additional model flexibility to SC through latent and dynamic formulations. These include robust and Bayesian approaches based on matrix estimation and ensemble methods such as in \citet{viviano2023synthetic}, Bayesian structural time-series models \citep{brodersen2015inferring}, and state-space extensions with time-varying coefficients or latent dynamics \citep{shao2022generalized, klinenberg2024timevarying,rho2024timeaware}. A closely related perspective frames synthetic control as a matrix completion problem, connecting causal panel models to low-rank methods in machine learning \citep{athey2021matrix}. In this view, counterfactual estimation is treated as imputing missing potential outcomes under a structured low-rank assumption. Probabilistic matrix factorization \citep{salakhutdinov2007pmf} provides a canonical probabilistic formulation of this idea that we exploit in our EFSC. Our work extends this perspective to a distributional setting, by factorizing the data-generating process itself and leveraging exponential-family structure to move beyond scalar observations.

\textbf{Stochastic and structural interventions.}
Stochastic interventions define causal effects through transformations of the treatment distribution rather than deterministic assignments \citep{diaz2012population, kennedy2019nonparametric}. A common instance is an exponential-tilt intervention, where the treatment density is modified via a multiplicative exponential factor \citep{diaz2019causal, schindl2024incremental, jetsupphasuk2025difference}. We consider stochastic interventions based on exponential tilting, which arise naturally within our exponential-family representation. We further evaluate our EFSC model under structured interventions acting on latent representations of units and time, capturing heterogeneity driven by unobserved factors. These induce nonlinear perturbations in latent space.

\section{Exponential Family Synthetic Controls}

We begin by describing the problem of estimating distributional causal effects in panels of datasets and the latent factor model underlying EFSC\@. We then present a black-box variational inference procedure for posterior estimation of the model parameters. Finally, we discuss the causal assumptions and interpretation of the model, and introduce several classes of interventions that motivate our empirical studies.

\subsection{Distributional Interventions in Panels of Datasets}\label{subsec:method_overview}

We observe a panel of datasets $\{\mby_{ij}\}$ for units $i=1,\dots,N$, and times $j=1,\dots,T$. Each dataset $\mby_{ij} = \{y_{ijk}\}_{k=1}^{m_{ij}}$ is an i.i.d. sample from an EF with natural parameter $\eta_{ij}$:
\begin{align}
  y_{ijk}|\eta_{ij} \stackrel{\text{iid}}{\sim} \textsc{expfam}(\eta_{ij}),\;k=1,\ldots,m_{ij}, \\ 
  \log \rmp(y_{ijk} \s \eta_{ij}) = \eta_{ij}t(y_{ijk}) - a(\eta_{ij}) + c(y_{ijk}),
  \end{align}
where $t(y)$ is the sufficient statistic function, $a(\eta)$ is the log-partition function and $c(y)$ the log-carrier term. 

Assume that a policy has been implemented in unit $i$ at time $T$, and that this policy induces a change in the data-generating distribution. For example, the policy might tilt the exponential family, so that $y_{iTk}\mid\tilde{\eta}_{iT}\stackrel{\mathrm{iid}}{\sim}\textsc{expfam}(\tilde{\eta}_{iT}),\;k=1,\ldots,m_{iT}$, and $\tilde{\eta}_{iT} = \eta_{iT} + \tau$.  Here the untilted $\eta_{iT}$ is the counterfactual parameter that would have generated $\mby_{iT}$ had the policy not been implemented. For simplicity, we initially focus on a single treated unit $i$ and an intervention at the final time period $j=T$, although the framework naturally extends to multiple treated units and post-treatment periods. Figure~\ref{fig:data_panel_and_tilt} illustrates this more general setting.

Our goal is to estimate the causal effect of the policy on the distribution of the data. Let $\eta^{\text{treat}}_{iT}$ denote the treated parameter and $\eta^{\text{ctrl}}_{iT}$ the untreated parameter. We can define the posterior expected effect as
\begin{align}
  \text{ECE}_{iT}\triangleq\E{\eta^{\text{treat}}_{iT} - \eta^{\text{ctrl}}_{iT} \mid \mby},\label{eq:causal_effect_posterior_diff}
\end{align}
or, more generally, a divergence between their induced distributions,
\begin{align}
  \text{ECD}_{iT}\triangleq\E{\mathrm{KL}\big(p(y \s \eta^{\text{treat}}_{iT})
      \,\|\,p(y \s \eta^{\text{ctrl}}_{iT})\big)\,\big|\,\mby},\label{eq:causal_effect_posterior_kl}
\end{align}
where $\text{KL}(p\,\|\,q)$ is the Kullback-Leibler divergence from $p$ to $q$. The challenge is that, while we observe data from $\eta^{\text{treat}}_{iT}$, we do not have any observations from $\eta^{\text{ctrl}}_{iT}$.

\subsection{Counterfactual Modeling via Latent Factorization}

To estimate the missing counterfactual distributions, we adapt the matrix completion SC approach from \cite{athey2021matrix}. We frame this problem as
imputing the corresponding missing entries in the matrix of untreated natural parameters. Since we have a panel of datasets, EFSC poses a matrix factorization over the parameters of the EF rather than the data itself. Specifically, we model the natural parameters as
\begin{align}
  \eta_{ij} = \alpha_i + \gamma_j + \mbtheta_i^\top \mbbeta_j,\label{eq:pmf}
\end{align}
with unit effects $\alpha_i$, time effects $\gamma_j$, and latent factors $\mbtheta_i, \mbbeta_j \in \mathbb{R}^r$, such that
\begin{align}
      \mbtheta &= [\mbtheta_1,\dots,\mbtheta_N],\;\text{and}\;
  \mbbeta = [\mbbeta_1,\dots,\mbbeta_T].
\end{align}
For families with constrained natural parameter spaces $\mathcal{H}\subsetneq\mathbb{R}$, appropriate transformations are applied to ensure that $\eta_{ij}\in\mathcal H$; see Appendix~\ref{app:multi-pmf}.

Through this probabilistic matrix factorization, EFSC factorizes the matrix of exponential-family natural parameters. Given a panel of datasets, posterior inference yields estimates of the latent factorization that characterizes the underlying data-generating process. Through their shared dependence on the latent factors, observations from the untreated cells of the panel provide information about the missing counterfactual natural parameters in the treated post-treatment cells.

Let $\Theta=(\boldsymbol{\alpha},\boldsymbol{\gamma},\mbtheta,\mbbeta)$ collect the parameters of the latent factorization, and let
$\mby^{\mathrm{ctrl}}$ denote the observations from cells not exposed to the intervention. The posterior over $\Theta$ induces a posterior distribution over the missing counterfactual natural parameter through
\begin{align}
p(\eta_{iT}^{\mathrm{ctrl}}\mid \mby^{\mathrm{ctrl}})
&=\int p(\eta_{iT}^{\mathrm{ctrl}}\mid \Theta)
p(\Theta\mid\mby^{\mathrm{ctrl}})
\,\mathrm d\Theta,\label{eq:induced_posterior}\\
\E{\eta_{iT}^{\mathrm{ctrl}}\mid\mby^{\mathrm{ctrl}}}
&=\EE{p(\Theta\mid\mby^{\mathrm{ctrl}})}
{\alpha_i+\gamma_T+\mbtheta_i^\top\mbbeta_T},
\end{align}
where $p(\eta_{iT}^{\mathrm{ctrl}}\mid\Theta)$ is concentrated at the value implied by the factorization in Equation~\eqref{eq:pmf}, with the appropriate constraint transformation when required. The resulting posterior distributions of the treated and counterfactual natural parameters are then used to evaluate the estimands in Equations~\eqref{eq:causal_effect_posterior_diff}--\eqref{eq:causal_effect_posterior_kl}.

\begin{figure}
    \centering
    \includegraphics[width=0.5\linewidth]{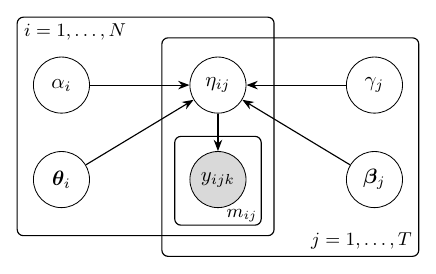}
    \caption{EFSC factorizes the natural parameters driving the data-generating process of entire datasets at each cell of the panel, allowing causal inference to be performed on full distributions rather than aggregated outcomes. }\label{fig:dag-efsc}
\end{figure}

To contrast with classical synthetic controls, consider an additive intervention on the outcome scale, $\tilde{Y}_{ijk}=Y_{ijk}+\delta$, where the treatment effect is recovered by estimating the missing counterfactual $Y_{ijk}$. EFSC adopts a more general distributional approach. The observations in each treated post-treatment cell inform the treated natural parameter $\eta_{ij}^{\mathrm{treat}}$, while EFSC imputes the missing counterfactual natural parameter $\eta_{ij}^{\mathrm{ctrl}}$ through the untreated latent factorization.

\subsection{Variational Inference}\label{subsec:bbvi}

The posterior distribution over the latent variables is generally analytically intractable due to the hierarchical latent factor model and general exponential-family likelihood.  We approximate the posterior distribution of the PMF parameters using variational inference (VI; \citeauthor{blei2017variational}, \citeyear{blei2017variational}).

VI approximates an intractable posterior by optimizing the parameters of a family of approximations. Let $q(\Theta;\mbnu)$ denote a variational approximation to $p(\Theta\mid\mby)$. We estimate the variational parameters $\mbnu$ by maximizing the evidence lower bound (ELBO),
\begin{equation}
    \mathcal{L}(\mbnu)=\EE{q}{\log p(\mby,\Theta)-\log q(\Theta;\mbnu)},
\end{equation}
where $\Theta$ is the vector of all $D=(N+T)(r+1)$ parameters defining $\mbeta$. We define a factorized (mean-field) variational approximation of $\Theta$ to model all the latent parameters defining the matrix factorization of $\mbeta$: 
\begin{equation}
    q(\Theta;\mbnu)=\prod_{l=1}^Dq_l(\Theta_l;\nu_l),    
\end{equation}
where each component $q_l$ is a univariate Gaussian so $\nu_l=\{\mu_l,\sigma^2_l\}$ for each $l=1,\ldots,D$. For simplicity, we assume that a one-parameter EF generates the data $y_{ijk}$ but our procedure directly extends to multi-parameter families, see Appendix~\ref{app:multi-pmf}. 

Assume a multivariate Gaussian prior on $\Theta$ with independent standard components. The log probabilities of the prior and variational distributions are:
\begin{align}
    \log p(\Theta)&= -\frac{1}{2}\sum_{l=1}^D\Theta_l^2\;+\mathrm{const.},\;
    \log q(\Theta;\mbnu)= -\frac{1}{2}\sum_{l=1}^D\left[\log\sigma^2_l+(\Theta_l-\mu_l)^2/\sigma^2_l\right]\;+\mathrm{const}.    
\end{align}
Let $\Omega$ denote the set of panel cells used in a given model fit. For the log-joint, we assume that the cell datasets are conditionally independent given $\Theta$ and the cell sizes, with  observations $y_{ijk}$ i.i.d. from an EF with parameter $\eta_{ij}$ within each cell. Let $T_{ij}=\sum_{k=1}^{m_{ij}}t(y_{ijk})$ be the sum of the sufficient statistics of the $(i,j)$-th cell of the panel. The log-joint is 
\begin{equation}
  \log p(\mby_{\Omega},\Theta)=\log p(\Theta)+\sum_{(i,j)\in\Omega}\left[\eta_{ij}T_{ij}-m_{ij}a(\eta_{ij})\right]\;+\mathrm{const}.    \end{equation}
The composition of $\Omega$ depends on the inferential task; for example, counterfactual estimation excludes treated post-intervention cells, while the placebo procedure uses different conditioning sets across its model fits. In Algorithm~\ref{alg:bbvi-efsc} we describe a generic black-box variational inference (BBVI; \citeauthor{ranganath2014black}, \citeyear{ranganath2014black}) procedure for estimating the model parameters; implementation details are provided in Appendix~\ref{app:bbvi_implementation}. An empirical Bayes extension that jointly learns the prior hyperparameters and variational parameters is in Appendix~\ref{app:eb}.

\begin{algorithm}[t]
\SetAlgoNoLine
\DontPrintSemicolon
\caption{BBVI for EFSC}\label{alg:bbvi-efsc}
\KwIn{Data $\mby$, log-joint of the model $\log p(\mby,\Theta)$, log-variational distribution $\log q(\Theta;\mbnu)$, number of Monte Carlo (MC) samples $S$}
\KwOut{Fitted variational parameters $\mbnu=\{\mbmu_{\nu},\mbsigma^2_{\nu}\}$}
 \textbf{Initialize} parameters $\mbnu$ randomly, set an adaptive step-size schedule $\rho_t$\;
 \Repeat{\textnormal{the ELBO converges or another stopping criterion is met}}{
 $L=\text{diag}(\mbsigma_{\nu})$\;
Draw $S$ samples from $q$ by i) $\mbz^{(s)}\sim\mathcal{N}(0,I_D)$, and ii) $\Theta^{(s)}=\mbmu_{\nu}+L\mbz^{(s)}$\;    
MC-approximate the $\text{ELBO}:\mathcal{L}(\mbnu)\approx S^{-1}\sum_{s=1}^S (\log p(\mby,\Theta^{(s)})-\log q(\Theta^{(s)};\mbnu))$\;   Compute the reparameterization gradient $\hat{\nabla}_{\mbnu}$ using automatic differentiation\;
   Update the variational parameters: $\mbnu=\mbnu+\rho_t\hat{\nabla}_{\mbnu}$\; 
 }
\end{algorithm}

\subsection{Causal Framework and Assumptions}
To interpret the ECE and ECD in
Equations~\eqref{eq:causal_effect_posterior_diff}--\eqref{eq:causal_effect_posterior_kl}
as causal effects, we connect the treated and counterfactual exponential-family
parameters to a distributional potential outcomes framework \citep{rubin2005causal} and state the assumptions
required for this interpretation, following the conventional assumptions in synthetic control methods \citep{abadie2021using}. We also establish a population-based identification result for the counterfactual natural parameters and, consequently, for the population targets underlying the ECE and ECD.

We begin with some notation. Let $\{\mathcal T,\mathcal C\}$ be a partition of $\{1,\dots,N\}$, denoting the sets of treated and control units,
respectively. Similarly, let $\mathcal P,\mathcal Q\subseteq\{1,\dots,T\}$ denote the pre- and
post-treatment periods, such that $\mathcal P=\{1,\dots,t_0-1\}$ and
$\mathcal Q=\{t_0,\dots,T\}$, for an intervention time $t_0\leq T$.
The observed untreated region of the panel is $\Omega_{\text{unt}}=\bigl(\mathcal{C}\times\{1,\ldots,T\}\bigr)
\cup\bigl(\mathcal{T}\times\mathcal{P}\bigr),$ and the treated--post-treatment target region is $\Omega_{\text{tgt}}=\mathcal{T}\times\mathcal{Q}$.

For each unit, let $A_i\in\{1,\ldots,T\}\cup\{\infty\}$ denote its realized intervention-adoption time, where $A_i=\infty$ denotes no adoption during the study window; for treated units, $A_i=t_0$. Let
$\mathbf A=(A_1,\ldots,A_N)$, and let $Y_{ijk}(\mathbf a)$ denote the potential
outcome under a hypothetical vector $\mathbf a\in(\{1,\ldots,T\}\cup\{\infty\})^N$ of adoption times. For
brevity, write $Y_{ijk}(a)\triangleq Y_{ijk}(a,\mathbf A_{-i})$ when unit $i$ adopts at time $a$ and the other units retain their realized adoption times. 

\begin{definition}[Untreated and treated potential-outcome distributions]
Let $\mathbf a^\infty=(\infty,\ldots,\infty)$ denote the all-never-treated
adoption vector. For every cell $(i,j)$, define the untreated potential outcome
by
\begin{align}
Y_{ijk}\triangleq Y_{ijk}(\mathbf a^\infty),
\end{align}
and let $P_{ij}$ denote its distribution: $Y_{ijk}\sim P_{ij}.$

For a treated post-treatment cell $(i,j)\in\Omega_{\mathrm{tgt}}$, define
\begin{align}
\tilde Y_{ijk}
\triangleq
Y_{ijk}(A_i)
=
Y_{ijk}(t_0),
\end{align}
and let $\tilde P_{ij}$ denote its distribution: $\tilde Y_{ijk}\sim\tilde P_{ij}.$
\end{definition}

Under the exponential-family observation model, these distributions are indexed by the natural
parameters $\mbeta_{ij},\tilde{\mbeta}_{ij}\in\mathcal H\subseteq\mathbb R^P$: $P_{ij}
=p(\,\cdot\,;\mbeta_{ij})$, and $\tilde P_{ij}=p(\,\cdot\,;\tilde{\mbeta}_{ij}).$

\begin{assumption}[Well-defined intervention and no interference; SUTVA]
\label{ass:sutva}
The potential outcomes $Y_{ijk}(\mathbf a)$ are well defined for every
adoption-time vector $\mathbf a$, with a single relevant version of the
intervention. Moreover, for any two adoption-time vectors $\mathbf a$ and
$\mathbf a'$, if $a_i=a_i'$, then
\begin{align*}
    Y_{ijk}(\mathbf a)\overset{d}{=}Y_{ijk}(\mathbf a').
\end{align*}
Thus, unit $i$'s potential-outcome distribution depends on the assignment
vector only through its own adoption time, which justifies the reduced
notation $Y_{ijk}(a_i)$ used below.
\end{assumption}
This is a distributional version of the stable unit treatment value
assumption (SUTVA): EFSC requires a well-defined intervention and rules out
spillover effects on each unit's marginal potential-outcome distribution
\citep{rubin2005causal,wager2024causal}.

\begin{assumption}[Consistency]
\label{ass:consistency}
For every observed cell, the observed random variable $Y_{ijk}^{\text{obs}}$ equals the potential
outcome corresponding to the unit's realized adoption time. Thus,
\begin{align*}
    Y_{ijk}^{\text{obs}}=Y_{ijk}(\mathbf A)=Y_{ijk}(A_i).
\end{align*}
\end{assumption}

Consistency links potential outcomes to the observed panel
\citep{hernan2020causal}.

\begin{assumption}[No anticipation]
\label{ass:noanticipation}
For every treated unit $i\in\mathcal{T}$, every pretreatment period
$j\in\mathcal{P}$, and every observation $k=1,\ldots,m_{ij}$, future exposure to the
intervention does not alter the pretreatment distribution relative to the
untreated condition. Thus,
\begin{align}
    Y_{ijk}(A_i)\overset{d}{=}Y_{ijk}(\infty),
\quad
(i,j)\in\mathcal{T}\times\mathcal{P}.
\end{align}
By Assumption~\ref{ass:sutva} and the definition of $P_{ij}$, both sides
therefore have distribution $P_{ij}$.
\end{assumption}

Together with consistency, this assumption makes the treated units'
pretreatment cells valid observations of the untreated data-generating
process \citep{abadie2021using}.

\begin{assumption}[Valid control units]
\label{ass:controls}
Every unit in $\mathcal{C}$ has $A_i=\infty$ and is not exposed to a substantively
equivalent version of the intervention during the study window. Thus, its
realized adoption condition is the untreated condition throughout the study.
\end{assumption}

Together, Assumptions~\ref{ass:sutva}--\ref{ass:controls} imply the
cell-wise observed-data distributions:
for control units, consistency, no interference, and $A_i=\infty$ give the
untreated distribution; for treated units before $t_0$, consistency, no
interference, and no anticipation give the untreated distribution; and on
$\Omega_{\text{tgt}}$, consistency, no interference, and the definition of
$\tilde Y_{ijk}$ give the treated distribution. Hence,
\begin{align}
    Y_{ijk}^{\text{obs}}
\sim
\begin{cases}
P_{ij},
& (i,j)\in\Omega_{\text{unt}},\\
\tilde P_{ij},
& (i,j)\in\mathcal{T}\times\mathcal{Q}.
\end{cases}
\end{align}
Note that $\tilde Y_{ijk}$ and $\tilde{\mbeta}_{ij}$ appear only in
treated post-treatment cells, while $Y_{ijk}$ and $\mbeta_{ij}$ describe the
untreated cells and the missing untreated counterfactuals.

\begin{theorem}[Counterfactual parameter identification]\label{theo:identification_one}
For a one-parameter EF with unrestricted natural parameter, let $M=\mbeta$ denote the $N\times T$ complete untreated natural-parameter matrix, partitioned as
\begin{align}
    M=\begin{pmatrix}
    A & B\\
    C & D
    \end{pmatrix},
\end{align}
according to control/treated units and pre/post-treatment periods in rows and columns, respectively. The block $D$ contains the counterfactual parameters and has dimension $|\mathcal{T}|\times|\mathcal{Q}|$. Suppose the causal assumptions above identify the observed cells in $A,B,$ and $C$ with their untreated distributions, the exponential-family parameterization is identifiable, and
\begin{align}
    \rank(A)=\rank(M).
\end{align}
Then the counterfactual block is identified as
\begin{align}
    D=CA^\dagger B,
\end{align}
where $A^\dagger$ is the Moore-Penrose pseudo-inverse of the pretreatment control block $A$.
\end{theorem}

The proof is in Appendix~\ref{app:proof_theorem}. An immediate consequence of Theorem~\ref{theo:identification_one} is that the untreated counterfactual cell distributions are identified, as are the population natural-parameter differences and KL divergences underlying the ECE and ECD, respectively. In practice, EFSC estimates the ECE and ECD by taking expectations of these identified population quantities under the variational posterior, conditioned on the observed panel.

We stated the theorem for a one-parameter EF with unrestricted $\mbeta$. The full proof for the general case of a family with $P$ reparameterized components, along with additional assumptions, is in Appendix~\ref{app:proof_theorem}. Note that Theorem~\ref{theo:identification_one} is not a posterior-consistency theorem: it says that the population observed-data distribution in the panel uniquely determines the counterfactual parameter block given the assumptions. When finite $m_{ij}$ observations are available in $(i,j)\in\Omega_{\text{unt}}$, EFSC estimates the identified block of parameters through the variational posterior over the latent factorization and the induced posterior over $\mbeta$ in $\Omega_{\text{tgt}}$. Showing that the exact posterior, or our variational approximation, concentrates around the true counterfactual block is a separate problem.

\subsection{Types of Interventions}\label{subsec:types_interventions}
We characterize an intervention through the change in the treated natural parameter $\eta_{ij}^{\mathrm{treat}}$ relative to its untreated counterfactual $\eta_{ij}^{\mathrm{ctrl}}$. We study three types of interventions in our experiments: (i) exponential tilts that preserve the exponential family structure, (ii) structured latent interventions that induce nonlinear, low-rank perturbations in the natural parameters, and (iii) distributional perturbations that introduce heavy-tailed noise and break the exponential family assumption. 

An exponential tilt adds a linear, homogeneous perturbation of magnitude $\tau$ to the natural parameter, i.e.
\begin{align}
    \tilde{\eta}_{ij} &= \eta_{ij} + \tau.
\end{align}
An example of an intervention that could be represented by a tilt is the effect of a tax reform on the yearly consumption of tobacco in California \citep{abadie2010synthetic}. The expected causal effect (ECE) in Equation~\eqref{eq:causal_effect_posterior_diff} would directly recover the effect $\tau$ of an exponential tilt \citep{efron2022exponential}. In contrast, the expected causal divergence (ECD) in Equation~\eqref{eq:causal_effect_posterior_kl} captures not only shifts in the mean but also changes in the full distribution implied by a perturbation of the exponential-family model in the treated units. The ECD admits a simple interpretation when the perturbation is an exponential tilt. In this case, 
\begin{equation}
    \text{KL}\big(p(y;\eta+\tau)\,\|\,p(y;\eta)\big)=\tau\EE{\eta+\tau}{t(y)}-a(\eta+\tau)+a(\eta),\label{eq:causal_effect_kl_bregman} 
\end{equation}
where $a(\cdot)$ is the log-partition function of the EF. Thus, under an exponential tilt, the ECD can be interpreted as the posterior mean of the Bregman divergence generated by $a(\cdot)$. See Appendix~\ref{app:exptilt} for details.

A structured latent intervention, on the other hand, induces heterogeneous effects driven by latent interactions, e.g.
\begin{align}
\tilde{\eta}_{ij} &= \eta_{ij} + \kappa \, \sigma(\mbtheta_i^\top\mbw)\, g(j),
\end{align}
where $\sigma(\cdot)$ denotes the sigmoid function, $g(j)$ is a function of the time index $j$, and $\kappa \in \mathbb{R}$ and $\mbw \in \mathbb{R}^r$. This formulation induces a perturbation that depends on the units' latent features $\mbtheta_i$, modulated along the direction $\mbw$, with time-varying dynamics governed by $g$. An example would be a drug treatment whose effect varies across biological contexts, such as cell types or genetic backgrounds \citep{mao2024learning}; in our formulation these effects may also evolve over time. This setting is particularly relevant for assessing the ability of models to capture nonlinear distributional shifts that are not directly observable at the level of outcomes, in contrast to the tilts. 

Finally, we consider a distributional perturbation that replaces treated post-treatment Gaussian observations by heavy-tailed Student-$t$ draws,
\begin{align}
    \tilde{y}_{ijk}=\mu_{ij}+s_{ij} t_{ijk},
    \; t_{ijk}\sim\mathrm{Student}\text{-}t(\nu_{\text{df}}),
\end{align}
where $s_{ij}$ is chosen so that $\tilde{y}_{ijk}$ is a Student-$t$ distribution with the same mean and variance as the original Gaussian. As the degrees of freedom, $\nu_{\text{df}}$, become smaller, the intervention alters the tail behavior of the treated distribution and leads to rare but extreme events. This setting allows us to assess the robustness of EFSC to model misspecification, since the treated observations no longer follow the assumed exponential-family model.


\section{Empirical Studies}
We assess the performance of EFSC in four experiments. First, we evaluate its ability to recover causal effects induced by exponential tilts across a range of exponential-family distributions. Second, we study more general perturbations in Gaussian panels, including structural interventions on the latent factors and heavy-tailed corruptions. Third, we run distributional placebo tests to assess the significance of the estimated effects. Finally, we analyze a real Medicaid dataset to investigate how the expansion of Medicaid under the Affordable Care Act (ACA) affected the distribution of health insurance coverage across U.S. states. Code to reproduce the experiments and apply EFSC to new datasets is available at \url{https://github.com/blei-lab/efsc}.
\subsection{Exponential Tilting across Panels of Exponential Families}\label{subsec:tilting_across_efs}
\begin{figure}[t]
\centering
\begin{subfigure}[b]{0.32\textwidth}
    \centering
    \includegraphics[width=\textwidth]{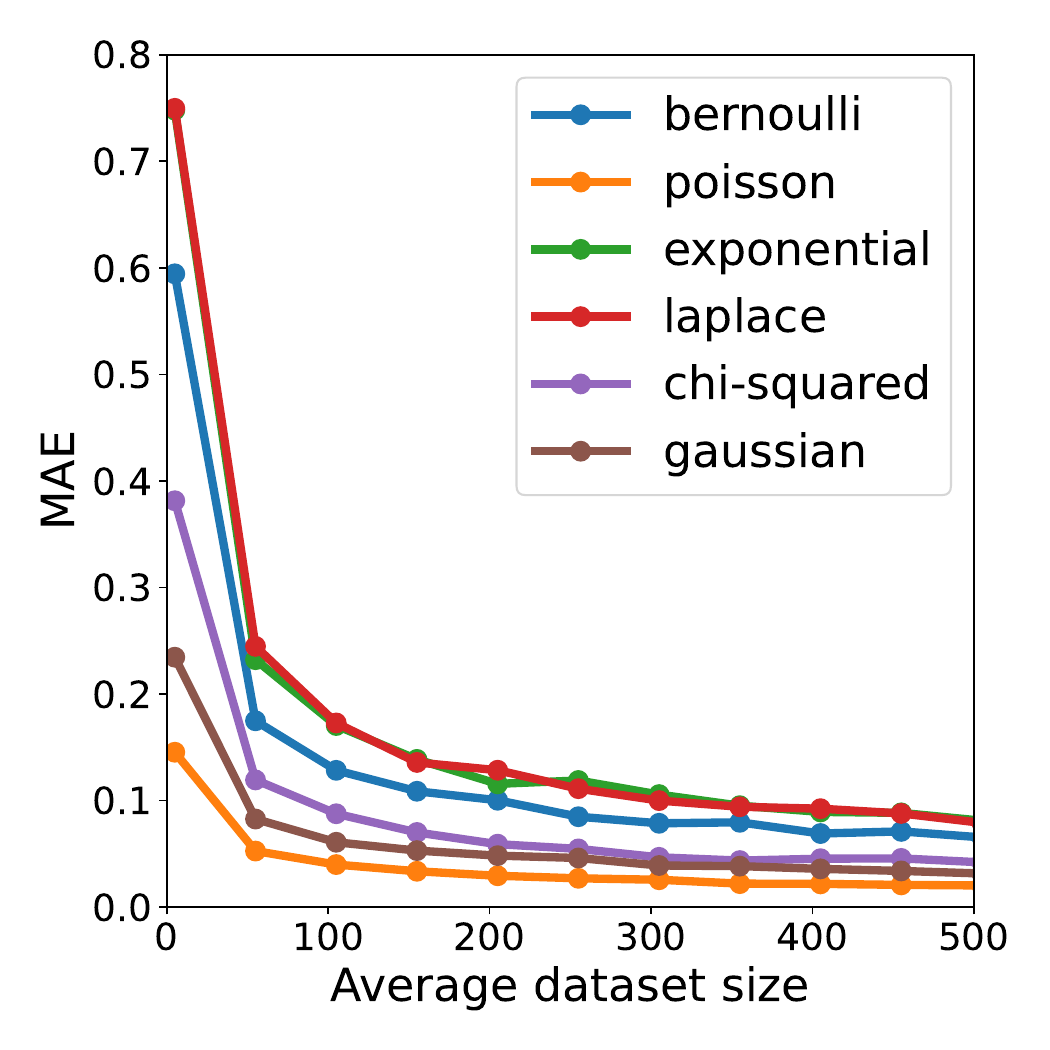}
    \caption{$\tau=0.1$}
\end{subfigure}
\hfill
\begin{subfigure}[b]{0.32\textwidth}
    \centering
    \includegraphics[width=\textwidth]{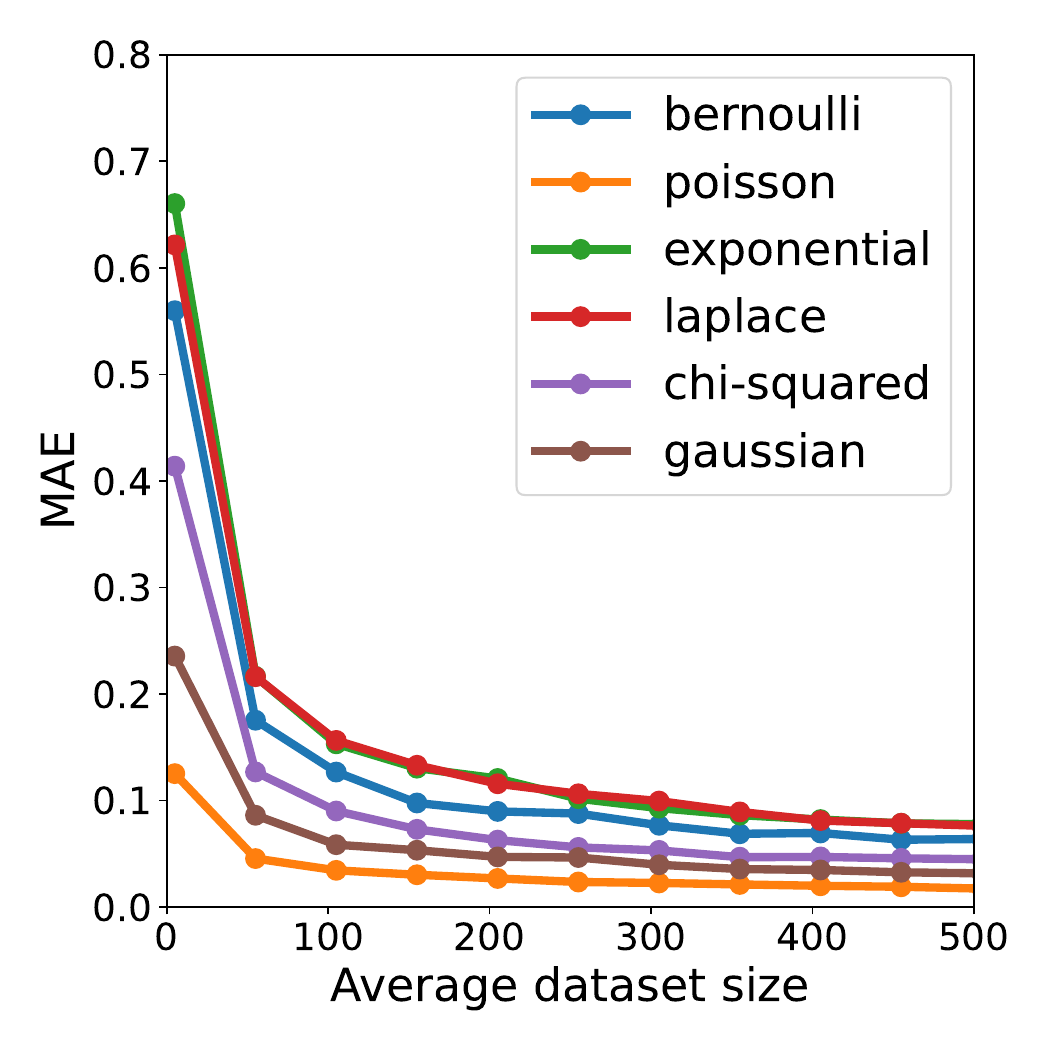}
    \caption{$\tau=0.5$}
\end{subfigure}
\hfill
\begin{subfigure}[b]{0.32\textwidth}
    \centering
    \includegraphics[width=\textwidth]{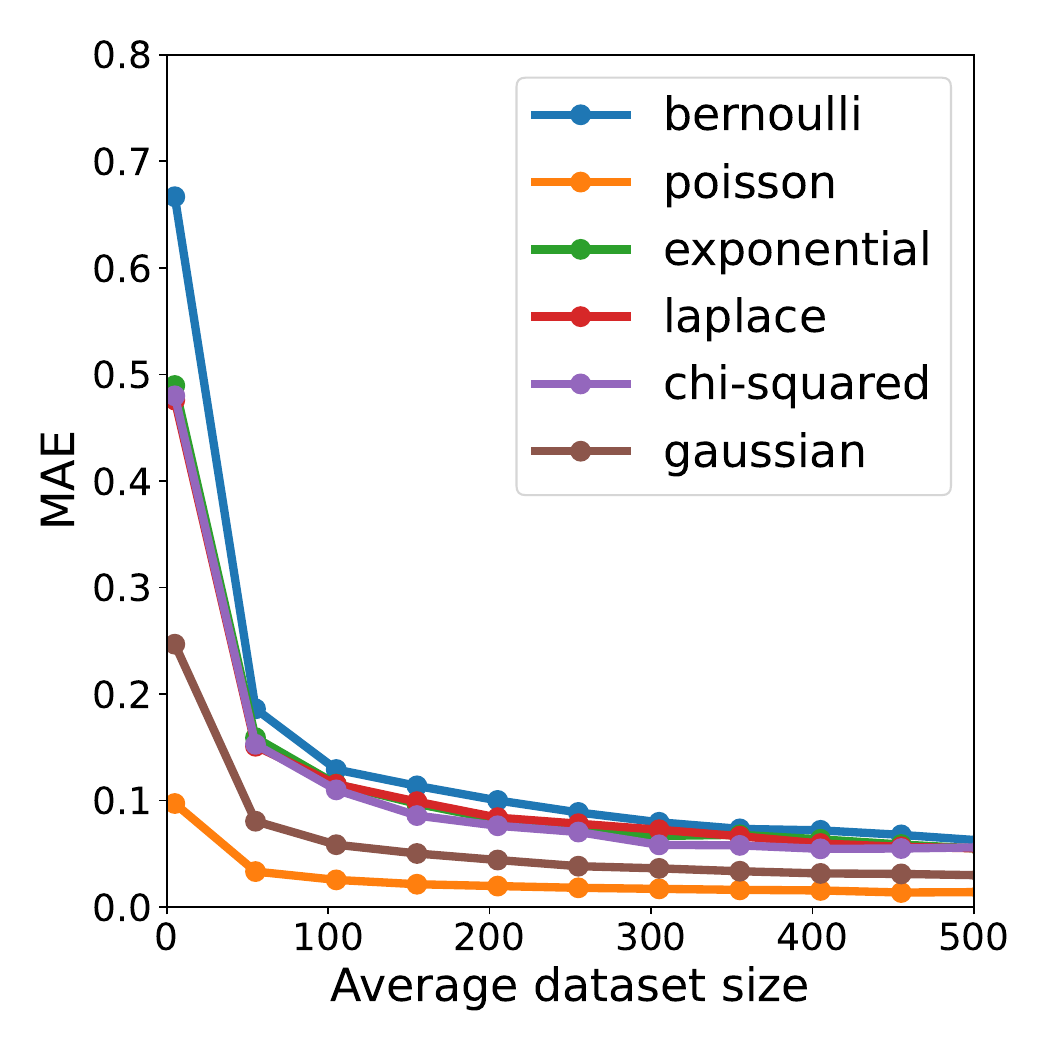}
    \caption{$\tau=2.0$}
\end{subfigure}
\caption{MAE for recovering the true exponential tilt $\tau$ across six one-dimensional EFs as the average number of observations per panel cell increases. The panel dimensions are fixed at $N=32$ and $T=64$. Recovery improves for all models and effect strengths $\tau \in \{0.1,0.5,2.0\}$ as the dataset size increases. Results are averaged over 20 replications.}
\label{fig:tilt_recovery_sparsity}
\end{figure}
In the first experiment, we investigate the ability of EFSC to recover the effect of an exponential-tilt intervention across six univariate exponential families: Bernoulli, Poisson, exponential, Laplace (known mean), Chi-Squared, and Gaussian (known variance). For a panel of size $N=32$ and $T=64$, we generate $m_{ij}\sim 1+\mathrm{Poisson}(\lambda_m)$ independent observations at each cell from $y_{ijk}\sim \mathrm{EF}(\eta_{ij})$, and then simulate an intervention $\tilde{\eta}_{ij}=\eta_{ij}+\tau$ for a subset of treated units from $t_0=52$. We fit EFSC using the BBVI procedure of Section~\ref{subsec:bbvi} to approximate the posterior distributions of the treated and counterfactual data-generating processes, and compute the posterior effect measure $\mathrm{ECE}_{ij}$ in Equation~\eqref{eq:causal_effect_posterior_diff}. In this experiment, we approximate both the ECE and ECD using plug-in natural-parameter estimates reconstructed from the variational posterior means of the factorization
parameters. The resulting ECE is expected to recover the true intervention magnitude $\tau$ for every treated cell in the panel.

Figure~\ref{fig:tilt_recovery_sparsity} reports the resulting MAE between the estimated ECE and the true tilt parameter for $\tau\in\{0.1,0.5,2\}$ as the average dataset size $\E{m_{ij}}=1+\lambda_m$ increases. The proposed model and inference procedure recover the intervention effect more accurately as the sampling density increases. This pattern is consistent across all exponential-family models, indicating that additional within-cell observations yield more informative sufficient statistics and, consequently, more accurate estimates of the latent natural parameters. Models parameterized by their mean (e.g., Poisson and Gaussian) perform particularly well, even when data are scarce (e.g., $\lambda_m=5$). In contrast, the estimation of proportions, rates, or scale parameters is more challenging in this regime. One possible explanation is the nonlinear reparameterizations required for unconstrained optimization, e.g., $\eta=\exp(z)-1$ in the Chi-Squared family. See Appendix~\ref{app:exptilt} for additional details.

We also evaluated the ability of EFSC to estimate the true expected causal divergence (ECD) in Equation~\eqref{eq:causal_effect_posterior_kl}. Across all exponential families, the estimated ECD generally tracked the true divergence, with accuracy improving as the number of observations per panel cell increased. Further details and ECD recovery plots are reported in Appendix~\ref{app:exp_tilt_ecd}.

\subsection{Intervention Benchmarks on Gaussian Panels}\label{subsec:gaussian_intervention_benchmarks}
We now focus on synthetic Gaussian panels and evaluate EFSC under the three types of interventions defined in Section~\ref{subsec:types_interventions}: exponential tilts of the natural parameters, structured latent interventions, and heavy-tailed contamination of treated outcomes. We simulate the data from the generating process in Appendix~\ref{app:generative_model} for a univariate Gaussian, which is a two-parameter EF with $\mbeta_{ij}=(\mu_{ij}/\sigma_{ij}^2,
-1/(2\sigma_{ij}^2))^\top\in\mathbb{R}\times(-\infty,0)$; see also Appendix~\ref{app:benchmark_gaussian_panels}. The unconstrained predictors corresponding to $\mbeta_{ij}$ are generated from a shared latent factorization and then mapped into the Gaussian natural-parameter space. Then we draw $m_{ij}$ samples for each cell of the panel from the corresponding distribution $y_{ijk}|\mbeta_{ij}$. For the exponential-tilt experiments, the dataset size is drawn independently
as $m_{ij}\sim 1+\operatorname{Poisson}(\lambda_m)$ for each panel cell. The structured-intervention and Student-$t$ experiments instead use the fixed cell sizes reported in the corresponding tables. The dimension of the latent factors, $\mbtheta_i$ and $\mbbeta_j$, is set to $r=2$; further details on the experimental setup are in Appendix~\ref{app:benchmark_gaussian_panels}.

We benchmark three models: i) the classic synthetic controls from \citet{abadie2010synthetic} for multiple treated units on the sample mean of each cell, $\bar{Y}_{ij}$. This model is expected to perform well in recovering the true average effects from linear interventions in $\mbeta$ like an exponential tilt. ii) SC on the cell-wise maximum likelihood estimates from $y_{ijk}$, i.e., on the matrix with elements $\hat{\mbeta}_{ij} = \arg\max_{\mbeta_{ij}} p(\mby_{ij} \mid \mbeta_{ij})$. This is an exponential family SC approach that works directly on the natural parameters and estimates counterfactuals using the available information from the MLE estimates. We denote this method as EFSC-MLE, and expect it to outperform the outcome-based version of SC on nonlinear interventions. iii) Our factorization-based model, labeled EFSC-PMF. Causal effects are calculated for the different models using the estimands proposed in Section~\ref{subsec:method_overview}. In the reported EFSC-PMF results, posterior quantities are approximated using either Monte Carlo draws or plug-in reconstructions evaluated at the variational posterior means, as detailed in Appendix~\ref{app:benchmark_gaussian_panels}.

\begin{table}[t]
\centering
\caption{MAE for recovering the true exponential tilts $\tau$ from estimated counterfactuals in Gaussian panels. EFSC-PMF, which leverages the latent factorization, improves over traditional synthetic control and MLE-based counterfactuals at the response level, and over the MLE-based estimator at the natural-parameter level. Results are averaged over 20 independent replications. Bold indicates the best-performing value in each setting. }\label{tab:exponential_tilt_results}
\begin{tabular}{llccccc}
\toprule
Model & Estimand & $\tau=0.1$ & $\tau=0.25$ & $\tau=0.5$ & $\tau=1$ & $\tau=2$ \\
\midrule
Synthetic Control & $\bar{Y}_{ij}-\sum_{l\in\mathcal{C}} w_{il}\bar{Y}_{lj}$
  & 0.169 & 0.169 & 0.169 & 0.169 & 0.169 \\
EFSC - MLE & $\bar{Y}_{ij}-\hat{\mu}^{\text{ctrl}}_{ij}$
  & 0.169 & 0.169 & 0.169 & 0.169 & 0.169 \\
EFSC - PMF & $\bar{Y}_{ij}-\hat{\mu}^{\text{ctrl}}_{ij}$
  & \textbf{0.161} & \textbf{0.161} & \textbf{0.161} & \textbf{0.161} & \textbf{0.161} \\
\addlinespace
EFSC - MLE & $\hat{\eta}_{ij}^{\text{treat}}-\hat{\eta}_{ij}^{\text{ctrl}}$
  & 0.180 & 0.184 & 0.195 & 0.239 & 0.367 \\
EFSC - PMF & $\mathbb{E}[{\eta}_{ij}^{\text{treat}}-{\eta}_{ij}^{\text{ctrl}}|\mby]$
  & \textbf{0.092} & \textbf{0.095} & \textbf{0.101} & \textbf{0.122} & \textbf{0.194} \\
\bottomrule
\end{tabular}
\end{table}

Unless stated otherwise, mean absolute errors (MAE) are computed across treated post-treatment cells within each replication and then averaged over replications; for vector-valued effects, the average also includes natural-parameter components.

Table~\ref{tab:exponential_tilt_results} presents the MAE between the estimated effects and the true effect from an exponential tilt, $\tilde{\mbeta}_{ij}=\mbeta_{ij}+(\tau,0)^T$, with increasing values of $\tau$ on a panel with $N=32$ units, $T=128$ time steps, and ragged cell sizes $m_{ij}\sim 1+\mathrm{Poisson}(\lambda_m)$ with $\lambda_m=55$. When the estimand is defined at the outcome level, all methods perform similarly, with the PMF-based estimator achieving a small improvement over both the synthetic control and the MLE-based approach. To place all estimators on the same natural-parameter tilt scale, the response-level effects in the first three rows of
the table are divided cellwise by the true simulated variance \(\sigma_{ij}^2\). The final two rows report recovery of the first natural-parameter component. Further details are provided in Appendix~\ref{app:benchmark_gaussian_panels}.

The outcome-level MAEs are insensitive to the strength of the intervention in the natural parameter space, but when it comes to recovering natural-parameter tilt, the differences are substantial. The MLE-based estimator consistently exhibits larger errors than EFSC-PMF across all values of $\tau$, indicating that leveraging latent structure is critical for accurately recovering distributional effects.

Overall, these results highlight that while outcome-level estimands may mask differences between methods, inference at the level of natural parameters reveals clear advantages of the factorization approach for lower dataset sizes. Figure~\ref{fig:exp_tilt_panel} further illustrates this behavior for the vector tilt $\mbtau=(0.4,-0.6)^\top$, showing that the PMF advantage is largest at smaller average cell sizes and decreases as local MLE estimates become more stable.

\begin{figure}[t]
    \centering
    \includegraphics[width=0.98\linewidth]{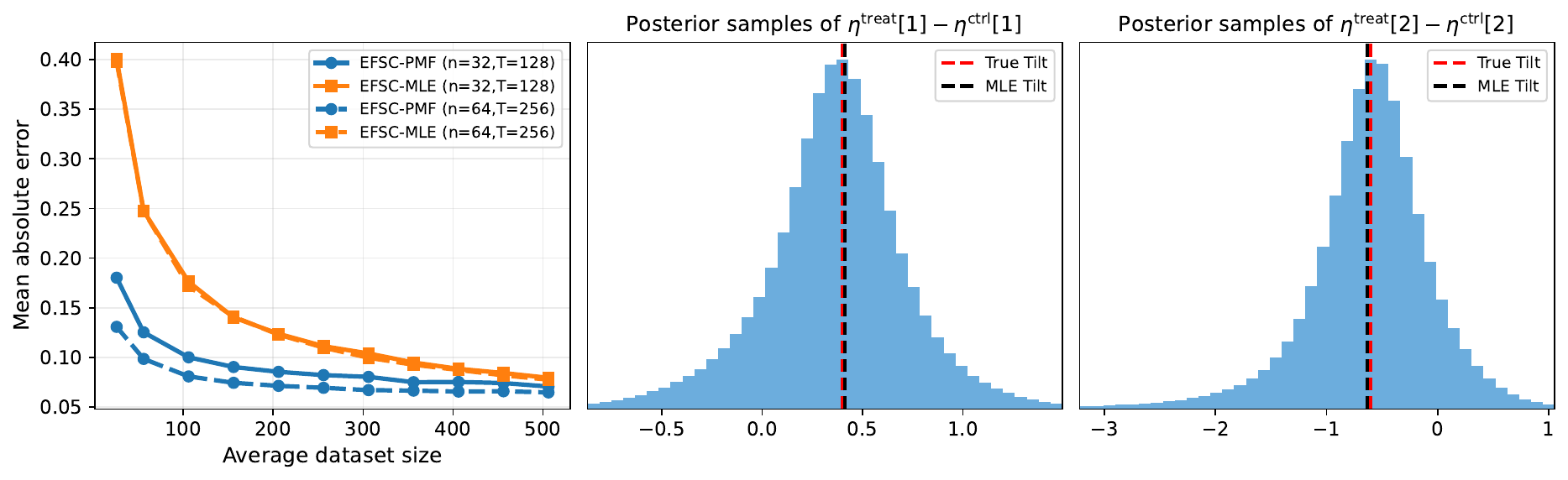}
    \caption{EFSC-PMF and EFSC-MLE recovery of a Gaussian exponential tilt $\mbtau=(0.4,-0.6)^\top$. Left: MAE versus average observations per cell for two panel sizes, averaged over 20 replications. Center and right: pooled EFSC-PMF posterior draws of $\mbeta^{\mathrm{treat}}_{ij}-\mbeta^{\mathrm{ctrl}}_{ij}$ for treated post-treatment cells in one representative panel; vertical lines mark the true tilt and EFSC-MLE estimate. Both estimators are centered near the true tilt.}\label{fig:exp_tilt_panel}
\end{figure}

We also evaluate the ability of the EFSC-PMF and EFSC-MLE models to recover the true KL divergence from the treated distribution to the counterfactual untreated distribution under a structured intervention with interaction between the latent factors:
\begin{align}
\tilde{\eta}_{ij}[1] &= \eta_{ij}[1] + \kappa_1 \, \sigma(\mbtheta_i^\top \mbw_1) \, \sigma(\mbbeta_j^\top \mbw_2), \\
\tilde{\eta}_{ij}[2] &= \eta_{ij}[2] - \kappa_2 \, \sigma(\mbtheta_i^\top \mbw_1) \, \sigma(\mbbeta_j^\top \mbw_2), \quad \tilde{\eta}_{ij}[2] < 0,
\end{align}
where $\sigma(\cdot)$ denotes the sigmoid function, $\kappa_1,\kappa_2\geq 0$ are constants, and $\mbw_1,\mbw_2\in\mathbb{R}^r$ are unit-norm vectors. 

Table~\ref{tab:kl_mae} reports the cell-wise MAE in recovering the true KL divergence across different intervention strengths $(\kappa_1,\kappa_2)$ and fixed sample sizes per cell $m$. EFSC-PMF substantially outperforms the MLE-based approach in low-data regimes, highlighting the advantage of leveraging latent structure to learn across units and time. As the number of observations per cell increases, the performance gap narrows, and MLE becomes competitive. These results corroborate our intuition that PMF provides significant gains in recovering distributional effects under structured latent interventions in sparse and moderate-data settings, while MLE requires larger sample sizes to achieve comparable performance.

\begin{table}[t]
\centering
\caption{MAE for recovering the true KL divergence under structured latent interventions on Gaussian panels. EFSC-PMF achieves lower errors than the MLE-based estimator, especially in sparse-data regimes. Results are averages over 20 replications.}\label{tab:kl_mae}
\begin{tabular}{llccccc}
\toprule
Model & $(\kappa_1,\kappa_2)$ & $m=5$ & $m=25$ & $m=50$ & $m=100$ & $m=200$ \\
\midrule
EFSC - MLE & (0.2, 0.1) & 0.812 & 0.052 & 0.024 & 0.012 & 0.006 \\
EFSC - PMF & (0.2, 0.1) & \textbf{0.061} & \textbf{0.013} & \textbf{0.008} & \textbf{0.004} & \textbf{0.003} \\
\addlinespace
EFSC - MLE & (0.6, 0.5) & 0.648 & 0.046 & 0.024 & 0.014 & 0.009 \\
EFSC - PMF & (0.6, 0.5) & \textbf{0.056} & \textbf{0.014} & \textbf{0.011} & \textbf{0.007} & \textbf{0.005} \\
\addlinespace
EFSC - MLE & (1.6, 1.5) & 0.525 & 0.055 & 0.037 & 0.024 & 0.017 \\
EFSC - PMF & (1.6, 1.5) & \textbf{0.062} & \textbf{0.025} & \textbf{0.017} & \textbf{0.013} & \textbf{0.010} \\
\bottomrule
\end{tabular}
\end{table}

Finally, we consider a simple misspecification experiment in which the treated post-treatment Gaussian observations are replaced by heavy-tailed Student-$t$ observations. The counterfactual distribution remains Gaussian, but the treated distribution is no longer in the exponential family. We compute the true cell-wise KL divergence from the treated Student-$t$ distribution to the counterfactual Gaussian distribution by Monte Carlo and
compare it with the KL induced by the Gaussian EFSC-PMF and EFSC-MLE reconstructions. Full details are provided in Appendix~\ref{app:benchmark_gaussian_panels}. Table~\ref{tab:student_t_corruption_kl} shows that EFSC-PMF is beneficial in sparse panels, where borrowing information across units and time stabilizes the reconstruction. As the number of observations per cell increases, the local MLE becomes
competitive under stronger misspecification, reflecting the increasing stability of local moment estimates despite the greater sampling variability
induced by heavy-tailed observations.

\begin{table}[t]
\centering
\caption{MAE for recovering the true KL divergence under heavy-tailed Student-$t$ corruption on Gaussian panels. EFSC-PMF is most beneficial in sparse panels, whereas the MLE-based estimator becomes competitive as the number of observations per cell increases. Results are averaged over 20 replications; the true divergence is approximated via MC.}\label{tab:student_t_corruption_kl}
\begin{tabular}{llccccc}
\toprule
Model & d.f. & $m=5$ & $m=25$ & $m=50$ & $m=100$ & $m=200$ \\
\midrule
EFSC - MLE & $80$ & 0.871 & 0.057 & 0.027 & 0.013 & 0.007 \\
EFSC - PMF & $80$ & \textbf{0.053} & \textbf{0.015} & \textbf{0.014} & \textbf{0.010} & {0.007} \\
\addlinespace
EFSC - MLE & $40$ & 0.874 & 0.059 & 0.028 & 0.013 & 0.006 \\
EFSC - PMF & $40$ & \textbf{0.057} & \textbf{0.014} & \textbf{0.013} & \textbf{0.010} & {0.006} \\
\addlinespace
EFSC - MLE & $20$ & 0.773 & 0.057 & 0.026 & 0.012 & \textbf{0.005} \\
EFSC - PMF & $20$ & \textbf{0.055} & \textbf{0.015} & \textbf{0.011} & \textbf{0.010} & 0.006 \\
\addlinespace
EFSC - MLE & $10$ & 0.928 & 0.059 & 0.024 & \textbf{0.011} & \textbf{0.006} \\
EFSC - PMF & $10$ & \textbf{0.057} & \textbf{0.016} & \textbf{0.012} & 0.012 & 0.008 \\
\addlinespace
EFSC - MLE & $5$ & 1.078 & 0.066 & 0.038 & \textbf{0.033} & \textbf{0.038} \\
EFSC - PMF & $5$ & \textbf{0.059} & \textbf{0.035} & \textbf{0.036} & 0.038 & 0.040 \\
\addlinespace
EFSC - MLE & $3$ & 1.025 & 0.219 & 0.196 & 0.185 & 0.187 \\
EFSC - PMF & $3$ & \textbf{0.151} & \textbf{0.163} & \textbf{0.176} & \textbf{0.173} & \textbf{0.178} \\
\bottomrule
\end{tabular}
\end{table}

\subsection{Distributional Placebo Tests}\label{subsec:placebo_tests}
We validate our framework using placebo tests based on the change in the KL divergences induced by the model. Define the increment in distributional discrepancy after the intervention as
\begin{align}
  \Delta_{\mathrm{KL}}=\mathrm{ECD}^{\mathrm{post}}-\mathrm{ECD}^{\mathrm{pre}},\label{eq:placebo_statistic}
\end{align}
where
\begin{align}
\mathrm{ECD}^{\mathrm{pre}}=
\frac{1}{|\mathcal{T}||\mathcal{P}|}
\sum_{i\in\mathcal{T}}
\sum_{j\in\mathcal{P}}
\mathrm{ECD}_{ij},
\end{align}
and similarly for the post-treatment period. These quantities represent the average expected causal divergence, as defined in Equation~\eqref{eq:causal_effect_posterior_kl}, before and after the intervention. In this experiment, we approximate the ECD 
terms by evaluating the KL divergence at natural parameters reconstructed from the variational posterior mean of the factorization.

For inference, we generate $B$ placebo assignments by randomly selecting subsets of control units with the same cardinality as the treated set
$\mathcal T$. In the synthetic experiments below, the control pool is sufficiently small that we enumerate all
$B=\tbinom{|\mathcal C|}{|\mathcal T|}$ same-cardinality subsets rather than sampling them, so the reference distribution contains every admissible placebo block. For each assignment, $b=1,\ldots,B$, we recompute the corresponding placebo statistic
\begin{align}
\Delta_{\mathrm{KL}}^{(b)}=\mathrm{ECD}^{\mathrm{post},(b)}-
\mathrm{ECD}^{\mathrm{pre},(b)}.
\end{align}
The placebo distribution is motivated by the randomization-style argument that, without a distributional treatment effect, the observed treated block should not produce systematically larger values of $\Delta_{\mathrm{KL}}$ than same-cardinality blocks of control units \citep{abadie2010synthetic,good2005permutation}. In short, the placebo statistics provide an empirical reference distribution for the observed statistic in Equation~\eqref{eq:placebo_statistic}. In the experiments below, we use the no-leakage implementation in Algorithm~\ref{alg:efsc-placebo}, which excludes the treated post-treatment cells from all conditioning sets used to construct the placebo distribution.

\begin{algorithm}[t]
\SetAlgoNoLine
\DontPrintSemicolon
\caption{No-leakage EFSC placebo test}\label{alg:efsc-placebo}
\KwIn{Panel datasets $\mby$, treated and control unit sets $\mathcal{T}$, $\mathcal{C}$, pre-treatment periods $\mathcal{P}$, post-treatment periods $\mathcal{Q}$, number of placebo assignments $B$}
\KwOut{Observed statistic $\Delta_{\mathrm{KL}}^{\mathrm{obs}}$, placebo statistics $\{\Delta_{\mathrm{KL}}^{(b)}\}_{b=1}^B$, empirical $p$-value}
Compute the observed statistic $\Delta_{\mathrm{KL}}^{\mathrm{obs}}$ from Equation~\eqref{eq:placebo_statistic} using EFSC fits for the pre- and post-treatment ECDs\;
\For{$b=1,\ldots,B$}{
Draw a placebo treated set $\mathcal{T}^{(b)}\subseteq\mathcal{C}$ with $|\mathcal{T}^{(b)}|=|\mathcal{T}|$\;

Fit the EFSC models required to recompute \(\mathrm{ECD}^{\mathrm{pre},(b)}\) and
\(\mathrm{ECD}^{\mathrm{post},(b)}\), excluding actual treated post-treatment
cells from all placebo conditioning sets and excluding the placebo
post-treatment block from its corresponding counterfactual conditioning set\;
Compute $\Delta_{\mathrm{KL}}^{(b)}=\mathrm{ECD}^{\mathrm{post},(b)}-
\mathrm{ECD}^{\mathrm{pre},(b)}$
}
Compute the right-tail empirical probability:
$\hat p=(1+\sum_{b=1}^B
\mathbb{I}\{\Delta_{\mathrm{KL}}^{(b)}
\geq \Delta_{\mathrm{KL}}^{\mathrm{obs}}\})/(B+1)$
\end{algorithm}

Figure~\ref{fig:placebo_tests} illustrates the placebo procedure in two synthetic settings. In the left panel, we consider a Poisson panel under exponential tilts of increasing magnitude. The placebo distribution is computed once from assignments of control units with the same cardinality as the treated set, following Algorithm~\ref{alg:efsc-placebo}, along with the observed statistics corresponding to each tilt. As the intervention becomes stronger, the observed $\Delta_{\mathrm{KL}}$ tends to move farther into the right tail of the placebo distribution. Appendix~\ref{app:placebo_tests} reports analogous placebo distributions for the six one-parameter exponential families used in Section~\ref{subsec:tilting_across_efs}. 

In the right panel, we consider a two-parameter Gaussian response and compare our three intervention mechanisms: an exponential tilt in the natural parameter, a structured latent intervention depending on the unit and time factors, and a Student-$t$ replacement of the treated post-treatment observations. This experiment evaluates whether the same statistic can detect distributional changes arising from different sources: a direct shift in the EF natural parameter, a structured latent perturbation, and a misspecified heavy-tailed response distribution. In all cases, the observed values lie in the right tail of the placebo distribution, indicating that the ECD statistic captures distributional discrepancies induced by these types of interventions. For the Student-$t$ replacement, the Gaussian plug-in statistic is compared formally with the placebo distribution. We additionally show the MC-evaluated KL under the known Student-$t$ distribution as an oracle diagnostic; because it is not computed using the same estimation rule as the placebo statistics, we do not assign it a placebo \(p\)-value. The oracle value is larger than the Gaussian plug-in KL, reflecting the limited sensitivity of the Gaussian approximation to tail-shape changes.

\begin{figure}[t]
    \centering
    \begin{subfigure}{0.48\linewidth}
        \centering
        \includegraphics[height=0.22\textheight]{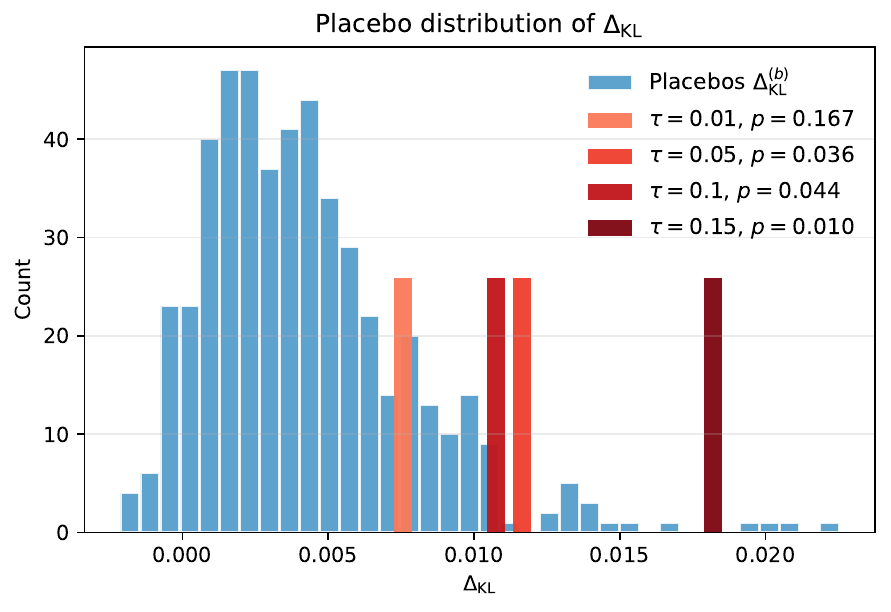}
        \caption{Poisson exponential tilting.}
    \end{subfigure}
    \hfill
    \begin{subfigure}{0.48\linewidth}
        \centering
        \includegraphics[height=0.22\textheight]{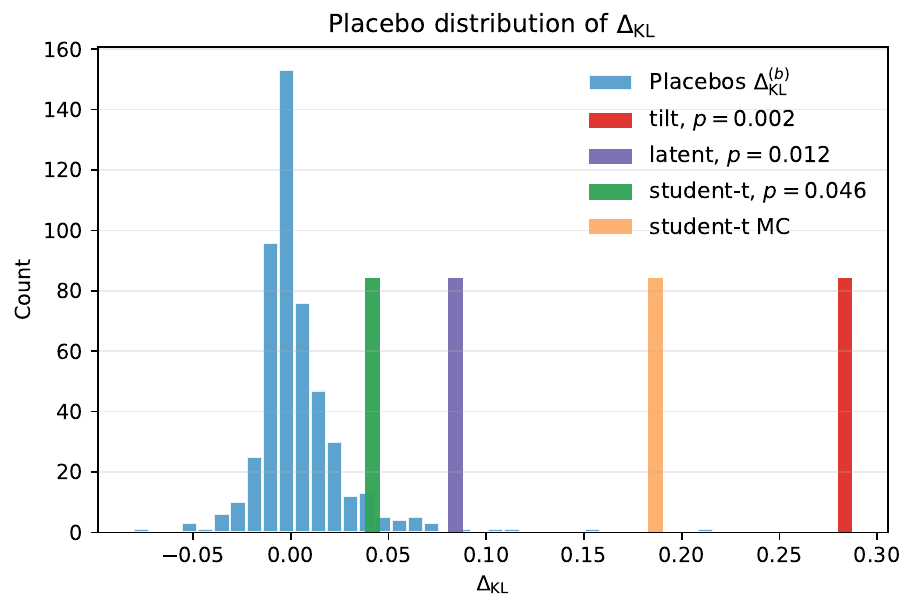}
        \caption{Gaussian intervention benchmark.}
    \end{subfigure}
    \caption{Distributional placebo tests on synthetic panels. Left: placebo distribution of $\Delta_{\mathrm{KL}}$ for a Poisson panel, with observed statistics for exponential tilts of increasing magnitude. Right: placebo distribution for a Gaussian panel, with observed statistics for
an exponential tilt, a structured latent intervention, and a Student-$t$ replacement evaluated using both Gaussian plug-in and MC-based KL. All interventions produce statistics in the right tail of the empirical placebo distribution.}
    \label{fig:placebo_tests}
\end{figure}

\subsection{Distributional Effects of the Medicaid Expansion}\label{subsec:medicaid}
Finally, we apply EFSC to evaluate the Medicaid expansion introduced under the
Affordable Care Act (ACA; \citeauthor{patientprotection2010},
\citeyear{patientprotection2010}). The expansion generally extended Medicaid
eligibility to nonelderly adults with incomes up to 138\% of the federal poverty
level in states that adopted the policy. States implemented the expansion at
different times, while others did not adopt it during our sample period,
producing a staggered treatment setting, see Figure~\ref{fig:medicaid_expansion}. Our goal is to estimate how the policy
changed the full distribution of health insurance coverage.

We use individual-level microdata from the American Community Survey (ACS),
obtained through IPUMS USA\footnote{IPUMS USA, University of Minnesota,
\url{https://usa.ipums.org/usa/}.} \citep{ruggles2025ipumsusa}, restrict the sample to adults aged 19--64
from low-income families, and classify each respondent into one of five mutually
exclusive insurance categories: uninsured, Medicaid, employer-sponsored
insurance, direct-purchase private insurance, or other coverage. Further details
on sample construction and insurance-category assignment are provided in
Appendix~\ref{app:medicaid_real_data}.

Within each state-year cell, we use the ACS person-level survey weights to aggregate respondents into counts over the five insurance categories. Thus each panel observation is a state-year multinomial response. The resulting dataset is a complete state-year panel covering the 50 states and the District of Columbia over 2008--2019. Treatment adoption is staggered: states enter treatment according to their Medicaid expansion implementation year, while states that do not expand during the 2008--2019 window serve as untreated donor units. States that expanded after 2019 are therefore considered untreated for the purposes of this study.  

\begin{figure}
    \centering
    \includegraphics[width=0.9\linewidth]{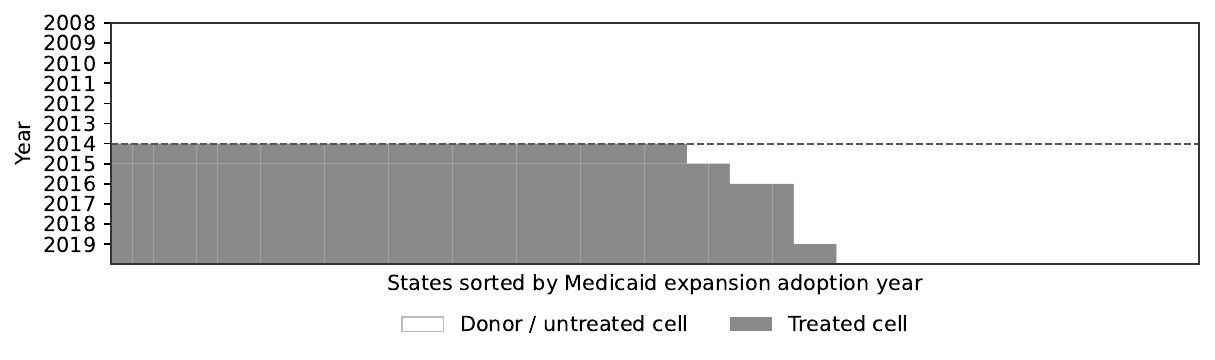}
    \caption{Medicaid expansion treatment timing for the 51 state-level units over
the 2008--2019 sample period. Approximately 70\% of the \(51\times12=612\)
state-year cells are donor/untreated cells. Most adopting states expanded
Medicaid in 2014, while 17 units did not adopt before 2020 and therefore remain
untreated throughout the estimation window.}\label{fig:medicaid_expansion}
\end{figure}

For each state $i$ and year $j$, let
$\mathbf y_{ij}=(y_{ij1},\ldots,y_{ijC})$ denote the weighted counts across
the $C=5$ insurance categories, and let $m_{ij}=\sum_{c=1}^{C}y_{ijc}$ be 
the survey-weighted total count in the cell. We model the category composition through a multinomial observation model,
\begin{align}
    \mathbf y_{ij} \mid m_{ij}, \boldsymbol{\pi}_{ij}
    &\sim\operatorname{Multinomial}(m_{ij}, \boldsymbol{\pi}_{ij}).
\end{align}
Because the ACS weights produce population-representative weighted totals rather
than literal independent sample sizes, we do not fit the model using the raw
weighted totals directly. Instead, for estimation we rescale each state-year
count vector to a common effective cell size of 1,000 while preserving the
survey-weighted category proportions. This keeps the likelihood on a comparable
scale across states and prevents large-population states from dominating the
fit solely because of their weighted cell totals; see Appendix~\ref{app:medicaid_real_data} for further details.

We use category $C$ ('other') as the reference category. The natural parameters are the
category logits
\begin{align}
    \eta_{ij}[c]&=\log\frac{\pi_{ijc}}{\pi_{ijC}},\quad c=1,\ldots,C-1.
\end{align}
For each non-reference category $c=1,\ldots,C-1$, we model the corresponding
natural parameter using an additive state effect, an additive time effect, and a
low-rank interaction term:
\begin{align}
    \eta_{ij}[c]&=\alpha_{i}[c]+\gamma_{j}[c] +
    \boldsymbol{\theta}_{i}[c]^{\top}\boldsymbol{\beta}_{j}[c].
\end{align}
Here $\alpha_{i}[c]$ captures category-specific state heterogeneity,
$\gamma_{j}[c]$ captures category-specific time variation, and
$\boldsymbol{\theta}_{i}[c]^{\top}\boldsymbol{\beta}_{j}[c]$ captures residual
state-time dependence through a low-rank factorization. We use a separate latent factorization for each logit, and fit the
multinomial EFSC model using BBVI with a mean-field Gaussian approximation over
the latent parameters and fixed independent Gaussian priors.

We train the counterfactual model on all donor/untreated cells, excluding
treated post-expansion state-year cells. We fit a second model on the
treated post-expansion cells. We then compute treatment effects by comparing 
the probability vectors implied by the variational posterior-mean logits 
under the treated and counterfactual fitted models.

\begin{table}[t]
\centering
\caption{Average estimated treated-post insurance distributions. Medicaid
coverage among low-income adults in treated states is $14.8$ percentage points higher under the treated model than under the counterfactual. The table reports the probabilities implied by the variational posterior-mean logits under both models, averaged over treated post-expansion state-year cells, and their difference.}\label{tab:acs_probability_shifts}
\begin{tabular}{lccc}
\toprule
Insurance category & Counterfactual & Treated & Difference \\
\midrule
Uninsured & 0.263 & 0.173 & -0.090 \\
Medicaid  & 0.355 & 0.503 &  0.148 \\
Employer  & 0.237 & 0.201 & -0.037 \\
Private   & 0.100 & 0.079 & -0.020 \\
Other     & 0.045 & 0.044 & -0.001 \\
\bottomrule
\end{tabular}
\end{table}

\begin{figure}
    \centering
    \includegraphics[width=1.0\linewidth]{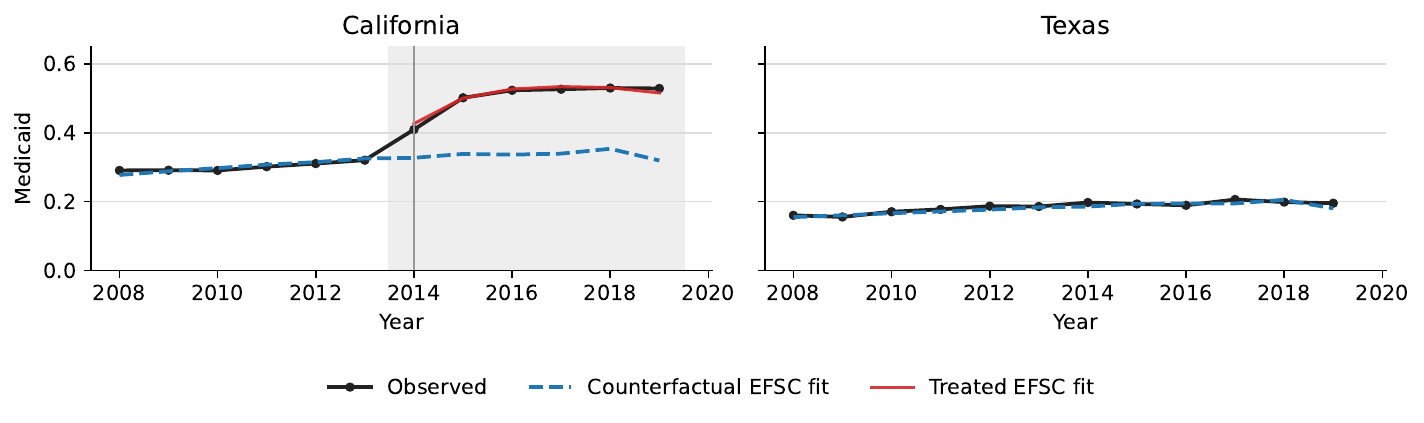}
    \caption{Estimated Medicaid probabilities for selected states over
2008--2019. The plot compares observed survey-weighted Medicaid proportions with the
probabilities implied by the variational posterior-mean logits under the
counterfactual model for both states and under the treated model for
California's post-expansion years. California illustrates a treated state with a sharp post-2014 increase
in Medicaid coverage relative to its counterfactual fit, while Texas is a
never-treated state during the analysis window.}\label{fig:medicaid_probability_paths}
\end{figure}

Table~\ref{tab:acs_probability_shifts} shows a clear redistribution of probability mass from uninsured coverage to Medicaid. The average treated-post Medicaid coverage probability increases from 0.355 under the counterfactual model to 0.503 under the treated model, while the uninsured probability decreases by 0.09. The remaining categories show smaller declines, and the ``Other'' category is essentially unchanged. This aggregate result is consistent with the primary policy mechanism of Medicaid expansion: coverage shifts away from uninsured status and toward Medicaid enrollment. Figure~\ref{fig:medicaid_probability_paths} illustrates this pattern for two specific states: California and Texas. California exhibits an increase in Medicaid coverage after the 2014 expansion relative to its counterfactual trajectory. In contrast, Texas, which does not expand Medicaid during the sample period, remains stable and serves as an untreated reference trajectory.

The previous aggregate table and selected trajectories summarize the main direction of the effect, but they do not reveal how distributional changes vary across states. To highlight this heterogeneity, we compute two state-level summaries. First, for each treated state $i$, we compute its average Medicaid probability effect over treated post-expansion years $\mathcal Q_i$,
\begin{align}
\Delta^{\mathrm{Medicaid}}_i=
\frac{1}{|\mathcal Q_i|}
\sum_{j\in\mathcal Q_i}
\left(
\hat{\pi}^{\mathrm{treat}}_{ij,\mathrm{Medicaid}}-
\hat{\pi}^{\mathrm{ctrl}}_{ij,\mathrm{Medicaid}}
\right).
\end{align}
Second, we compute a posterior-mean approximation to the corresponding average expected causal divergence,
\begin{align}
\mathrm{ECD}_i=
\frac{1}{|\mathcal Q_i|}
\sum_{j\in\mathcal Q_i}
\operatorname{KL}
\left(
\hat{\boldsymbol{\pi}}^{\mathrm{treat}}_{ij}
\,\|\, 
\hat{\boldsymbol{\pi}}^{\mathrm{ctrl}}_{ij}
\right).
\end{align}
This approximation evaluates the KL divergence between the treated and
counterfactual probability vectors implied by their respective variational
posterior-mean logits. We report the probability-vector KL, rather than the
multinomial-cell KL multiplied by the effective cell size, so that the quantity
is on the scale of distributional change in category probabilities.

\begin{figure}
    \centering
    \includegraphics[width=0.67\linewidth]{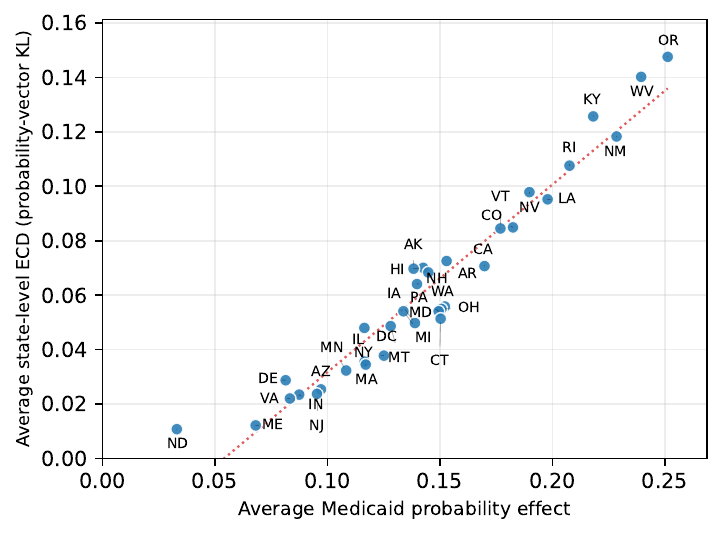}
    \caption{State-level Medicaid probability effect versus the posterior-mean approximation
to state-level expected causal divergence (ECD), averaged over treated post-expansion years. States with
large Medicaid probability effects tend to have large distributional effects.}\label{fig:medicaid_state_ecd_scatter}
\end{figure}

Figure~\ref{fig:medicaid_state_ecd_scatter} compares these two summaries across
treated states. The strong positive association shows that larger Medicaid
enrollment effects are generally accompanied by larger distributional effects.
The relationship is not exact because ECD summarizes changes in the full
insurance distribution rather than Medicaid enrollment alone.

To identify which categories drive the state-level distributional effects, we
also compute the average post-treatment probability shift for every category,
\begin{align}
\Delta_{ic}=
\frac{1}{|\mathcal Q_i|}
\sum_{j\in\mathcal Q_i}
\left(
\hat{\pi}^{\mathrm{treat}}_{ijc}-
\hat{\pi}^{\mathrm{ctrl}}_{ijc}
\right).
\end{align}

\begin{figure}
    \centering
    \includegraphics[width=1.0\linewidth]{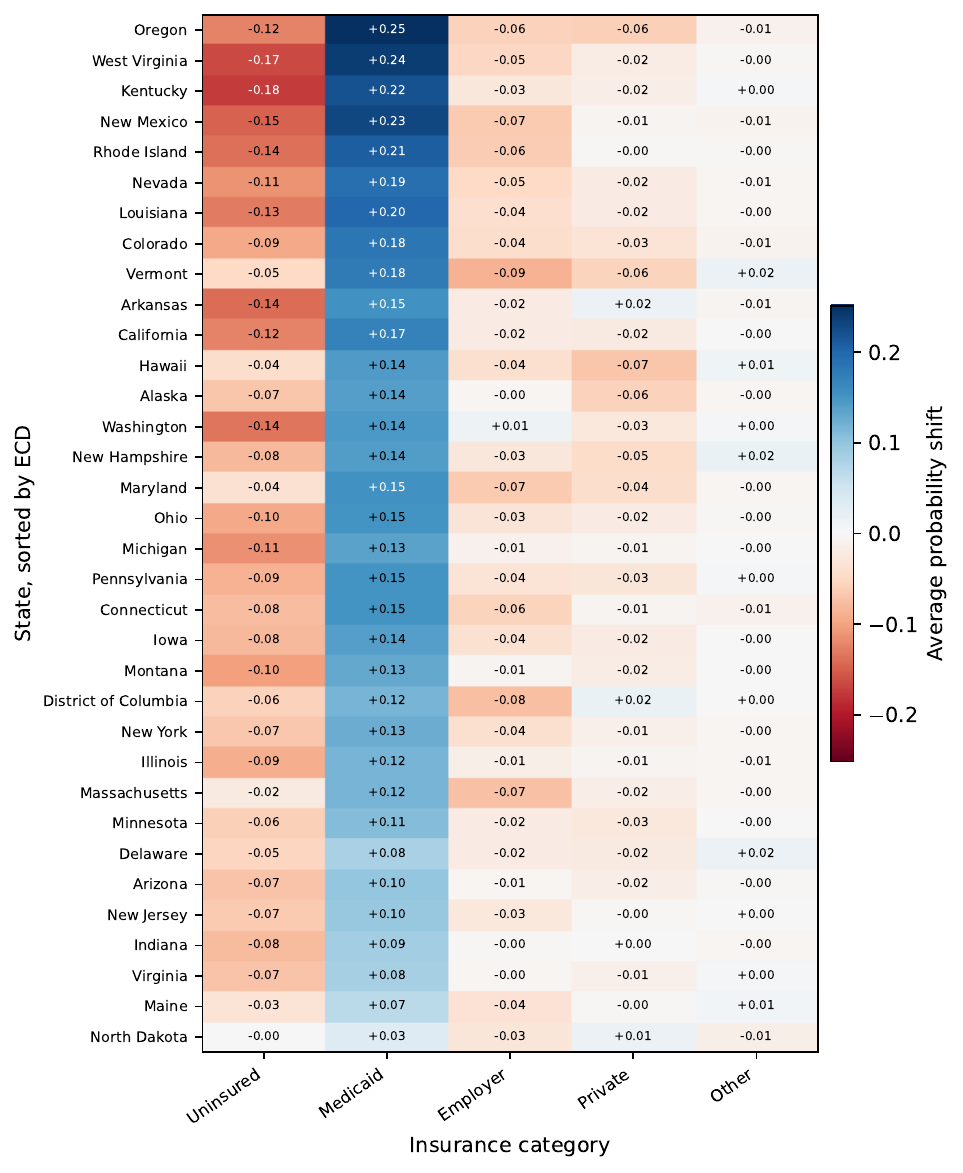}
    \caption{Average post-treatment probability shifts by state and insurance category. The dominant pattern is a shift from uninsured to Medicaid coverage. States are sorted by their average ECD over treated post-expansion years. Positive values indicate categories whose probability increased relative to the counterfactual, while negative values indicate decreases.}\label{fig:medicaid_heatmap}
\end{figure}

Figure~\ref{fig:medicaid_heatmap} shows how the insurance distribution changes across treated states. The dominant pattern is a movement from uninsured coverage to Medicaid coverage, consistent with the averaged results in Table~\ref{tab:acs_probability_shifts}. States such as West Virginia, Kentucky, New Mexico, and Oregon exhibit large reductions in uninsured coverage together with large increases in Medicaid coverage. Other states show smaller Medicaid effects or redistribute probability mass across several categories.

Although state-level ECD is strongly correlated with the Medicaid probability
effect, the heatmap shows that similar Medicaid increases can arise through
different redistributions across the remaining insurance categories. These
state-level differences illustrate one advantage of the proposed distributional
framework. An analysis based on a single outcome, such as the uninsured rate or
Medicaid enrollment alone, would summarize only one component of the treatment
effect and would not distinguish these redistribution patterns. EFSC instead
models the entire multinomial distribution of insurance coverage and therefore
captures how probability mass is redistributed across categories. This
distributional view reveals both the average policy pattern and the
heterogeneity in how different state insurance markets respond to Medicaid
expansion.

Finally, we assess whether the observed distributional change is large relative to a placebo reference distribution. In the synthetic experiments in Section~\ref{subsec:placebo_tests}, placebo assignments were generated by selecting control groups with the same cardinality as the treated group. In the
Medicaid data this direct permutation is not available, since the number of treated expansion states exceeds the number of controls. We therefore use a posterior-predictive placebo procedure that preserves the same
randomization logic while avoiding duplicate control units. We first fit an EFSC generator model only on cells that are not exposed to Medicaid expansion, generate untreated synthetic panels from the fitted latent factorization, and then apply the same no-leakage placebo algorithm to synthetic treated blocks
with the same cardinality and staggered adoption-year schedule as the observed treated states. 

Under the ideal posterior-predictive construction, the predictive mean satisfies
\begin{align}
    \E{\mby^{\text{rep}}|\mby^{\text{ctrl}}}=\EE{q(\Theta|\mby^{\text{ctrl}})}
    {\E{\mby^{\text{rep}}|\Theta}},
\end{align}
where $\mby^{\text{rep}}$ and $\mby^{\text{ctrl}}$ are the synthetic panel and the untreated data, respectively. In the empirical implementation, we use a plug-in approximation that evaluates the fitted factorization at the variational posterior mean, resamples source-state profiles uniformly with
replacement, and draws new multinomial outcomes from the resulting probability vectors. Informally, this predictive procedure generates plausible untreated scenarios, comprising synthetic states and their trajectories of insurance coverage distributions, from which we construct a placebo reference distribution for the observed statistic. In each replication $b$, we independently draw one synthetic untreated panel and one placebo assignment and store the resulting placebo statistic directly, thereby preserving both predictive panel variation and placebo-assignment variation. To the extent that the fitted plug-in predictive distribution approximates the untreated data-generating process, this empirical reference distribution approximates the corresponding placebo null. Full algorithmic details are given in
Appendix~\ref{app:medicaid_posterior_predictive_placebo}.

The test statistic is the change in the posterior-mean approximation to the average ECD,
\begin{align}
\Delta_{\mathrm{KL}}=
\mathrm{ECD}^{\mathrm{post}}-
\mathrm{ECD}^{\mathrm{pre}},
\end{align}
where both terms average the probability-vector KL divergence over the relevant state-year cells. For the observed treated states, the post period consists of treated post-expansion cells and the pre period consists of the corresponding treated pre-expansion cells. Each statistic is constructed from four separate EFSC fits---target and observed fits for the post block and target and observed fits for the pre block---using the no-leakage construction described in Appendix~\ref{app:medicaid_posterior_predictive_placebo}. 

For each placebo replication \(b\), the same calculation is performed after independently generating one synthetic untreated panel and assigning the observed staggered treatment schedule to a same-cardinality block of synthetic untreated states. The resulting placebo statistic is stored directly,
\begin{align}
\Delta_{\mathrm{KL}}^{(b)}=
\mathrm{ECD}^{\mathrm{post},(b)}-
\mathrm{ECD}^{\mathrm{pre},(b)}.
\end{align}
The collection $\{\Delta_{\mathrm{KL}}^{(b)}\}_{b=1}^{5000}$ forms the approximate posterior-predictive placebo distribution.

\begin{figure}
    \centering
    \includegraphics[width=0.6\linewidth]{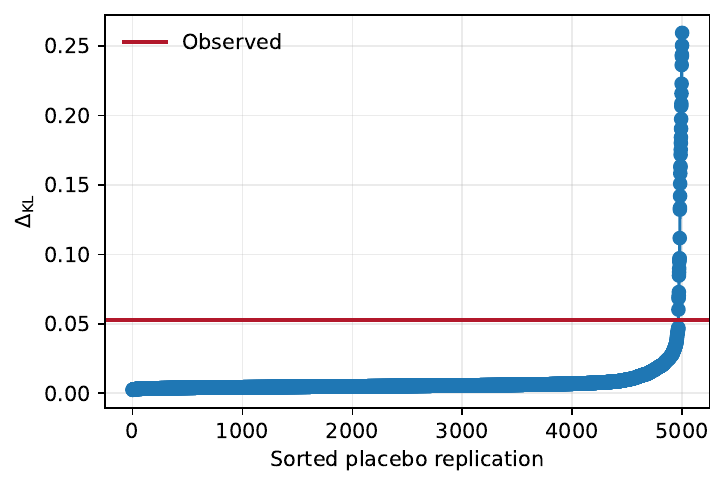}
    \caption{Posterior-predictive placebo distribution for the Medicaid application. The observed statistic lies in the upper-tail of the distribution, with $\hat{p}=0.0072$. The ordered points report $\Delta_{\mathrm{KL}}$ across $5,000$ joint placebo replications generated from the untreated EFSC predictive model, with each replication combining one independently generated synthetic panel and one independently drawn same-cardinality placebo assignment.}\label{fig:medicaid_posterior_predictive_placebo}
\end{figure}

Figure~\ref{fig:medicaid_posterior_predictive_placebo} shows that the observed Medicaid-expansion statistic,
$\Delta_{\mathrm{KL}}^{\mathrm{obs}}=0.0525$, lies in the upper tail of the posterior-predictive placebo distribution. The corresponding upper-tail placebo probability is approximately $\hat p=0.0072$. Even after accounting for predictive variation under the fitted untreated latent factorization and for the staggered adoption structure, the distributional change observed after Medicaid expansion is large relative to the changes produced by joint synthetic-panel and same-cardinality placebo-assignment draws under the untreated model. Together with the category-specific shifts in Figure~\ref{fig:medicaid_heatmap}, this result is consistent with Medicaid expansion being associated with a broad change in the distribution of
insurance coverage, with increased Medicaid enrollment representing the dominant margin of that change.

\section{Conclusion}
We introduced exponential family synthetic controls (EFSC), a probabilistic extension of synthetic control methods for disaggregated panel data. Instead of aggregating each unit-time cell to a scalar outcome, EFSC treats each cell as a dataset generated from an exponential-family distribution and uses a latent factorization to structure the corresponding natural parameters. This allows the model to reconstruct counterfactual data-generating distributions and to define causal estimands beyond mean effects, including posterior natural-parameter effects and KL-based expected causal divergences. We estimate the model with black-box variational inference, yielding a flexible implementation that can be applied across exponential-family observation models and scale to large datasets. 

In synthetic experiments, EFSC recovers exponential tilts across binary, count, continuous, and non-Gaussian panels, and performs well under structured latent interventions and heavy-tailed perturbations. In the Medicaid analysis, increased Medicaid enrollment is the dominant margin of the estimated distributional change, and state-level ECD is consequently strongly associated with the Medicaid probability effect. The full-category analysis reveals heterogeneity across states in how probability mass is redistributed among uninsured, employer, private, and other coverage. This illustrates how distributional analysis complements scalar treatment-effect summaries even when one outcome accounts for most of the observed variation.

There are several avenues for further research. The experiments cover a range of canonical exponential families, but broader benchmarks are needed for richer observation models, including additional multivariate responses, overdispersion, zero inflation, censoring, and other features common in applied panel data. The factorization used here is simple, combining additive unit and time effects with low-rank interactions; future work could incorporate dynamic latent factors, hierarchical priors, nonlinear embeddings, covariates, or graph-based structure. On the causal side, KL divergence is only one way to compare treated and counterfactual distributions. Other discrepancies, such as Wasserstein distances, may be preferable when tail behavior or support changes are central. Finally, EFSC has potential applications in areas ranging from biomedicine to
economics and public policy, where interventions alter complex distributions rather than only average outcomes.

\acks{This work was partially supported by the Wallenberg AI, Autonomous Systems and Software Program (WASP) funded by the Knut and Alice Wallenberg Foundation, Sweden.}

\appendix
\section{Notation}
\begin{table}[H]
\centering
\begin{tabular}{cll}
\hline
Notation & Description & Dimension/Domain \\
\hline
$i$ & Unit index & $i\in\{1,\ldots,N\}$ \\
$j$ & Time index & $j\in\{1,\ldots,T\}$ \\
$k$ & Observation index & $k\in\{1,\ldots,m_{ij}\}$ \\
$m_{ij}$ & Number of observations in cell $(i,j)$ & $m_{ij}\in\mathbb N$ \\
$P$ & Number of natural parameters & $P\in\mathbb N$ \\
$\mbeta_{ij}$ & Natural parameter vector for unit $i$ and time $j$ &
$\mbeta_{ij}\in\mathcal H$ \\
$\mathcal H$ & Natural parameter space of the EF &
$\mathcal H\subseteq\mathbb R^P$ \\
$\mbz_{ij}$ & Unconstrained predictor vector for unit $i$ and time $j$ &
$\mbz_{ij}\in\mathbb R^P$ \\
$Z_p$ & Unconstrained predictor matrix for component $p$ &
$Z_p\in\mathbb R^{N\times T}$ \\
$h$ & Constraint bijection & $h:\mathbb R^P\rightarrow\mathcal H$ \\
$y_{ijk}$ & $k$-th observation for unit $i$ and time $j$ &
$y_{ijk}\in\mathrm{supp}\;\mathrm{EF}(\mbeta_{ij})$ \\
$\mbalpha_i$ & Unit-specific PMF intercept &
$\mathbb R^P$ \\
$\mbgamma_j$ & Time-specific PMF intercept &
$\mathbb R^P$ \\
$\mbtheta_i$ & Unit-specific latent factors &
$\mathbb R^{P\times r}$ \\
$\mbbeta_j$ & Time-specific latent factors &
$\mathbb R^{P\times r}$ \\
$r$ & Latent-factor rank & $r\in\mathbb N$ \\
$\Theta$ & Collection of model parameters &
$\mathbb R^D$ \\
$D$ & Number of model parameters &
$\mathbb N$ \\
$\mbnu$ & Variational parameters &
$\mathbb R^D\times\mathbb R_{+}^{D}$ \\
$q(\Theta;\mbnu)$ & Variational approximation &
Distribution over $\Theta$ \\
$\Omega$ & Set of panel cells used in a model fit &
$\Omega\subseteq\{1,\ldots,N\}\times\{1,\ldots,T\}$ \\
$\mbtau$ & Intervention effect in natural-parameter space &
$\mathbb R^P$ \\
$\mathcal C$ & Set of control units &
$\mathcal C\subseteq\{1,\ldots,N\}$ \\
$\mathcal T$ & Set of treated units &
$\mathcal T\subseteq\{1,\ldots,N\}$ \\
$\mathcal P$ & Pretreatment periods &
$\mathcal P\subseteq\{1,\ldots,T\}$ \\
$\mathcal Q$ & Post-treatment periods &
$\mathcal Q\subseteq\{1,\ldots,T\}$ \\
$\mathcal Q_i$ & Post-treatment periods for unit $i$ &
$\mathcal Q_i\subseteq\{1,\ldots,T\}$ \\
\hline
\end{tabular}
\caption{Notation used throughout the paper.}
\end{table}

\newpage
\section{Methodological Details}
\subsection{Multi-parameter and Constrained PMF Factorization}\label{app:multi-pmf}
We first describe the factorization for an unrestricted natural-parameter
space. For a one-parameter EF with $\mathcal{H}=\mathbb{R}$, collect the
natural parameters in Equation~\eqref{eq:pmf} into the matrix:
\begin{align}
    \mbeta=\mbalpha\otimes\mathbf{1}_T^T+\mathbf{1}_N\otimes\mbgamma^T+\mbtheta^T\mbbeta,
\end{align}
where $\mbeta\in \mathbb{R}^{N\times T}$, with effect matrices $\mbtheta\in \mathbb{R}^{r\times N}$ and $\mbbeta\in \mathbb{R}^{r\times T}$, and intercept vectors $\mbalpha\in \mathbb{R}^{N}$ and $\mbgamma\in \mathbb{R}^{T}$. The column vectors of ones, $\mathbf{1}_N$ and $\mathbf{1}_T$, are of dimension $N$ and $T$, respectively, and $\otimes$ is the Kronecker product. For an EF with $P$ natural parameters and unrestricted
natural-parameter space $\mathcal{H}=\mathbb{R}^P$, we stack the
factorizations vertically, resulting in an $NP\times T$ matrix of
natural parameters:
\begin{equation}
\begin{pmatrix}
  \mbeta[1]\\ 
  \mbeta[2]\\
  \vdots\\
  \mbeta[P]
\end{pmatrix}=
\begin{pmatrix}
  \mbalpha[1]\otimes\mathbf{1}_T^T+\mathbf{1}_N\otimes\mbgamma[1]^T+\mbtheta[1]^T\mbbeta[1]\\ 
  \mbalpha[2]\otimes\mathbf{1}_T^T+\mathbf{1}_N\otimes\mbgamma[2]^T+\mbtheta[2]^T\mbbeta[2]\\
  \vdots\\
   \mbalpha[P]\otimes\mathbf{1}_T^T+\mathbf{1}_N\otimes\mbgamma[P]^T+\mbtheta[P]^T\mbbeta[P]
\end{pmatrix}.
\end{equation}
For a constrained natural-parameter space
$\mathcal{H}\subsetneq\mathbb{R}^P$, directly applying an unconstrained
factorization to $\mbeta$ need not produce valid natural parameters.
We therefore place the factorization on an unconstrained predictor
scale. For each component $p=1,\ldots,P$, define
\begin{equation}
    Z_p
    =
    \mbalpha[p]\otimes\mathbf{1}_T^T
    +\mathbf{1}_N\otimes\mbgamma[p]^T
    +\mbtheta[p]^T\mbbeta[p]
    \in\mathbb{R}^{N\times T},
\end{equation}
and let
\begin{equation}
    \mbz_{ij}
    =
    \bigl(z_{ij}[1],\ldots,z_{ij}[P]\bigr)^T,
    \quad
    z_{ij}[p]=Z_p[i,j].
\end{equation}
The natural parameter in cell $(i,j)$ is then
\begin{equation}
    \mbeta_{ij}=h(\mbz_{ij})\in\mathcal{H},
    \quad
    h:\mathbb{R}^P\rightarrow\mathcal{H},
\end{equation}
where $h$ is a fixed, known bijection. When
$\mathcal{H}=\mathbb{R}^P$, we take $h$ to be the identity, and the
construction reduces to the direct natural-parameter factorization
above.

Thus, the low-rank structure is imposed on each unconstrained predictor
matrix $Z_p$. When $h$ is nonlinear, the corresponding natural-parameter
matrix need not itself be low rank; EFSC completes the predictor
matrices and then applies $h$ cellwise.

The component-specific construction contains $D=P(N+T)(r+1)$ scalar
model parameters and allows each unconstrained predictor component to
exhibit a distinct latent structure. A shared factorization can instead
be obtained by reusing effects and latent factors across components.
The unrestricted one-parameter factorization in
Equation~\eqref{eq:pmf} is recovered when $P=1$ and $h$ is the identity.

Table~\ref{tab:constraint-transformations} lists the transformations
used for the EFs considered in this paper.
\begin{table}[H]
\centering
\small
\begin{tabular}{cccc}
\toprule
Distribution & $h(z)$ & $\mathcal H$ & $h^{-1}(\eta)$ \\
\midrule
Bernoulli
& $z$ & $\mathbb R$ & $\eta$ \\
Poisson
& $z$ & $\mathbb R$ & $\eta$ \\
Exponential
& $-\exp(z)$ & $(-\infty,0)$ & $\log(-\eta)$ \\
Laplace, known $\mu$
& $-\exp(z)$ & $(-\infty,0)$ & $\log(-\eta)$ \\
Chi-Squared
& $\exp(z)-1$ & $(-1,\infty)$ & $\log(1+\eta)$ \\
Gaussian, known $\sigma^2$
& $z$ & $\mathbb R$ & $\eta$ \\
Gaussian, unknown $\mu,\sigma^2$
& $\bigl(z[1],-\exp\{z[2]\}\bigr)^T$
& $\mathbb R\times(-\infty,0)$
& $\bigl(\eta[1],\log\{-\eta[2]\}\bigr)^T$ \\
Multinomial, $C$ categories
& $\mbz$ & $\mathbb R^{C-1}$ & $\mbeta$ \\
\bottomrule
\end{tabular}
\caption{Unconstrained-to-natural-parameter transformations used by EFSC.
For the multinomial family, $\mbz\in\mathbb R^{C-1}$ contains the
reference-category logits, with the $C$th logit fixed at zero. Each
transformation is a bijection onto the stated natural-parameter space.}
\label{tab:constraint-transformations}
\end{table}

\subsection{Generative Model}\label{app:generative_model}
The EFSC model defines a generative process over a panel of datasets indexed by
units $i=1,\ldots,N$ and time periods $j=1,\ldots,T$. For an exponential
family with $P$ natural parameters, the model generates the data as follows:

\begin{enumerate}
\item For each parameter component $p=1,\ldots,P$, draw unit-specific effects
\begin{align}
\alpha_i[p]\sim\mathcal N(0,\sigma_\alpha^2),
\quad i=1,\ldots,N.
\end{align}

\item For each parameter component $p=1,\ldots,P$, draw time-specific effects
\begin{align}
\gamma_j[p]\sim\mathcal N(0,\sigma_\gamma^2),
\quad j=1,\ldots,T.
\end{align}

\item For each parameter component $p=1,\ldots,P$, draw latent factors
\begin{align}
\mbtheta_i[p]
&\sim\mathcal N(\mathbf 0,\sigma_\theta^2 I_r),
\quad i=1,\ldots,N,\\
\mbbeta_j[p]
&\sim\mathcal N(\mathbf 0,\sigma_\beta^2 I_r),
\quad j=1,\ldots,T.
\end{align}

\item Construct the unconstrained predictors
\begin{align}
z_{ij}[p]
=
\alpha_i[p]+\gamma_j[p]+\mbtheta_i[p]^T\mbbeta_j[p],
\quad
(i,j)\in\{1,\ldots,N\}\times\{1,\ldots,T\},
\end{align}
and let
\begin{align}
\mbz_{ij}
&=
\left(
z_{ij}[1],\ldots,z_{ij}[P]
\right)^T,\\
\mbeta_{ij}
&=
h(\mbz_{ij})\in\mathcal H,
\end{align}
where $h:\mathbb R^P\rightarrow\mathcal H$ is the appropriate transformation
into the natural-parameter space and is the identity when the space is
unconstrained.

\item For each cell $(i,j)$, fix or externally generate the number of
observations $m_{ij}$. In the synthetic experiments, cell sizes are either
fixed or generated from a specified distribution, such as
\begin{align}
m_{ij}\sim 1+\operatorname{Poisson}(\lambda_m).
\end{align}
EFSC conditions on the realized cell sizes during estimation.

\item For each observation $k=1,\ldots,m_{ij}$, generate
\begin{align}
y_{ijk}
\mid
\mbeta_{ij}
\stackrel{\mathrm{iid}}{\sim}
\mathrm{EF}(\mbeta_{ij}).
\end{align}
\end{enumerate}

The construction above allows component-specific factorizations. A shared
factorization is obtained by reusing the same effects and latent factors across
components before applying the constraint transformation; this is the
construction used to generate the Gaussian benchmark panels. The quantities
$\sigma_\alpha$, $\sigma_\gamma$, $\sigma_\theta$, and $\sigma_\beta$ are
experiment-specific data-generating scales and need not equal the prior scales
used during estimation.

The resulting collection of datasets
\begin{align}
\mby=
\left\{
y_{ijk}:
i=1,\ldots,N,\;
j=1,\ldots,T,\;
k=1,\ldots,m_{ij}
\right\}
\end{align}
forms a panel of disaggregated observations whose distribution is governed by a low-rank probabilistic matrix factorization on the unconstrained predictor scale.

\subsection{Exponential Tilting}\label{app:exptilt}
Let $p(y;\mbtheta)$ be the distribution of a univariate random variable $Y$, parameterized by $\mbtheta$. The $\mbtau$-exponentially tilted 
distribution of $Y$, $p_{\mbtau}(y;\mbtheta)$, is defined as
\begin{equation}
    p_{\mbtau}(y;\mbtheta)=\frac{\exp\{\mbtau^T\mathbf{t}(y)\}p(y;\mbtheta)}
    {\psi_{\mbtheta}(\mbtau)},
\end{equation}
where $\psi_{\mbtheta}(\mbtau) = \mathbb E_{Y\sim p(\cdot;\mbtheta)} [\exp\{\mbtau^T\mathbf{t}(Y)\}]$ is the moment-generating function of $\mathbf{t}(Y)$ evaluated at $\mbtau$, and $\mbtau$ is chosen such that the MGF exists.

For an exponential-family distribution,
\begin{align}
p(y;\mbeta)
&=
\exp\left\{
\mbeta^T \mathbf{t}(y)
-a(\mbeta)
+c(y)
\right\},
\end{align}
where $a(\mbeta)$ is the log-partition function and $c(y)$ is the
log-carrier term. The normalizing constant of the exponential tilt is
\begin{align}
\psi_{\mbeta}(\mbtau)
&=
\mathbb E_{\mbeta}
\left[
\exp\left\{\mbtau^T \mathbf{t}(Y)\right\}
\right] \\
&=
\exp\left\{
a(\mbeta+\mbtau)-a(\mbeta)
\right\}.
\end{align}
Consequently, the tilted distribution is
\begin{align}
p_{\mbtau}(y;\mbeta)
&=
\frac{
\exp\left\{\mbtau^T \mathbf{t}(y)\right\}
p(y;\mbeta)
}{
\psi_{\mbeta}(\mbtau)
} \\
&=
\exp\left\{
(\mbeta+\mbtau)^T \mathbf{t}(y)
-a(\mbeta+\mbtau)
+c(y)
\right\}.
\end{align}
Therefore, provided that $\mbeta+\mbtau\in\mathcal H$, exponential tilting
produces another member of the same exponential family with natural parameter $\tilde{\mbeta}=\mbeta+\mbtau$. Table~\ref{tab:expfam_tilted} presents examples of exponential tilts considered in this paper for one-parameter EFs.

For a Gaussian distribution with unknown mean and variance ($P=2$), the natural parameters are $\eta[1]=\mu/\sigma^2$ and $\eta[2]=-1/(2\sigma^2)$.
Applying a two-dimensional exponential tilt $\mbtau=(\tau_1,\tau_2)^T$ yields $\tilde{\eta}[1]=\eta[1]+\tau_1$ and $\tilde{\eta}[2]=\eta[2]+\tau_2$.
The corresponding model parameters are $\tilde{\mu}=-(\mu+\sigma^2\tau_1)/(2\sigma^2\tau_2-1)$ and $\tilde{\sigma}^2=-\sigma^2/(2\sigma^2\tau_2-1),$
provided that $\tilde{\eta}[2]<0$.

\renewcommand{\arraystretch}{1}
\begin{table}[t]
    \centering
    \begin{tabular}{cccccc}
        \toprule
        Distribution & $t(y)$ & \raisebox{-0.1ex}{$\theta$}& \raisebox{-0.05ex}{$\tilde{\eta}$}& \raisebox{-0.4ex}{$\tilde{\theta}$} & $\tilde{\eta}\;\text{domain}$ \\
       \midrule
        Bernoulli & $y$ & $\pi$ & $\tau+\operatorname{logit}\pi$ & $(\pi\exp\tau)/(1-\pi+\pi\exp\tau)$& $\mathbb{R}$\\
        Poisson & $y$ & $\lambda$ & $\tau+\log\lambda$ & $\lambda \exp\tau$& $\mathbb{R}$\\
        Exponential & $y$ & $\beta$ & $\tau-\beta$ & $\beta-\tau$& $(-\infty,0)$\\
        Laplace, known $\mu$ & $|y-\mu|$ & $\beta$ & $\tau-1/\beta$ & $\beta/(1-\beta\tau)$& $(-\infty,0)$ \\
        Chi-Squared & $\log y$ & $\alpha$  & $\tau+\alpha/2-1$ & $\alpha+2\tau$& $(-1,\infty)$\\
        Gaussian, known $\sigma^2$ & $y/\sigma$ & $\mu$ & $\tau+\mu/\sigma$& $\mu+\sigma\tau$& $\mathbb{R}$\\
\bottomrule
    \end{tabular}
    \caption{Univariate exponential-family distributions used in our experiments. The sufficient statistic is $t(y)$, the original parameter is $\theta$, and the tilted natural parameter is $\tilde{\eta}=\eta(\theta)+\tau$, with corresponding inverse $\tilde{\theta}=\theta(\tilde{\eta})$. The tilt $\tau$ is assumed to preserve the natural-parameter space.}\label{tab:expfam_tilted}
\end{table}

Under exponential tilting, the expected causal divergence (ECD) in Equation~\eqref{eq:causal_effect_posterior_kl} is closely related to the
geometry of the log-partition function. Figure~\ref{fig:kl_tilt} provides a graphical interpretation of this relationship. We orient the KL from the treated distribution to the counterfactual distribution because it averages the log-density ratio with respect to the treated data-generating process, which is directly informed by observed post-treatment outcomes, while $\eta$ defines the reconstructed no-treatment reference. Geometrically, for $\tilde\eta=\eta+\tau$, this divergence is the gap between $a(\eta)$ and the first-order Taylor approximation of $a$ around $\tilde\eta$, evaluated at $\eta$. Reversing the KL arguments instead anchors the tangent at $\eta$, corresponding to $\tau=0$,
and evaluates the gap at $\tilde\eta$.
\begin{figure}[t]
    \centering
    \includegraphics[width=0.45\linewidth]{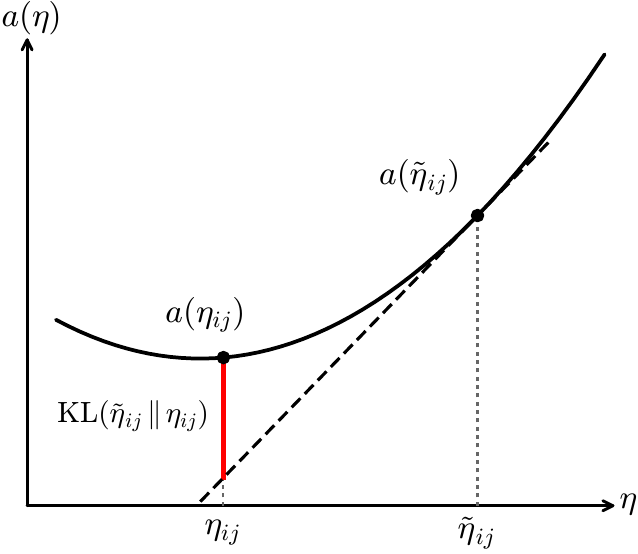}
    \caption{A linear perturbation shifts the natural parameter from $\eta_{ij}$ to $\tilde{\eta}_{ij}$ and induces a KL divergence between the
counterfactual and treated distributions. For the treated-to-counterfactual direction used in the ECD, the dashed line is tangent to the log-partition function at the treated parameter $\tilde{\eta}_{ij}$, and the red segment shows the gap between this tangent and $a(\eta_{ij})$ at the reconstructed counterfactual. This gap equals $\operatorname{KL}\!\left(p(y;\tilde{\eta}_{ij})\,\|\,p(y;\eta_{ij})\right)$. Adapted from  \cite{wainwright2008graphical} and \cite{efron2022exponential}.}
    \label{fig:kl_tilt}
\end{figure}

\subsection{Counterfactual Reconstruction}
The EFSC model can be interpreted as a probabilistic matrix-completion
procedure operating on unconstrained predictors that are mapped into the
natural parameter space of an EF. Let $\mathcal T$ and $\mathcal C$ denote the treated and control units, respectively, and let $\mathcal P$ and $\mathcal Q$ denote the pre- and post-treatment periods. We define the counterfactual training set
\begin{align}
\Omega_{\mathrm{obs}}
=
\bigl(\mathcal C\times(\mathcal P\cup\mathcal Q)\bigr)
\cup
\bigl(\mathcal T\times\mathcal P\bigr),
\end{align}
and the target set
\begin{align}
\Omega_{\mathrm{tgt}}
=
\mathcal T\times\mathcal Q.
\end{align}
Let $\mby_{\Omega}$ denote the collection of datasets indexed by a set
$\Omega$. For clarity, we present the reconstruction using scalar notation for
a one-parameter EF; the multi-parameter case follows the stacked construction
in Appendix~\ref{app:multi-pmf}.

After fitting the PMF model using the datasets
$\mby_{\Omega_{\mathrm{obs}}}$, draws
$\Theta^{(s)}\sim
q_{\mathrm{ctrl}}(\Theta;\hat{\mbnu}_{\mathrm{ctrl}})$
from the fitted variational posterior induce an approximation to the posterior
distribution over the untreated natural parameters in the target set,
\begin{align}
q_{\mathrm{ctrl}}(\eta_{\Omega_{\mathrm{tgt}}}^{\mathrm{ctrl}}
\mid\mby_{\Omega_{\mathrm{obs}}})
\approx
p(\eta_{\Omega_{\mathrm{tgt}}}^{\mathrm{ctrl}}
\mid\mby_{\Omega_{\mathrm{obs}}}).
\end{align}
For a given variational draw $\Theta^{(s)}$, the corresponding counterfactual
natural parameters are reconstructed through the PMF model,
\begin{align}
z_{ij}^{(s)}
&=
\alpha_i^{(s)}+\gamma_j^{(s)}
+
(\mbtheta_i^{(s)})^\top\mbbeta_j^{(s)},\\
\eta_{ij}^{\mathrm{ctrl},(s)}
&=
h(z_{ij}^{(s)}),
\;
(i,j)\in\Omega_{\mathrm{tgt}},
\end{align}
where $h$ maps the unconstrained predictor into the natural-parameter space
and is the identity when no constraint transformation is required.

The collection of draws
$\{\eta_{ij}^{\mathrm{ctrl},(s)}\}_{s=1}^S$ approximates the posterior distribution of the untreated natural parameter and
therefore defines the induced counterfactual predictive distribution
\begin{align}
p(y_{ij}^{\mathrm{ctrl}}
\mid
\mby_{\Omega_{\mathrm{obs}}}
)
\approx
\int p(y_{ij}^{\mathrm{ctrl}}
\mid
\eta_{ij}^{\mathrm{ctrl}})
q_{\mathrm{ctrl}}(\eta_{ij}^{\mathrm{ctrl}}
\mid\mby_{\Omega_{\mathrm{obs}}})
\,\mathrm d\eta_{ij}^{\mathrm{ctrl}}.
\end{align}

For treated datasets $(i,j)\in\Omega_{\mathrm{tgt}}$, causal inference is
based on comparing the fitted variational distributions
\begin{align}
q_{\mathrm{ctrl}}
(\eta_{ij}^{\mathrm{ctrl}}
\mid
\mby_{\Omega_{\mathrm{obs}}})
\end{align}
and
\begin{align}
q_{\mathrm{treat}}
(\eta_{ij}^{\mathrm{treat}}
\mid
\mby_{\Omega_{\mathrm{tgt}}}).
\end{align}
The expected causal effect (ECE) and expected causal divergence (ECD)
introduced in Section~\ref{subsec:method_overview} are the posterior expectations
of $\eta_{ij}^{\mathrm{treat}}-\eta_{ij}^{\mathrm{ctrl}}$
and $\mathrm{KL}\bigl(
p(y;\eta_{ij}^{\mathrm{treat}})
\|p(y;\eta_{ij}^{\mathrm{ctrl}})
\bigr)$, respectively. When posterior sampling is used, each MC iteration independently draws the complete treated and counterfactual factorization parameters from their respective variational distributions and reconstructs the corresponding natural-parameter surfaces. All reported ECD calculations instead use the plug-in KL divergence evaluated at natural parameters reconstructed from the respective variational posterior means.

\subsection{Proof of the Identification Theorem}\label{app:proof_theorem}
This subsection provides the additional notation and assumptions required for the identification result for a reparameterized $P$-dimensional EF, followed by the formal definition of identification and the proof of the general Theorem~\ref{theo:identification_two}. We then state several corollaries and remarks concerning the identified counterfactual distributions and the population targets of the ECE and ECD.

For each unconstrained-predictor component $p\in\{1,\ldots,P\}$, define the
complete untreated predictor matrix
\begin{align}
    \mbz_{ij}=h^{-1}(\mbeta_{ij}),\quad Z_p=\mbz[p]\in\mathbb{R}^{N\times T}.
\end{align}
After ordering control units before treated units and pretreatment periods
before post-treatment periods, partition $Z_p$ as
\begin{align}
    Z_p=\begin{pmatrix}
A_p & B_p\\
C_p & D_p
\end{pmatrix},
\end{align}
where
\begin{align*}
A_p&=Z_p[\mathcal{C},\mathcal{P}]
\in\mathbb{R}^{|\mathcal{C}|\times|\mathcal{P}|},
&
B_p&=Z_p[\mathcal{C},\mathcal{Q}]
\in\mathbb{R}^{|\mathcal{C}|\times|\mathcal{Q}|},\\
C_p&=Z_p[\mathcal{T},\mathcal{P}]
\in\mathbb{R}^{|\mathcal{T}|\times|\mathcal{P}|},
&
D_p&=Z_p[\mathcal{T},\mathcal{Q}]
\in\mathbb{R}^{|\mathcal{T}|\times|\mathcal{Q}|}.
\end{align*}
Thus, $A_p$, $B_p$, and $C_p$ contain untreated predictors associated with
observed cells, whereas $D_p$ is the missing treated--post-treatment
counterfactual predictor block.

\begin{assumption}[Exponential-family observation model]
\label{ass:ef}
Fix a known exponential family whose support and
quantities $(c,\mathbf t,a,\mathcal H)$ are common to every cell and every
admissible specification. Throughout, the natural parameters and factor
quantities are fixed unknown constants, and the observed-data distribution is their
sampling distribution conditional on the treatment design and cell sizes. Assume
$m_{ij}\geq 1$ for every panel cell. The untreated potential outcomes satisfy
\begin{align}
  Y_{ijk}
\overset{\mathrm{iid}}\sim\textsc{expfam}(\mbeta_{ij}),
\quad k=1,\ldots,m_{ij},
\end{align}
for every cell $(i,j)$. For treated post-treatment cells
$(i,j)\in\Omega_{\text{tgt}}$, the treated potential outcomes belong to the
same family and satisfy
\begin{align}
    \tilde Y_{ijk}
\overset{\mathrm{iid}}\sim
\textsc{expfam}(\tilde{\mbeta}_{ij}),
\quad
k=1,\ldots,m_{ij}.
\end{align}
The density or probability mass function is
\begin{align}
    p(y\s\mbeta)=\exp\left\{
\mbeta^\top\mathbf t(y)-a(\mbeta)+c(y)
\right\},
\quad
\mbeta\in\mathcal H\subseteq\mathbb{R}^P,
\end{align}
where $\mathbf t(y)$ are the sufficient statistics, $a(\mbeta)$ is the
log-partition function, and $c(y)$ is the log-carrier term.
\end{assumption}

This assumption defines the distributional observation layer of EFSC: each
unit--time cell contains repeated observations from one exponential-family
distribution, rather than a single aggregated outcome. The complete untreated
surface is indexed by $\mbeta_{ij}$, while $\tilde{\mbeta}_{ij}$ indexes
only the observed perturbation in treated post-treatment cells.

\begin{assumption}[Identifiable natural parameterization]
\label{ass:ef-identifiable}
The fixed exponential-family representation is identifiable: the map from
natural parameters to distributions is injective,
\begin{align}
    p(\,\cdot\s\mbeta)=p(\,\cdot\s\mbeta')
\quad\Longrightarrow\quad
\mbeta=\mbeta'.
\end{align}
A standard sufficient condition is minimality, meaning that the sufficient
statistics contain no nontrivial affine dependence. Therefore, the population
distribution of an observed cell uniquely determines its natural parameter.
\end{assumption}
Minimal canonical exponential-family representations have the required
injectivity property \citep{wainwright2008graphical}. All the EFs used in our
synthetic experiments are minimal. The Medicaid multinomial model uses $C-1$
reference-category logits, removing the common-shift redundancy of $C$
unrestricted logits.

\begin{assumption}[Stable untreated predictor factorization]
\label{ass:factorization}
There is a fixed, known bijection $h:\mathbb{R}^P\rightarrow\mathcal H$ such that, for every cell,
\begin{align}
    \mbeta_{ij}=h(\mbz_{ij}),
    \quad
    \mbz_{ij}=\left(z_{ij}[1],\ldots,z_{ij}[P]\right)^\top.
\end{align}
For each component $p\in\{1,\ldots,P\}$, the complete untreated predictor
surface satisfies
\begin{align}
    z_{ij}[p]=\alpha_i[p]+\gamma_j[p]+
\mbtheta_i[p]^\top\mbbeta_j[p],
\quad
\mbtheta_i[p],\mbbeta_j[p]\in\mathbb{R}^r,
\end{align}
for every unit $i$ and period $j$, including the unobserved target cells
$(i,j)\in\mathcal{T}\times\mathcal{Q}$.

Equivalently,
\begin{align}
    Z_p=\mbalpha[p]\otimes\mathbf{1}_T^T
    +\mathbf{1}_N\otimes\mbgamma[p]^T
    +\mbtheta[p]^T\mbbeta[p],
\end{align}
where $\mbtheta[p]\in\mathbb{R}^{r\times N}$ and
$\mbbeta[p]\in\mathbb{R}^{r\times T}$, and therefore
$\rank(Z_p)\le r+2$.
\end{assumption}

This is a structural restriction: the untreated unconstrained predictors,
not the aggregated outcomes, share stable unit, time, and low-rank interaction
structure. When $h$ is nonlinear, the corresponding natural-parameter matrix
need not be low rank; completion is therefore performed on $Z_p$ and $h$ is
applied afterward. In the Medicaid model, $h$ is the identity and each
reference-category logit has its own state effects, year effects, and latent
state--year interaction, whose untreated structure is assumed to continue
after expansion.

\begin{assumption}[Anchor-block rank condition]
\label{ass:anchor}
For each component $p$, let
\begin{align}
    k_p=\rank(Z_p).
\end{align}
The control--pretreatment anchor block contains all effective directions
of the complete untreated predictor matrix:
\begin{align}
    \rank(A_p)=\rank(Z_p)=k_p.
\end{align}
This condition implies the necessary dimensional inequalities
\begin{align}
    k_p\le |\mathcal{C}|,
    \quad
    k_p\le |\mathcal{P}|,
\end{align}
but it does not require $|\mathcal{C}|\ge|\mathcal{T}|$.
\end{assumption}

\begin{definition}[Population identification from the observed-data distribution]\label{def:identification}
Let $\mathfrak M$ denote the class of full-data specifications satisfying
Assumptions~\ref{ass:sutva}--\ref{ass:anchor}. A specification
$\mathcal M\in\mathfrak M$ includes the untreated cell distributions
$P_{ij}$, the treated post-treatment cell distributions
$\tilde P_{ij}$, their natural parameters $\mbeta_{ij}$ and
$\tilde{\mbeta}_{ij}$, the untreated predictors
$\mbz_{ij}=h^{-1}(\mbeta_{ij})$, and the joint distribution of the complete
observed panel. Let $\mathbb P_{\mathcal M}^{\text{obs}}$ denote the joint
observed-data distribution under specification $\mathcal M$.
The untreated counterfactual block $\mbeta_{\mathcal{T},\mathcal{Q}}$
is identified if, for any two specifications
$\mathcal M_1,\mathcal M_2\in\mathfrak M$,
\begin{align}
  \mathbb P_{\mathcal M_1}^{\text{obs}}
=
\mathbb P_{\mathcal M_2}^{\text{obs}}
\quad\Longrightarrow\quad
\mbeta_{\mathcal{T},\mathcal{Q}}(\mathcal M_1)
=
\mbeta_{\mathcal{T},\mathcal{Q}}(\mathcal M_2).
\end{align}
\end{definition}
The definition states that any two admissible data-generating
specifications that are observationally indistinguishable must agree on the
missing untreated parameters, i.e. they have the same counterfactual. Because
$h$ is a fixed bijection, this is equivalent to identification of the
corresponding untreated predictor block $\mbz_{\mathcal{T},\mathcal{Q}}$.

\begin{theorem}[Unique identification of the counterfactual natural parameters]
\label{theo:identification_two}
Suppose Assumptions~\ref{ass:sutva}--\ref{ass:anchor} hold. Then, for each
predictor component $p\in\{1,\ldots,P\}$, the missing counterfactual
predictor block
\begin{align}
    D_p=Z_p[\mathcal{T},\mathcal{Q}]
\end{align}
is identified from the observed-data distribution. In particular,
\begin{align}
    \boxed{D_p=C_pA_p^\dagger B_p,}
\end{align}
where $A_p^\dagger$ denotes the Moore--Penrose pseudoinverse of $A_p$.
Equivalently,
\begin{align}
    \boxed{Z_p[\mathcal{T},\mathcal{Q}]=
Z_p[\mathcal{T},\mathcal{P}]
\left(Z_p[\mathcal{C},\mathcal{P}]\right)^\dagger
Z_p[\mathcal{C},\mathcal{Q}].}
\end{align}
Hence $\mbz_{ij}$, and therefore the vector-valued counterfactual natural
parameter $\mbeta_{ij}=h(\mbz_{ij})\in\mathcal H$, is identified for every
$(i,j)\in\mathcal{T}\times\mathcal{Q}$.
\end{theorem}

\begin{proof}
Fix a component $p$ and suppress the subscript to simplify notation:
\begin{align}
    Z=
\begin{pmatrix}
A&B\\
C&D
\end{pmatrix},
\quad
\rank(Z)=\rank(A)=k.
\end{align}
The proof has two parts. The first identifies the observed untreated predictor
blocks $A$, $B$, and $C$ from the observed-data distribution. The second
proves that these three blocks uniquely determine the counterfactual predictor
block $D$.

\medskip
\noindent
\textbf{Part I: the observed-data distribution identifies $A$, $B$, and $C$.}

By Assumptions~\ref{ass:sutva}--\ref{ass:ef}, every cell
$(i,j)\in\Omega_{\text{unt}}$ has observed marginal distribution
\begin{align}
    Y_{ijk}^{\text{obs}}\sim P_{ij},
\quad
P_{ij}=p(\,\cdot\s\mbeta_{ij}).
\end{align}
Because $m_{ij}\geq1$, every such cell contributes an observed marginal.
Because a joint distribution determines each of its marginals, the population
observed-data distribution determines these cell distributions. By
Assumption~\ref{ass:ef-identifiable}, equality of exponential-family cell
distributions implies equality of their natural parameters. Since $h$ is a
known bijection under Assumption~\ref{ass:factorization}, each identified
natural parameter uniquely determines
\begin{align}
    \mbz_{ij}=h^{-1}(\mbeta_{ij}).
\end{align}
Consequently, the observed-data distribution uniquely determines
\begin{align}
    A=Z[\mathcal{C},\mathcal{P}],
\quad
B=Z[\mathcal{C},\mathcal{Q}],
\quad
C=Z[\mathcal{T},\mathcal{P}].
\end{align}

\medskip
\noindent
\textbf{Part II: the low-rank predictor structure uniquely determines $D$.}

Because $\rank(Z)=k$, there exists a rank factorization
\begin{align}
    Z=UV^\top,
\end{align}
where $U\in\mathbb{R}^{N\times k}$ and $V\in\mathbb{R}^{T\times k}$ both
have full column rank $k$. Partition these factors conformably with $Z$:
\begin{align}
    U=
\begin{pmatrix}
U_{\mathcal{C}}\\
U_{\mathcal{T}}
\end{pmatrix},
\quad
V=
\begin{pmatrix}
V_{\mathcal{P}}\\
V_{\mathcal{Q}}
\end{pmatrix}.
\end{align}
Then
\begin{align}
    A=U_{\mathcal{C}}V_{\mathcal{P}}^\top,
\quad
B=U_{\mathcal{C}}V_{\mathcal{Q}}^\top,
\quad
C=U_{\mathcal{T}}V_{\mathcal{P}}^\top,
\quad
D=U_{\mathcal{T}}V_{\mathcal{Q}}^\top.
\end{align}

By Assumption~\ref{ass:anchor}, $\rank(A)=k$. Since
$A=U_{\mathcal{C}}V_{\mathcal{P}}^\top$ and each factor has rank at most
$k$, it follows that
\begin{align}
    \rank(U_{\mathcal{C}})=k,
    \quad
    \rank(V_{\mathcal{P}})=k.
\end{align}
Thus, $U_{\mathcal{C}}$ has full column rank and
$V_{\mathcal{P}}^\top$ has full row rank. For a product of a
full-column-rank matrix and a full-row-rank matrix, the reverse-order identity
holds \citep{strang2019linear}:
\begin{align}
    A^\dagger
=\left(U_{\mathcal{C}}V_{\mathcal{P}}^\top\right)^\dagger
=\left(V_{\mathcal{P}}^\top\right)^\dagger
(U_{\mathcal{C}})^\dagger.
\end{align}
Moreover,
\begin{align}
    (U_{\mathcal{C}})^\dagger U_{\mathcal{C}}=I_k,
\quad
V_{\mathcal{P}}^\top
\left(V_{\mathcal{P}}^\top\right)^\dagger=I_k.
\end{align}
Using the expressions for $B$ and $C$, we obtain
\begin{align}
CA^\dagger B
&=
\left(U_{\mathcal{T}}V_{\mathcal{P}}^\top\right)
\left(V_{\mathcal{P}}^\top\right)^\dagger
(U_{\mathcal{C}})^\dagger
(U_{\mathcal{C}}V_{\mathcal{Q}}^\top)\\
&=
U_{\mathcal{T}}
[V_{\mathcal{P}}^\top
\left(V_{\mathcal{P}}^\top\right)^\dagger]
\left[(U_{\mathcal{C}})^\dagger U_{\mathcal{C}}\right]
V_{\mathcal{Q}}^\top\\
&=
U_{\mathcal{T}}I_kI_kV_{\mathcal{Q}}^\top\\
&=
U_{\mathcal{T}}V_{\mathcal{Q}}^\top\\
&=D.
\end{align}
Therefore,
\begin{align}
    D=CA^\dagger B.
\end{align}
The right-hand side depends only on the identified blocks $A$, $B$, and $C$,
and the preceding argument applies to every admissible completion satisfying
$\rank(Z)=\rank(A)$. Hence every admissible completion has the same block
$D=CA^\dagger B$. More explicitly, if two admissible specifications induce
the same observed-data distribution, Part I implies that they have the same
blocks $A$, $B$, and $C$, and the formula above then implies that they have
the same counterfactual predictor block $D$.

Because the argument applies to every component $p=1,\ldots,P$, the full
vector $\mbz_{ij}$ is identified on
$\mathcal{T}\times\mathcal{Q}$. The fixed map $h$ then uniquely determines
$\mbeta_{ij}=h(\mbz_{ij})$. Therefore, any two admissible specifications
that induce the same observed-data distribution agree on
$\mbeta_{\mathcal{T},\mathcal{Q}}$, which is identification according to
Definition~\ref{def:identification}.
\end{proof}

\begin{corollary}[Identification of counterfactual cell distributions]
\label{cor:distribution}
Under the assumptions of Theorem~\ref{theo:identification_two}, the untreated
counterfactual distribution $p(\,\cdot\s\mbeta_{ij})$ is identified for every
$(i,j)\in\mathcal{T}\times\mathcal{Q}$.
\end{corollary}
\begin{proof}
Theorem~\ref{theo:identification_two} identifies $\mbeta_{ij}$. By
Assumption~\ref{ass:ef}, the counterfactual distribution is the known function
$p(\,\cdot\s\mbeta_{ij})$ of this identified natural parameter.
\end{proof}

\begin{corollary}[Identification of the population target underlying ECE]
\label{cor:ece}
Under the assumptions of Theorem~\ref{theo:identification_two}, the
vector-valued population natural-parameter contrast
$\mbtau_{ij}=\tilde{\mbeta}_{ij}-\mbeta_{ij}$ is identified for every
$(i,j)\in\mathcal{T}\times\mathcal{Q}$.
\end{corollary}
\begin{proof}
By the definition of $\tilde Y_{ijk}$, consistency, and
Assumptions~\ref{ass:ef}--\ref{ass:ef-identifiable}, the observed
treated--post-treatment cell distribution identifies
$\tilde{\mbeta}_{ij}$. Theorem~\ref{theo:identification_two} identifies the
untreated counterfactual parameter $\mbeta_{ij}$. Their difference is
therefore identified.
\end{proof}

\begin{corollary}[Identification of the population target underlying ECD]
\label{cor:ecd}
Suppose that the relevant Kullback--Leibler divergence is finite.
Then $\delta_{ij}=\text{KL}(
 p(\,\cdot\s\tilde{\mbeta}_{ij})\,\|\,
 p(\,\cdot\s\mbeta_{ij}))$ is identified for every
$(i,j)\in\mathcal{T}\times\mathcal{Q}$. Any fixed deterministic aggregate of
these cell-level divergences is also identified.
\end{corollary}
\begin{proof}
By the definition of $\tilde Y_{ijk}$, consistency, and
Assumptions~\ref{ass:ef}--\ref{ass:ef-identifiable}, the observed
treated--post-treatment cell distribution identifies
$\tilde{\mbeta}_{ij}$. Theorem~\ref{theo:identification_two} identifies the
untreated counterfactual parameter $\mbeta_{ij}$. Hence any two admissible
specifications inducing the same observed-data distribution agree on both
distributions entering the divergence. Since the KL divergence is a fixed
functional of this ordered pair of distributions, they also agree on
$\delta_{ij}$.
\end{proof}

\begin{corollary}[Identification under staggered adoption]
\label{cor:staggered-identification}
Let $\mathcal G$ denote the set of finite adoption times and, for each
$g\in\mathcal G$, define
\begin{align}
\mathcal T_g&=\{i:A_i=g\},&
\mathcal P_g&=\{j:j<g\},&
\mathcal Q_g&=\{j:j\ge g\}.
\end{align}
Let $\mathcal C=\{i:A_i=\infty\}$ denote the never-treated units and define
\begin{align}
Z_{p,g}
=
Z_p[\mathcal C\cup\mathcal T_g,\{1,\ldots,T\}].
\end{align}
Suppose that, for every $g\in\mathcal G$, Assumptions~1--7 hold on the
subpanel containing $\mathcal C\cup\mathcal T_g$, with
$(\mathcal T,\mathcal P,\mathcal Q)$ replaced by
$(\mathcal T_g,\mathcal P_g,\mathcal Q_g)$, and that, for every component $p$,
\begin{align}
\rank\!\left(Z_p[\mathcal C,\mathcal P_g]\right)
=
\rank(Z_{p,g}).
\end{align}
Then, for every $g\in\mathcal G$,
\begin{align}
Z_p[\mathcal T_g,\mathcal Q_g]
=
Z_p[\mathcal T_g,\mathcal P_g]
\left(
Z_p[\mathcal C,\mathcal P_g]
\right)^\dagger
Z_p[\mathcal C,\mathcal Q_g],
\end{align}
and hence the untreated counterfactual natural parameters are identified on
the staggered region
\begin{align}
\Omega_{\mathrm{tgt}}^{\mathrm{stag}}
=
\bigcup_{g\in\mathcal G}
\left(\mathcal T_g\times\mathcal Q_g\right).
\end{align}
Consequently, the corresponding counterfactual cell distributions and the
population targets underlying the ECE and ECD are identified whenever the
relevant KL divergences are finite.
\end{corollary}

\begin{proof}
Suppose two admissible full-data specifications induce the same observed-data
distribution for the complete staggered panel. Their restrictions to
$\mathcal C\cup\mathcal T_g$ therefore induce the same observed-data
distribution for every $g\in\mathcal G$. For each $g$, the restricted
subpanel has common adoption time $g$ and rectangular target block
$\mathcal T_g\times\mathcal Q_g$. Theorem~\ref{theo:identification_two}
therefore identifies $Z_p[\mathcal T_g,\mathcal Q_g]$ through the displayed
completion formula. Taking the union over adoption cohorts yields
$\Omega_{\mathrm{tgt}}^{\mathrm{stag}}$.
\end{proof}
The cohort decomposition is only an identification device and does not
require the model to be estimated separately for each cohort.

\begin{remark}[Latent-factor non-identifiability]
The latent factors themselves need not be uniquely identified. For example,
rotations or other invertible changes of coordinates may leave
$\mbtheta[p]^\top\mbbeta[p]$ unchanged.
Theorem~\ref{theo:identification_two} requires uniqueness of the completed
predictor matrices $Z_p$, and hence of the induced natural-parameter surface,
not uniqueness of a particular factor parameterization.
\end{remark}

\begin{remark}[Number of controls]
The rank condition requires $|\mathcal{C}|\ge k_p$, not
$|\mathcal{C}|\ge|\mathcal{T}|$. Thus, a small control group can identify the
counterfactuals of many treated units if the control-unit factor vectors span
all effective unit-factor directions of the complete untreated predictor
matrix. Under the full additive predictor factorization, $k_p\le r+2$.
\end{remark}

\section{Experimental Details}
\subsection{BBVI Implementation}\label{app:bbvi_implementation}
We implement BBVI using a mean-field Gaussian variational family. The variational family is parameterized by the means and variances of the latent parameters,
\begin{align}
\mbnu=\{\mbmu_{\nu},\mbsigma^2_{\nu}\}.
\end{align}
There are $D_{\nu}=2D$ variational parameters in total, where
$D=P(N+T)(r+1)$ is the number of scalar model parameters in $\Theta$ and
$P$ is the dimension of $\mbeta_{ij}$. When calculating the ELBO, we use the
reparameterization trick:
\begin{align}
\Theta^{(s)}
=
\mbmu_{\nu}+L\mbz^{(s)},
\quad
\mbz^{(s)}\sim\mathcal N(\mathbf 0,I_D),
\quad
s=1,\ldots,S,
\end{align}
where $L=\operatorname{diag}(\mbsigma_{\nu})$ under the mean-field assumption.
The ELBO is estimated via Monte Carlo as
\begin{align}
\mathcal L(\mbnu)
\approx
\frac{1}{S}
\sum_{s=1}^{S}
\Big[
\log p(\mby,\Theta^{(s)})
-
\log q(\Theta^{(s)};\mbnu)
\Big].
\end{align}
In the implementation, the variance parameters are stored and optimized on the
log-variance scale. Also, whenever natural parameters are constrained, we
reparameterize the model in terms of unconstrained latent variables and apply
appropriate transformations prior to evaluating the log-joint. For example,
in the Gaussian exponential family, the second natural parameter must satisfy
\begin{align}
\eta_{ij}[2]<0.
\end{align}
We therefore construct an unconstrained predictor $z_{ij}[2]\in\mathbb R$
from the PMF and define
\begin{align}
\eta_{ij}[2]
=
-\exp\left(z_{ij}[2]\right),
\end{align}
which guarantees that the natural-parameter space is respected throughout
optimization.

The negative ELBO is minimized using the Adam optimizer with a full-batch
objective. Unless otherwise stated, we use $S=2$ Monte Carlo samples per ELBO
evaluation. The learning rate and number of optimization epochs are specified
separately for each experiment. All gradients are computed using automatic
differentiation in PyTorch. All experiments were executed on a laptop equipped
with an Intel Core i7-13620H CPU, 16\,GB RAM, and an NVIDIA RTX 3050 GPU
(6\,GB VRAM).

\subsection{Empirical Bayes Hyperparameter Learning}\label{app:eb}
The default implementation of EFSC assumes fixed Gaussian prior
hyperparameters for the unit effects, time effects, and latent factors. As an
alternative, we consider a simple empirical Bayes (EB) extension that estimates
the prior hyperparameters directly from the observed panel.

Specifically, using hyperparameters shared across parameter dimensions, for
each parameter dimension $p=1,\ldots,P$, we assume
\begin{align}
\alpha_i[p]
&\sim
\mathcal N(\mu_\alpha,\sigma_\alpha^2),\\
\gamma_j[p]
&\sim
\mathcal N(0,\sigma_\gamma^2),\\
\mbtheta_i[p]
&\sim
\mathcal N(\mathbf 0,\sigma_\theta^2 I_r),\\
\mbbeta_j[p]
&\sim
\mathcal N(\mathbf 0,\sigma_\beta^2 I_r).
\end{align}

The hyperparameters are collected in
\begin{align}
\mbphi=
(\mu_\alpha,\sigma_\alpha,\sigma_\gamma,\sigma_\theta,\sigma_\beta).
\end{align}

Rather than fixing $\mbphi$ a priori, we estimate it jointly with the
variational parameters by maximizing the ELBO. The Monte Carlo approximation
becomes
\begin{align}
\mathcal L(\mbnu,\mbphi)
\approx
\frac{1}{S}
\sum_{s=1}^{S}
\Big[
\log p(\mby,\Theta^{(s)};\mbphi)
-
\log q(\Theta^{(s)};\mbnu)
\Big].
\end{align}

The resulting optimization problem is
\begin{align}
(\hat{\mbnu},\hat{\mbphi})
=
\arg\max_{\mbnu,\mbphi}
\mathcal L(\mbnu,\mbphi).
\end{align}

The gradients of the ELBO with respect to both $\mbnu$ and $\mbphi$ are computed using automatic differentiation, and optimization is performed jointly using Adam. The EB
extension adds only five scalar hyperparameters to the optimization problem.

\subsection{Exponential tilts across EFs}\label{app:exp_tilt_experiments}
The experiments in Section~\ref{subsec:tilting_across_efs} evaluate the ability
of EFSC to recover the effects of exponential-tilt interventions across
multiple one-parameter exponential-family distributions. The data-generation, model-fitting and evaluation procedure is as follows:
\begin{enumerate}
\item Select an exponential-family distribution from: Bernoulli,
Poisson, Exponential, Laplace (known mean), Chi-squared, and Gaussian (known
variance). The known Laplace mean and Gaussian variance are set to $0$ and $1$,
respectively.

\item Generate a panel of natural parameters using the generative model in
Appendix~\ref{app:generative_model}, with $N=32$ units, $T=64$ time steps, and
latent-factor dimension $r=2$. The unit-intercept mean on the unconstrained PMF
predictor scale is set to $-1$ for Bernoulli, $1$ for Poisson, $\log 3$ for
Exponential, Laplace, and Chi-squared, and $0$ for Gaussian. All unit effects,
time effects, and latent-factor entries are generated with standard deviation
$0.05$, with the remaining means set to zero.

\item For each cell $(i,j)$, draw the number of observations as
$m_{ij}\sim 1+\operatorname{Poisson}(\lambda_m)$, where
$\lambda_m\in\{5,55,\ldots,505\}$ and
$\mathbb E[m_{ij}]=1+\lambda_m$.

\item Define the intervention time at $t_0=52$ and select the final six units
for treatment, $\mathcal T=\{N-5,\ldots,N\}$.

\item For every treated unit and post-treatment period $j\geq t_0$, apply an
exponential tilt
\begin{align}
\tilde{\eta}_{ij}
=
\eta_{ij}+\tau,
\quad
\tau\in\{0.1,0.5,2.0\}.
\end{align}
See Figure~\ref{fig:poisson_tilt_example} for an illustration on a panel of
Poisson log-intensities.

\item Draw observations from the original distributions outside the
treated post-treatment block and from the intervened distributions inside it:
\begin{align}
y_{ijk}\mid\eta_{ij}
&\stackrel{\mathrm{iid}}{\sim}
\mathrm{EF}(\eta_{ij}),
&& (i,j)\notin\Omega_{\mathrm{tgt}},\\
\tilde{y}_{ijk}\mid\tilde{\eta}_{ij}
&\stackrel{\mathrm{iid}}{\sim}
\mathrm{EF}(\tilde{\eta}_{ij}),
&& (i,j)\in\Omega_{\mathrm{tgt}},
\end{align}
for $k=1,\ldots,m_{ij}$.

\item Fit EFSC twice using the BBVI procedure described in
Appendix~\ref{app:bbvi_implementation}: first on
$\mby_{\Omega_{\mathrm{obs}}}$, with the treated post-treatment block masked
out, and second on $\mby_{\Omega_{\mathrm{tgt}}}$, using only the treated
post-treatment block. Using the natural parameters reconstructed at the
respective variational posterior means, compute
\begin{align}
\widehat{\mathrm{ECE}}_{ij}
=
\hat{\eta}_{ij}^{\mathrm{treat}}
-
\hat{\eta}_{ij}^{\mathrm{ctrl}},
\quad
(i,j)\in\Omega_{\mathrm{tgt}}.
\end{align}

\item Evaluate the recovery accuracy using the mean absolute error
\begin{align}
\mathrm{MAE}
=
\frac{1}{|\Omega_{\mathrm{tgt}}|}
\sum_{(i,j)\in\Omega_{\mathrm{tgt}}}
\left|
\widehat{\mathrm{ECE}}_{ij}-\tau
\right|.
\end{align}
\end{enumerate}

The PMF parameters are assigned independent Gaussian priors with standard
deviation $3$. The unit-effect prior means are set to the corresponding
family-specific baseline values above, and all remaining prior means are zero.
Each BBVI fit uses $S=2$ Monte Carlo samples per ELBO evaluation and $2{,}000$
optimization epochs. The Adam learning rate is $0.05$ for Bernoulli, Poisson,
and Gaussian, $0.005$ for Exponential and Laplace, and $0.002$ for
Chi-squared. Unless otherwise stated, all reported results are averages over
$20$ independently generated panels.

\begin{figure}
    \centering
    \includegraphics[width=0.9\linewidth]{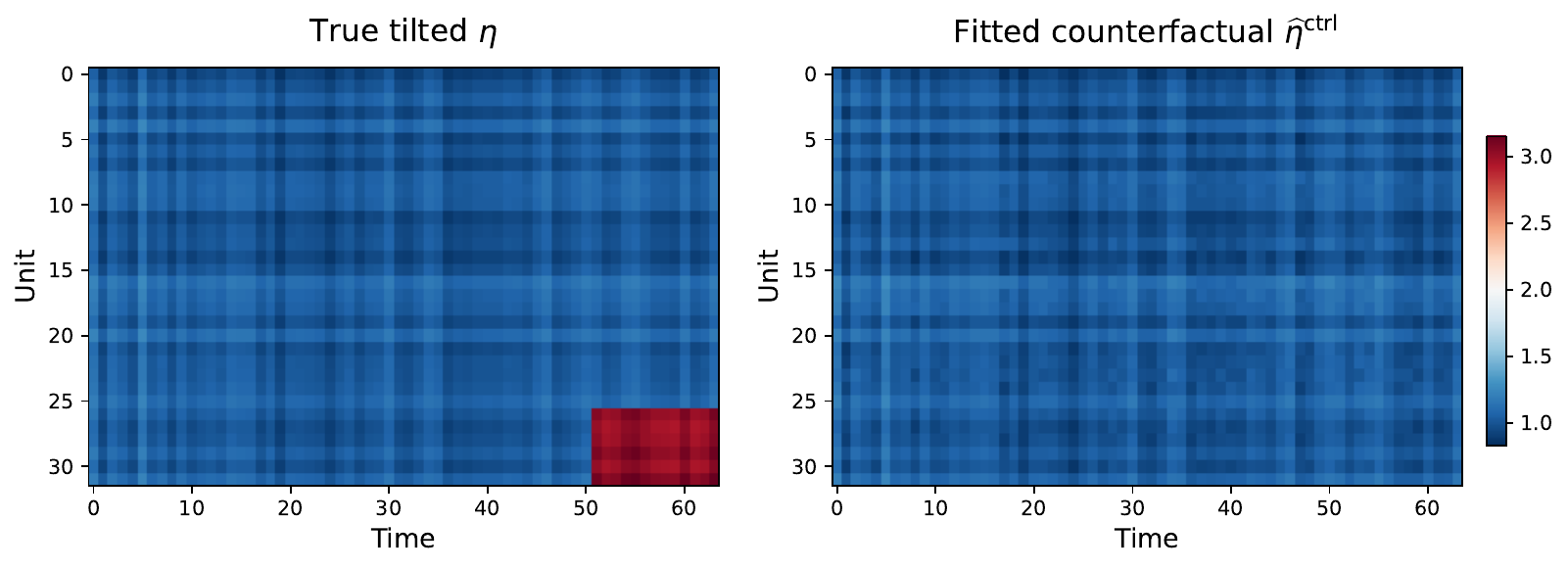}
    \caption{Illustration of the exponential-tilt experiment for a panel of
    Poisson datasets with expected cell size $\E{m_{ij}}=106$. Left: true natural-parameter matrix
    after applying an additive tilt $\tau=2$ to the treated post-intervention
    block, corresponding to the last six units and final 13 time periods.
    Right: EFSC reconstruction of the counterfactual natural-parameter matrix,
    obtained by fitting the model with the treated post-intervention block
    masked out.}
    \label{fig:poisson_tilt_example}
\end{figure}

\subsection{Expected causal divergence in the exponential-tilt experiment}
\label{app:exp_tilt_ecd}
Using the same simulated panels and EFSC fits described in
Appendix~\ref{app:exp_tilt_experiments}, we also evaluate recovery of the
expected causal divergence (ECD). This provides a complementary measure of the
effect of the intervention by quantifying the divergence between the tilted
and counterfactual distributions rather than the additive shift in the natural
parameter.

For each treated post-treatment cell
$(i,j)\in\Omega_{\mathrm{tgt}}$, the ground-truth cellwise divergence
corresponding to the ECD is computed from the known data-generating natural
parameters as
\begin{align}
\mathrm{ECD}_{ij}^{\mathrm{true}}
=
\mathrm{KL}\left(
p(y;\eta_{ij}+\tau)
\|p(y;\eta_{ij})
\right).
\end{align}
Following the identity in
Equation~\eqref{eq:causal_effect_kl_bregman}, for a one-parameter exponential
family with log-partition function $a(\cdot)$ and sufficient statistic $t(y)$,
this quantity can be written as
\begin{align}
\mathrm{ECD}_{ij}^{\mathrm{true}}
=
\tau\,\EE{\eta_{ij}+\tau}{t(Y)}
-
a(\eta_{ij}+\tau)
+
a(\eta_{ij}).
\end{align}

In the fitted model, we estimate this quantity using a plug-in approximation
based on natural-parameter estimates reconstructed at the variational
posterior means of the factorization parameters from the two EFSC fits. Let
$\hat{\eta}_{ij}^{\mathrm{ctrl}}$ denote the counterfactual natural parameter
reconstructed at the variational posterior mean of the fit with the treated
post-treatment block masked out, and let
$\hat{\eta}_{ij}^{\mathrm{treat}}$ denote the corresponding reconstruction
from the fit to the treated post-treatment block. We compute
\begin{align}
\widehat{\mathrm{ECD}}_{ij}
=
\mathrm{KL}\left(
p(y;\hat{\eta}_{ij}^{\mathrm{treat}})
\|p(y;\hat{\eta}_{ij}^{\mathrm{ctrl}})
\right).
\end{align}

Although ECD is defined at the cell level, we evaluate estimation accuracy
through the average ECD over the treated block rather than through a cellwise
mean absolute error. This choice avoids comparing raw cellwise KL errors
across families whose divergences have substantially different scales and
curvature. In particular, the KL divergence is nonlinear in the natural
parameter, so small errors in the baseline natural-parameter estimate can
induce large local KL errors in some families even when the average divergence
is recovered accurately.

For each family, tilt magnitude $\tau$, and dataset-size parameter $\lambda_m$,
we therefore compare
\begin{align}
\overline{\mathrm{ECD}}^{\mathrm{true}}
&=
\frac{1}{|\Omega_{\mathrm{tgt}}|}
\sum_{(i,j)\in\Omega_{\mathrm{tgt}}}
\mathrm{ECD}_{ij}^{\mathrm{true}},
\quad
\overline{\widehat{\mathrm{ECD}}}
=
\frac{1}{|\Omega_{\mathrm{tgt}}|}
\sum_{(i,j)\in\Omega_{\mathrm{tgt}}}
\widehat{\mathrm{ECD}}_{ij}.
\end{align}
Results for all families and $\tau=0.1$ and $\tau=2$ are presented in
Figure~\ref{fig:ecd_tilt_recovery}; the curves average these quantities over
the 20 independently generated panels. In general, we see that EFSC is able to
recover the family-specific scale of the divergence induced by the same
exponential-tilt intervention used in the ECE experiment.

\begin{figure}[!htbp]
\centering
\begin{subfigure}{0.9\textwidth}
\centering
\includegraphics[width=\textwidth]{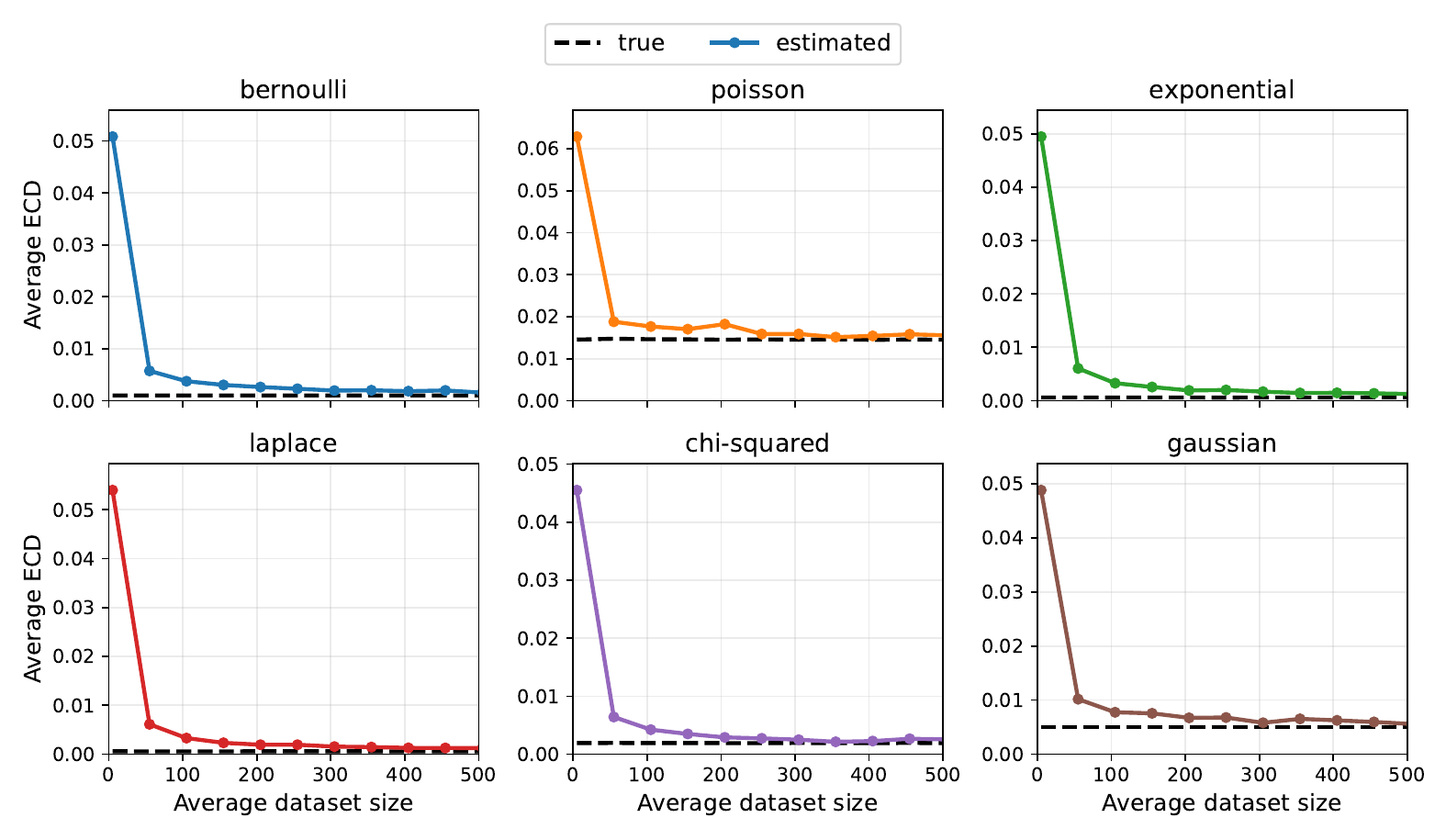}
\caption{$\tau=0.1$}
\end{subfigure}

\vspace{0.75em}

\begin{subfigure}{0.9\textwidth}
\centering
\includegraphics[width=\textwidth]{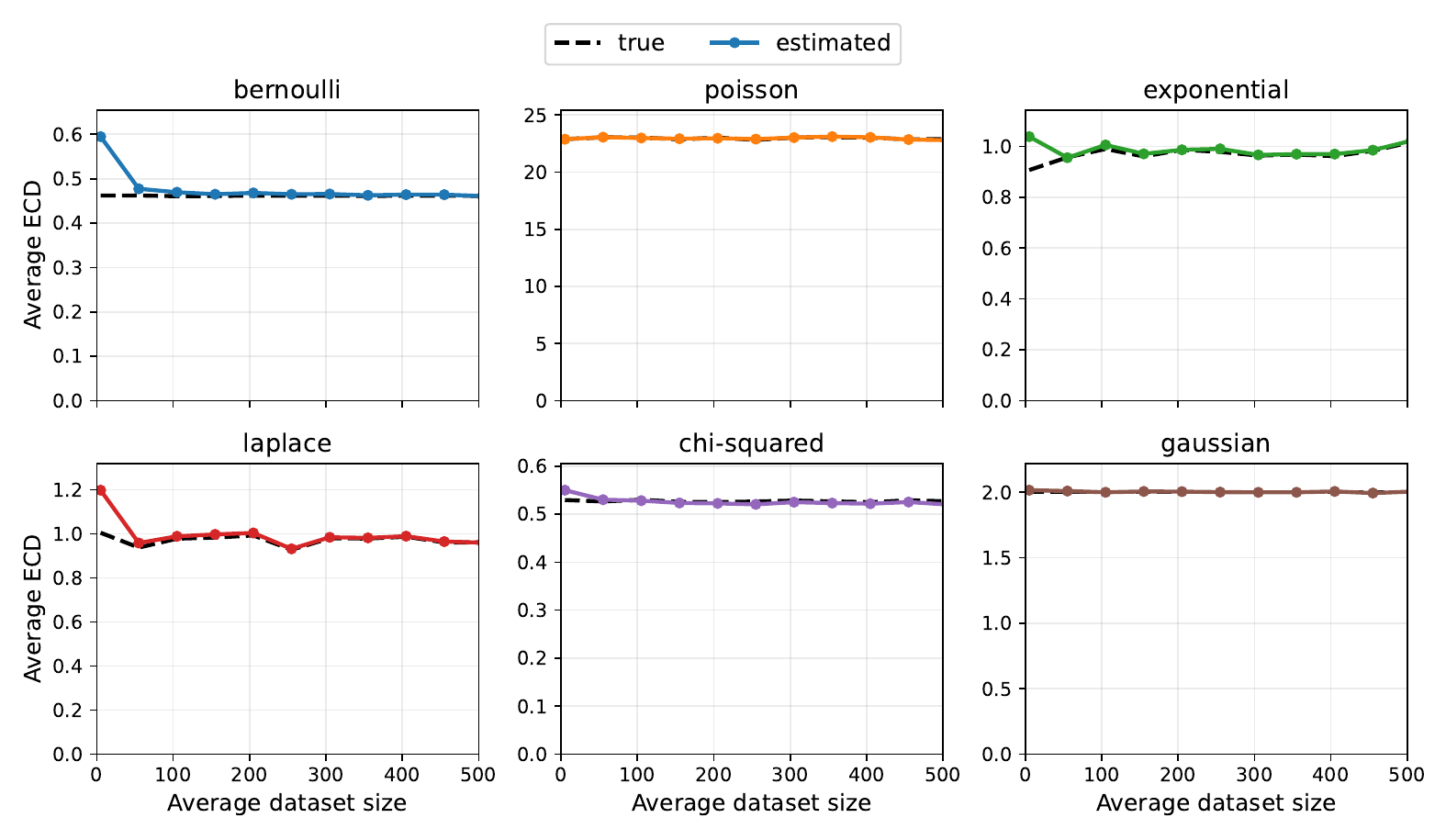}
\caption{$\tau=2.0$}
\end{subfigure}

\caption{
Estimated versus true expected causal divergence (ECD) in the
exponential-tilt experiment. Each panel compares the true treated-block
average ECD (dashed black line) with the plug-in EFSC estimate (solid colored
line) as the average dataset size increases. Results are shown separately for
each one-parameter exponential family and are averaged over 20 independently
generated panels. Top figure: small tilt, $\tau=0.1$; bottom figure: larger
tilt, $\tau=2.0$.
}
\label{fig:ecd_tilt_recovery}
\end{figure}

\subsection{Intervention Benchmarks on Gaussian Panels}
\label{app:benchmark_gaussian_panels}
The experiments in Section~\ref{subsec:gaussian_intervention_benchmarks} evaluate EFSC on panels of univariate Gaussian datasets. For each cell $(i,j)$, the observation model is
\begin{align}
y_{ijk}\mid \mbeta_{ij}
\stackrel{\mathrm{iid}}{\sim}
\mathcal N(\mu_{ij},\sigma_{ij}^2),
\quad\mbeta_{ij}=
\left(\eta_{ij}[1],\eta_{ij}[2]\right)^T,
\end{align}
with natural parameters
\begin{align}
\eta_{ij}[1]=\frac{\mu_{ij}}{\sigma_{ij}^2},
\quad\eta_{ij}[2]=
-\frac{1}{2\sigma_{ij}^2}.
\end{align}
For estimation, following the multi-parameter PMF representation in Appendix~\ref{app:multi-pmf}, we factorize the two corresponding unconstrained predictor components separately:
\begin{align}
\eta_{ij}[1]
&=\alpha_i[1]+\gamma_j[1]
+\mbtheta_i[1]^T\mbbeta_j[1],
\\
\log\left(-\eta_{ij}[2]\right)
&=\alpha_i[2]+\gamma_j[2]
+\mbtheta_i[2]^T\mbbeta_j[2].
\end{align}
Equivalently,
\begin{align}
\eta_{ij}[2]
=-\exp\left\{\alpha_i[2]
+\gamma_j[2]+\mbtheta_i[2]^T\mbbeta_j[2]
\right\}<0,
\end{align}
which preserves the natural-parameter space of the Gaussian exponential family.

For data generation, however, the two unconstrained predictor components are generated from a shared latent factorization, as described in the main text, using latent dimension $r=2$. In all experiments, the intervention is applied to approximately the last $20\%$ of units and time periods.

We compare three estimators. First, the standard synthetic-control estimator is
applied to the matrix of cell averages $\bar Y_{ij}$. Second, EFSC-MLE applies
synthetic control separately to the cell-wise maximum likelihood estimates
$\hat{\mbeta}_{ij}$. Third, EFSC-PMF fits the factorized model by BBVI, using
the observed cells $\Omega_{\mathrm{obs}}$ to reconstruct the untreated
counterfactual natural parameters in $\Omega_{\mathrm{tgt}}$. For EFSC-PMF,
treated-cell parameters are estimated from a second fit using only the treated
post-treatment block. Unless otherwise stated, all reported errors are averaged
over $20$ independently generated panels.

For the exponential-tilt experiment in Table~\ref{tab:exponential_tilt_results},
we use panels with $N=32$ units and $T=128$ time periods. Cell sizes are ragged,
with
\begin{align}
m_{ij}\sim 1+\mathrm{Poisson}(\lambda_m),\quad\lambda_m=55.
\end{align}
For treated post-treatment cells, we apply the natural-parameter intervention
\begin{align}
\tilde{\mbeta}_{ij}
=\mbeta_{ij}+(\tau,0)^T,\quad
\tau\in\{0.1,0.25,0.5,1,2\}.
\end{align}
For the response-level estimators, the estimated mean effects are divided
cellwise by the true simulated counterfactual variance $\sigma_{ij}^2$ so that
they are expressed on the natural-parameter tilt scale. The natural-parameter
estimators instead compare the estimated change in the first component,
$\eta_{ij}[1]$, directly with $\tau$. The recovery error is then the MAE between
the resulting estimate and the true tilt. The EFSC-PMF quantities reported in
Table~\ref{tab:exponential_tilt_results} are approximated using $2{,}000$
variational-posterior draws.

Figure~\ref{fig:exp_tilt_panel} uses the same construction with vector tilt
$\mbtau=(0.4,-0.6)^T$ and compares recovery for
$(N,T)\in\{(32,128),(64,256)\}$ and
$\lambda_m\in\{25,55\}\cup\{105,155,\ldots,505\}$. The MAE curves use plug-in
reconstructions evaluated at the variational posterior means, whereas the center
and right panels display pooled variational-posterior draws from one
representative panel. An illustration of this two-parameter tilt and the
corresponding EFSC counterfactual reconstruction is shown in
Figure~\ref{fig:gaussian_tilt_eta_example}.

\begin{figure}[t]
    \centering
    \includegraphics[width=0.95\linewidth]{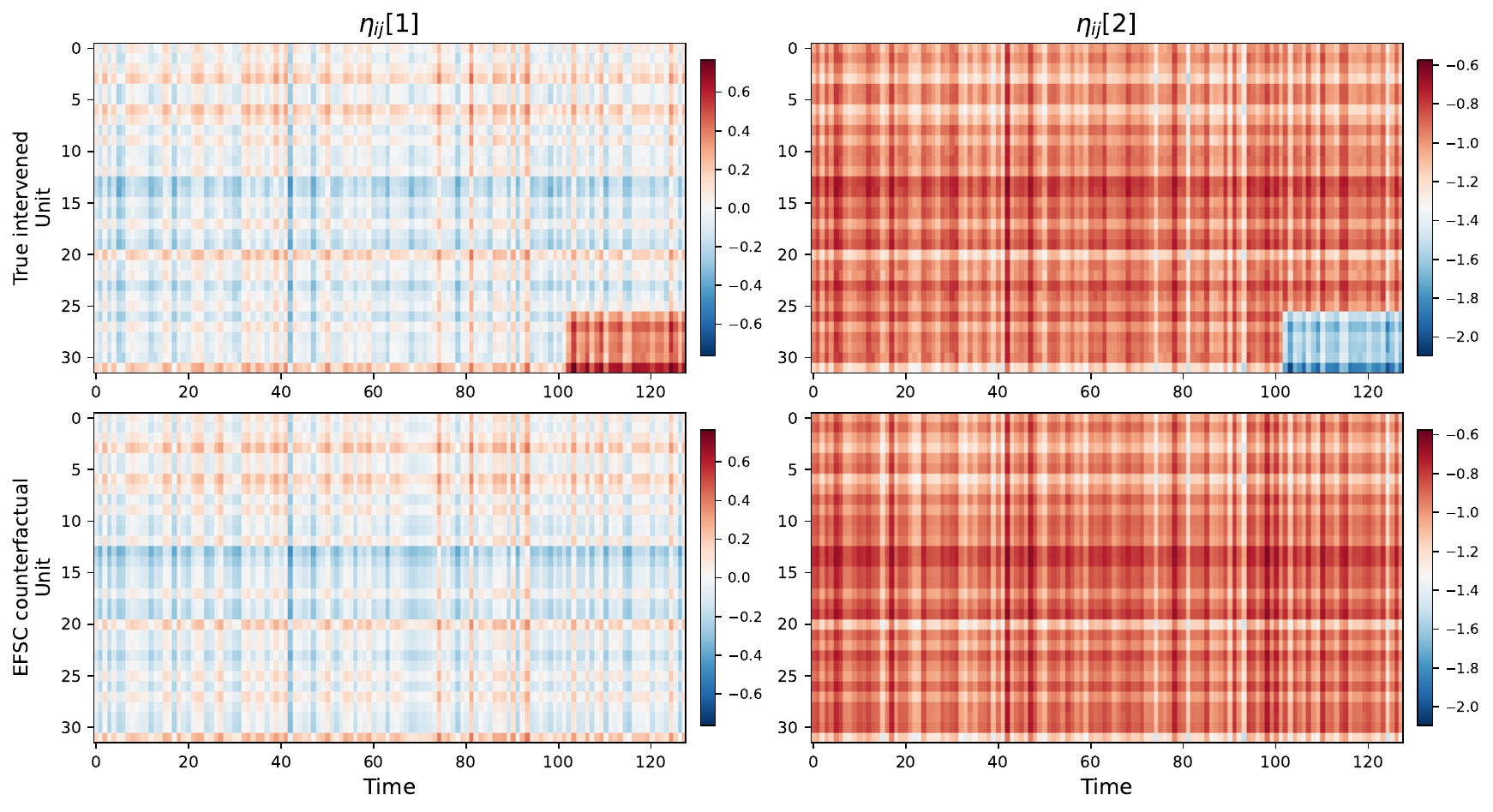}
    \caption{Illustration of the Gaussian exponential-tilt experiment used in Figure~\ref{fig:exp_tilt_panel}. Top row: simulated natural-parameter matrices after applying the two-parameter tilt $\mbtau=(0.4,-0.6)^T$ to the treated post-treatment block. Bottom row: EFSC-PMF reconstruction of the untreated counterfactual natural-parameter matrices, obtained by fitting the model with the treated post-treatment block masked out.}\label{fig:gaussian_tilt_eta_example}
\end{figure}

For the structured latent-intervention experiment in Table~\ref{tab:kl_mae}, we
use panels with $N=32$ units and $T=128$ time periods, and fixed
$m\in\{5,25,50,100,200\}$ in order to remove the variability due to ragged cell sizes. The
treated post-treatment natural parameters are perturbed by
\begin{align}
\tilde{\eta}_{ij}[1]&=\eta_{ij}[1]+\kappa_1
\sigma(\mbtheta_i^\top\mbw_1)
\sigma(\mbbeta_j^\top\mbw_2),
\\
\tilde{\eta}_{ij}[2]&=\eta_{ij}[2]-\kappa_2
\sigma(\mbtheta_i^\top\mbw_1)
\sigma(\mbbeta_j^\top\mbw_2),
\quad\tilde{\eta}_{ij}[2]<0,
\end{align}
where $\sigma(\cdot)$ denotes the sigmoid function,
$\mbw_1,\mbw_2\in\mathbb R^r$ are unit-norm vectors, and
$(\kappa_1,\kappa_2)\in\{(0.2,0.1),(0.6,0.5),(1.6,1.5)\}$. For each treated
post-treatment cell, the true KL divergence is computed from the known
untreated and intervened Gaussian parameters,
\begin{align}
\mathrm{KL}_{ij}^{\mathrm{true}}
=\mathrm{KL}\left(
\mathcal N(\tilde{\mu}_{ij},\tilde{\sigma}_{ij}^2)
\|\mathcal N(\mu_{ij},\sigma_{ij}^2)\right).
\end{align}
The EFSC-MLE and EFSC-PMF estimates are evaluated by the plug-in KL divergence
between the estimated treated and counterfactual Gaussian distributions. For
EFSC-PMF, these distributions are reconstructed at the variational posterior
means. The reported metric is the cell-wise MAE over
$\Omega_{\mathrm{tgt}}$.

Finally, for the heavy-tailed misspecification experiment in
Table~\ref{tab:student_t_corruption_kl}, we use smaller panels with $N=32$
units and $T=64$ time periods, fixed cell sizes
$m\in\{5,25,50,100,200\}$, and degrees of freedom
$\nu_{\mathrm{df}}\in\{80,40,20,10,5,3\}$. This experiment deliberately falls outside the treated-outcome condition in Assumption~\ref{ass:ef} and is included solely as a robustness exercise under model misspecification. The treated post-treatment Gaussian observations are replaced by variance-matched Student-$t$ draws,
\begin{align}
\tilde y_{ijk}=\mu_{ij}+s_{ij}t_{ijk},
\quad t_{ijk}\sim\mathrm{Student}\text{-}t(\nu_{\text{df}}),
\end{align}
where $s_{ij}$ is chosen so that
\begin{align}
s_{ij}^2\frac{\nu_{\mathrm{df}}}{\nu_{\mathrm{df}}-2}
=\sigma_{ij}^2.
\end{align}
Thus the treated distribution has the same mean and variance as the original
Gaussian cell distribution, but heavier tails when $\nu_{\mathrm{df}}$ is small.
Because the treated distribution is no longer Gaussian, the true cell-wise KL
divergence
\begin{align}
\mathrm{KL}_{ij}^{\mathrm{true}}
=
\mathrm{KL}\left(
t_{\nu_{\mathrm{df}}}(\mu_{ij},s_{ij})
\|
\mathcal N(\mu_{ij},\sigma_{ij}^2)
\right)
\end{align}
is approximated using $20{,}000$ Monte Carlo draws per treated post-treatment
cell. We then compare this baseline with the plug-in Gaussian KL divergences
induced by EFSC-MLE and EFSC-PMF, and report the cell-wise MAE over the treated
post-treatment block. For EFSC-PMF, the Gaussian parameters are reconstructed
at the variational posterior means.

All EFSC-PMF fits use independent standard-normal priors and $S=2$ Monte Carlo
samples per ELBO evaluation. The scalar-tilt and structured-intervention
experiments use $3{,}000$ optimization epochs and learning rate $0.1$, while
the vector-tilt experiment uses $2{,}000$ epochs and learning rate $0.1$. The
Student-$t$ experiment uses $8{,}000$ epochs and learning rate $0.01$.

\subsection{Distributional Placebo Tests}\label{app:placebo_tests}
Here we provide additional details for the placebo experiments in
Section~\ref{subsec:placebo_tests}. The experiments use the same generative
construction and EFSC fitting procedure described in the previous appendix
sections. In each case, we generate a synthetic panel of datasets, designate a
subset of treated units and post-treatment periods, and compute the placebo
statistic
\begin{align*}
\Delta_{\mathrm{KL}}=\mathrm{ECD}^{\mathrm{post}}
-\mathrm{ECD}^{\mathrm{pre}},
\end{align*}
as defined in Equation~\eqref{eq:placebo_statistic}. Each statistic uses four
fitted natural-parameter surfaces: a post-period counterfactual fit excluding
the target post-treatment block, a post-period fit using only that block, a
pre-period reference fit using all pretreatment cells, and a pre-period fit
using only the target units' pretreatment cells. The pretreatment term
therefore serves as a baseline discrepancy for the target units before the
intervention, while the post-treatment term measures the corresponding
discrepancy when the intervention may be present. Throughout this subsection,
the fitted natural parameters are reconstructed at the variational posterior
means.

Placebo assignments are formed by selecting subsets of control units with the
same cardinality as the treated set. Following
Algorithm~\ref{alg:efsc-placebo}, actual treated post-treatment cells are
excluded from all conditioning sets used to construct the placebo reference
distribution.

For the one-parameter exponential-family experiments, we use panels with
$N=16$ units, $T=32$ time periods, latent rank $r=2$, and fixed cell size
$m=100$. The first $12$ units are used as controls and the last $4$ units are
assigned to treatment. All $\binom{12}{4}=495$ admissible placebo assignments
are enumerated. The intervention time is set to $t_0=27$, so that the first
$26$ periods form the pretreatment block and the remaining periods form the
post-treatment block. For each family, treated post-treatment observations are
generated after applying exponential tilts of increasing magnitude to the
natural parameter,
\begin{align}
\tilde{\eta}_{ij}=\eta_{ij}+\tau,
\quad i\in\mathcal{T},\; j\in\mathcal{Q}.
\end{align}
The placebo distribution is computed once for each family under the untreated
panel, and the observed statistics corresponding to the different values of
$\tau$ are overlaid on the same empirical null distribution.

Figure~\ref{fig:placebo_univariate_efs} shows the resulting placebo distributions
for the one-parameter exponential families considered in the paper: Bernoulli,
Poisson, exponential, Laplace with known mean, chi-squared, and Gaussian with known variance. Across families, the placebo distributions are concentrated near zero, although
small finite-sample offsets can appear because both pre- and post-treatment ECDs
are estimated using fitted EFSC models. Inference is therefore calibrated
against the empirical placebo distribution rather than against a zero-centered
reference.

\begin{figure}[!htbp]
    \centering
    \includegraphics[width=0.95\linewidth]{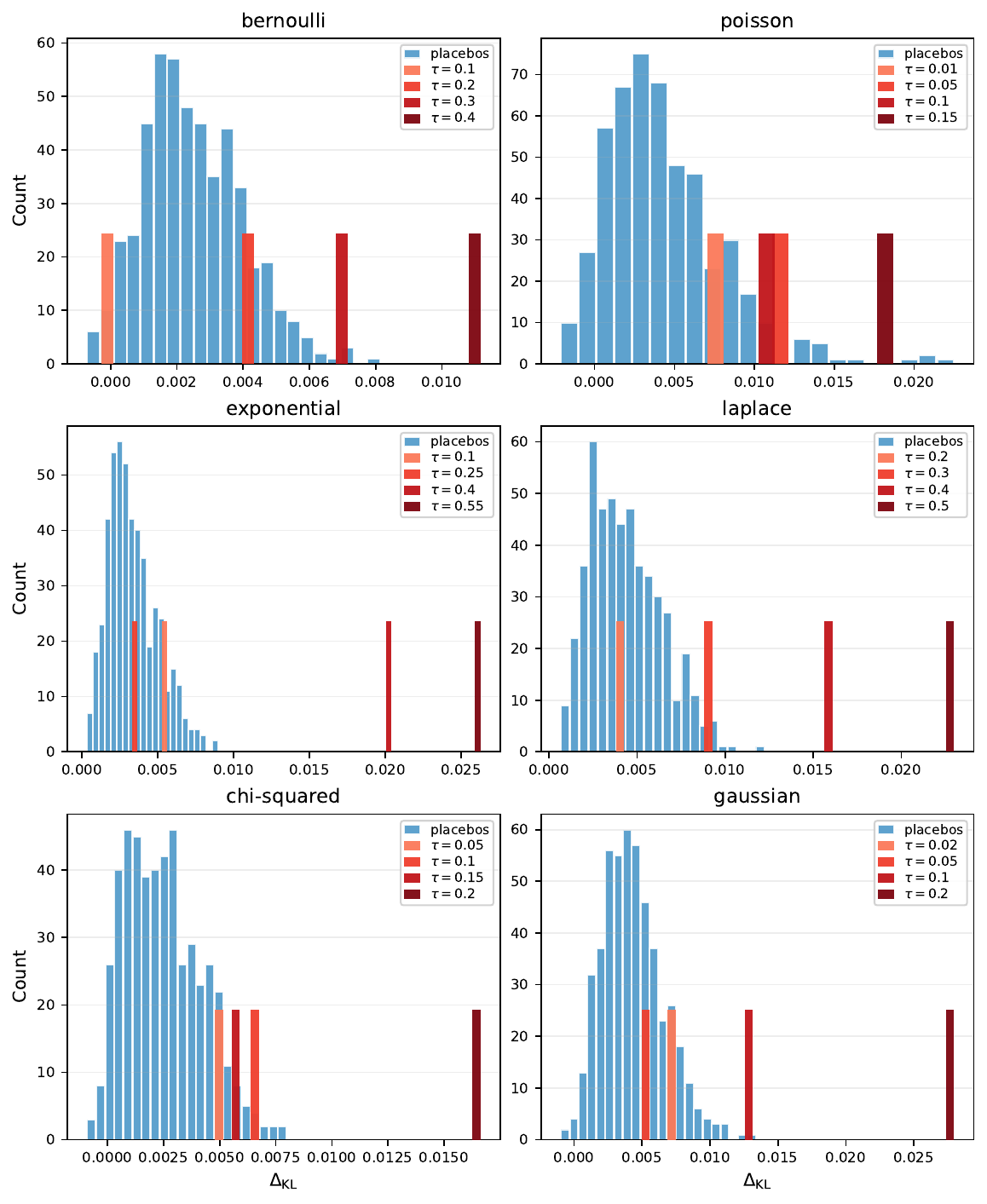}
    \caption{Distributional placebo tests across one-parameter exponential-family panels. Each panel shows the empirical placebo distribution of $\Delta_{\mathrm{KL}}$ under the no-leakage procedure in Algorithm~\ref{alg:efsc-placebo}, with observed statistics overlaid for exponential tilts of increasing magnitude. The examples illustrate how the same ECD-based placebo diagnostic can be applied across different observation models.}
    \label{fig:placebo_univariate_efs}
\end{figure}

For the Gaussian intervention experiment in
Section~\ref{subsec:placebo_tests}, we use a two-parameter Gaussian EF panel
with the same no-leakage placebo construction and settings
$N=16$, $T=32$, $r=2$, and $m=100$. The first $12$ units are controls, the
last $4$ units are treated, and the first $26$ periods form the pretreatment
block. We compare three interventions on the treated post-treatment block:
an exponential tilt with $\tau=1$, a structured latent intervention with
$(\kappa_1,\kappa_2)=(2,2)$, and a variance-matched Student-$t$ replacement
with $\nu_{\mathrm{df}}=3$. The first two interventions perturb the treated
distribution while remaining within the fitted Gaussian exponential family,
whereas the latter moves the treated response distribution outside the fitted
Gaussian family.

For the interventions that remain inside the fitted exponential family, the
observed statistic is computed using the plug-in KL divergence between the
fitted treated and counterfactual natural parameters. For the Student-$t$
replacement, we additionally approximate the KL divergence by Monte Carlo using
the known Student-$t$ density and the fitted Gaussian counterfactual density,
\begin{align}
\mathrm{KL}(p_t\|p_{\mathrm{cf}})
=\EE{Y\sim p_t}{\log p_t(Y)-\log p_{\mathrm{cf}}(Y)},
\end{align}
where $p_t$ denotes the Student-$t$ replacement distribution and
$p_{\mathrm{cf}}$ denotes the Gaussian counterfactual distribution induced by
EFSC. This post-treatment divergence is approximated using $20{,}000$ Monte
Carlo draws per treated post-treatment cell and combined with the same Gaussian
plug-in pretreatment discrepancy to form the displayed oracle
$\Delta_{\mathrm{KL}}$. Because this MC quantity is not computed using the same
estimation rule as the placebo statistics, it is shown only as an oracle
diagnostic and is not assigned a placebo $p$-value. The formal placebo test is
therefore calibrated to the fitted Gaussian observation model, while the oracle
diagnostic illustrates the additional discrepancy induced by changes in tail
shape.

All placebo fits use $S=2$ Monte Carlo samples per ELBO evaluation and $500$
optimization epochs. The Gaussian intervention experiment uses learning rate
$0.1$. For the one-parameter experiments, the learning rates are $0.05$ for
Bernoulli and Gaussian with known variance, $0.1$ for Poisson, $0.005$ for
exponential and Laplace, and $0.002$ for Chi-Squared.

\subsection{Details on the Medicaid Expansion Experiment}
\label{app:medicaid_real_data}
This appendix provides additional details for the Medicaid expansion application
in Section~\ref{subsec:medicaid}. We use the same EFSC principles as in the synthetic experiments, but in this case the panel observations are survey-weighted multinomial counts rather than simulated exponential-family datasets. Our goal is to estimate how Medicaid expansion changed the full distribution of health insurance coverage among low-income adults across the United States.

\subsubsection{ACS/IPUMS Sample Construction}
We use individual-level American Community Survey (ACS) microdata obtained
through IPUMS USA. The underlying ACS data were provided by the United States Census Bureau. The analysis sample is restricted to adults aged 19--64 whose family income is at or below 138\% of the poverty threshold.\footnote{The Affordable Care Act establishes a threshold of 133\% of the federal poverty level, while the five-percentage-point income disregard under the MAGI eligibility rules produces an effective threshold of 138\%; see
\url{https://www.medicaid.gov/faq/2020-04-13/92591}.}
The income restriction is chosen to focus on the population most directly
exposed to the Medicaid expansion eligibility margin. For each respondent, we
construct a mutually exclusive insurance category from the available health
insurance indicators:
\begin{align*}
\mathcal C_{\mathrm{ins}}=\{\text{Uninsured},
\text{Medicaid},\text{Employer},\text{Private},\text{Other}\},    
\end{align*}
such that
\begin{align*}
c=1:\text{Uninsured},\;c=2:\text{Medicaid},\;
c=3:\text{Employer},\;c=4:\text{Private},\;
c=5:\text{Other}.    
\end{align*}
When multiple forms of coverage are reported, the respondent is assigned to a
single category using the following precedence, from highest to lowest: uninsured, Medicaid, employer, private, and other coverage. Each respondent therefore contributes one categorical outcome. Here, ``employer'' denotes employer-sponsored insurance and ``private'' denotes direct-purchase private insurance.

Let $\mathcal I_{ij}$ denote the set of respondents in the analysis sample
for state $i$ and year $j$. For each $k\in\mathcal I_{ij}$, let $w_{ijk}$
denote the ACS person-level survey weight and let
$z_{ijk}\in\{1,\ldots,C\}$ denote the respondent's insurance category, with
$C=5$. The raw state-year category counts are constructed as survey-weighted
totals,
\begin{align}
y^{\mathrm{raw}}_{ijc}
=
\sum_{k\in\mathcal I_{ij}}w_{ijk}\,
\mathbb{I}\{z_{ijk}=c\},
\quad c=1,\ldots,C.
\end{align}
After restricting the processed data to 2008--2019, the analysis panel
contains $N=51$ units, corresponding to the 50 states and the District of
Columbia, and $T=12$ years. Thus each panel cell $(i,j)$ is a five-category vector of counts
\begin{align}
\mathbf y^{\mathrm{raw}}_{ij}
=(y^{\mathrm{raw}}_{ij1},\ldots,y^{\mathrm{raw}}_{ijC}).    
\end{align}
Treatment timing follows the state-level Medicaid expansion implementation
dates reported by KFF State Health Facts\footnote{KFF State Health Facts,
``Status of State Action on the Medicaid Expansion Decision,''
\url{https://www.kff.org/affordable-care-act/state-indicator/state-activity-around-expanding-medicaid-under-the-affordable-care-act/}.}.
Let $A_i\in\{1,\ldots,T\}$ denote the time index corresponding to the first
calendar year in which state $i$ is treated within the 2008--2019 analysis
window, where $j=1$ corresponds to 2008. If state $i$ does not expand during
the window, we set $A_i=\infty$. The treated-post set is
\begin{align}
\Omega_{\mathrm{post}}=
\{(i,j): A_i<\infty,\; j\ge A_i\},
\end{align}
and the untreated, or donor, set of cells is
\begin{align}
\Omega_{\mathrm{unt}}
=\{(i,j): A_i=\infty\}
\cup\{(i,j): A_i<\infty,\; j<A_i\}.
\end{align}
States that expand after 2019 are therefore untreated for this analysis window.
This convention is important because the counterfactual model is never trained
on treated post-expansion cells. Also, since the ACS outcomes are measured annually, expansions implemented from
September onward are coded as beginning in the following calendar year. Within
the 2008--2019 analysis window, this convention affects only Alaska, whose
September 2015 expansion is coded as beginning in 2016.

\subsubsection{Effective Multinomial Counts}
The ACS person weights make the raw totals $y^{\mathrm{raw}}_{ijc}$
population-representative, but these weighted totals are not literal independent
sample sizes. If the raw weighted totals were used directly as multinomial
counts, large-population states would dominate the likelihood primarily because
their survey-weighted totals are larger. To avoid this scale distortion, we
preserve the survey-weighted category proportions but rescale every state-year
cell to a common effective multinomial size.

Define the state-year survey-weighted category proportions
\begin{align}
\hat p^{\mathrm{raw}}_{ijc}=
\frac{y^{\mathrm{raw}}_{ijc}}
{\sum_{\ell=1}^{C}y^{\mathrm{raw}}_{ij\ell}},\quad c=1,\ldots,C.
\end{align}
For estimation, we use an effective cell size \(m_0=1,000\) and construct integer
effective counts
\begin{align}
y_{ijc}\approx m_0\,\hat p^{\mathrm{raw}}_{ijc},
\quad \sum_{c=1}^{C}y_{ijc}=m_0.
\end{align}
In the implementation, this is done by taking the floor of
$m_0\hat p^{\mathrm{raw}}_{ijc}$ and assigning the remaining counts to the
categories with the largest survey-weighted probabilities so that the total is exactly $m_0$
in every state-year cell. This produces the multinomial panel used by EFSC:
\begin{align}
\mathbf y_{ij}=
(y_{ij1},\ldots,y_{ijC}),
\quad m_{ij}=\sum_{c=1}^{C}y_{ijc}=m_0.    
\end{align}
The purpose of this rescaling is not to alter the empirical insurance
composition, but to put all state-year likelihood contributions on a comparable
scale.

\subsubsection{Multinomial EFSC Model}
\begin{figure}[t]
    \centering
    \includegraphics[width=0.9\linewidth]{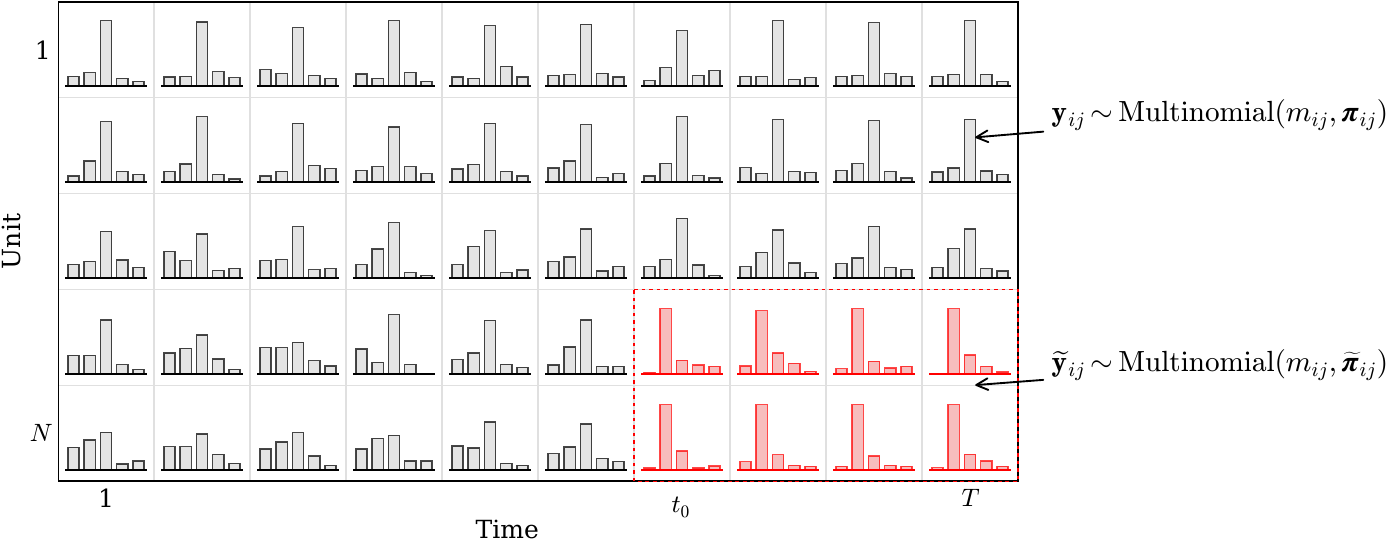}
    \caption{Panel of multinomial responses with $N$ units and $T$ time periods,
where the last two units are affected by a policy that shifts the category
distribution from time $t_0$ onward. This is an idealized scenario: the actual
panel for the Medicaid application has more units than time periods, and the
intervention times are staggered.}\label{fig:data_panel_and_tilt_multinomial}
\end{figure}
For each state $i$ and year $j$, we model the effective count vector as
\begin{align}
\mathbf y_{ij}\mid m_0,\boldsymbol{\pi}_{ij}
\sim
\operatorname{Multinomial}(m_0,\boldsymbol{\pi}_{ij}),
\quad
\boldsymbol{\pi}_{ij}\in\Delta^{C-1}.
\end{align}
We use the fifth category, ``Other'', as the reference category. The natural
parameters are the logits
\begin{align}
\eta_{ij}[c]=\log\frac{\pi_{ijc}}{\pi_{ijC}},
\quad c=1,\ldots,C-1,
\end{align}
with inverse maps
\begin{align}
\pi_{ijc}
&=
\frac{\exp(\eta_{ij}[c])}
{1+\sum_{\ell=1}^{C-1}\exp(\eta_{ij}[\ell])},
\quad c=1,\ldots,C-1,
\\
\pi_{ijC}
&=\frac{1}
{1+\sum_{\ell=1}^{C-1}\exp(\eta_{ij}[\ell])}.
\end{align}

For each non-reference logit $c=1,\ldots,C-1$, EFSC uses the PMF
\begin{align}
\eta_{ij}[c]=\alpha_i[c]+\gamma_j[c]+\boldsymbol{\theta}_i[c]^\top
\boldsymbol{\beta}_j[c].
\end{align}
The parameters $\alpha_i[c]$ and $\gamma_j[c]$ are state and year effects
for logit $c$, respectively, while
$\boldsymbol{\theta}_i[c]^\top\boldsymbol{\beta}_j[c]$ captures residual
state-year dependence through a rank-$r$ latent factorization. In this
experiment we use a separate factorization for each of the $C-1=4$ logits, and fix $r=1$.
The variational approximation and optimization follow
Appendix~\ref{app:bbvi_implementation}, using independent
$\mathcal N(0,2^2)$ priors. The principal fit uses $S=2$, while the placebo
generator and placebo fits use $S=4$.

Let $q_{\mathrm{ctrl}}(\Theta)$ denote the variational posterior obtained by
fitting EFSC on the untreated set $\Omega_{\mathrm{unt}}=\{(i,j):A_i=\infty\text{ or }j<A_i\}$. The posterior-mean
counterfactual natural parameters are
\begin{align}
\hat\eta^{\mathrm{ctrl}}_{ij}[c]=\hat\alpha_i[c]+\hat\gamma_j[c]
+\hat{\boldsymbol{\theta}}_i[c]^\top\hat{\boldsymbol{\beta}}_j[c],
\end{align}
with corresponding probability vector
$\hat{\boldsymbol{\pi}}^{\mathrm{ctrl}}_{ij}$. For the treated post-expansion
block, we also fit a treated model using the cells in $\Omega_{\mathrm{post}}=\{(i,j):A_i<\infty\text{ and }j\ge A_i\}$, yielding $\hat{\boldsymbol{\pi}}^{\mathrm{treat}}_{ij}$ for $(i,j)\in\Omega_{\mathrm{post}}$. The average probability-shift in Table~\ref{tab:acs_probability_shifts} in Section~\ref{subsec:medicaid} reports
\begin{align}
\frac{1}{|\Omega_{\mathrm{post}}|}
\sum_{(i,j)\in\Omega_{\mathrm{post}}}
\hat{\pi}^{\mathrm{ctrl}}_{ijc},
\quad
\frac{1}{|\Omega_{\mathrm{post}}|}
\sum_{(i,j)\in\Omega_{\mathrm{post}}}
\hat{\pi}^{\mathrm{treat}}_{ijc},
\end{align}
and their difference, for each insurance category $c$.

\subsubsection{State-Level Distributional Summaries}
Because adoption is staggered, each treated state has its own post-treatment period
\begin{align}
\mathcal Q_i=
\{j: j\ge A_i\},
\quad A_i<\infty.
\end{align}
For each treated state \(i\), the average Medicaid probability effect is
\begin{align}
\Delta^{\mathrm{Medicaid}}_i=
\frac{1}{|\mathcal Q_i|}
\sum_{j\in\mathcal Q_i}
\left(\hat{\pi}^{\mathrm{treat}}_{ij,\mathrm{Medicaid}}-
\hat{\pi}^{\mathrm{ctrl}}_{ij,\mathrm{Medicaid}}
\right).
\end{align}
The corresponding distributional effect is measured by the probability-vector
KL divergence
\begin{align}
\operatorname{KL}
\left(
\operatorname{Cat}
\bigl(\hat{\boldsymbol{\pi}}^{\mathrm{treat}}_{ij}\bigr)
\,\middle\|\,
\operatorname{Cat}
\bigl(\hat{\boldsymbol{\pi}}^{\mathrm{ctrl}}_{ij}\bigr)
\right)
&=
\sum_{c=1}^{C}
\hat{\pi}^{\mathrm{treat}}_{ijc}
\log\frac{
\hat{\pi}^{\mathrm{treat}}_{ijc}
}{
\hat{\pi}^{\mathrm{ctrl}}_{ijc}
}
\notag\\
&=
\frac{1}{m_0}
\operatorname{KL}
\left(
\operatorname{Multinom}
\bigl(m_0,\hat{\boldsymbol{\pi}}^{\mathrm{treat}}_{ij}\bigr)
\,\middle\|\,
\operatorname{Multinom}
\bigl(m_0,\hat{\boldsymbol{\pi}}^{\mathrm{ctrl}}_{ij}\bigr)
\right).
\label{eq:kl_vector_distributional}
\end{align}
We report this KL on the probability-vector scale. That is, we do not multiply
by the effective cell size $m_0$. This convention makes the reported ECD a
measure of distributional change in category probabilities rather than a
multinomial log-likelihood contrast scaled by the chosen effective count.

For the heatmap in Figure~\ref{fig:medicaid_heatmap}, we further compute category-specific shifts
\begin{align}
\Delta_{ic}=\frac{1}{|\mathcal Q_i|}
\sum_{j\in\mathcal Q_i}
\left(\hat{\pi}^{\mathrm{treat}}_{ijc}
-\hat{\pi}^{\mathrm{ctrl}}_{ijc}\right),
\quad c=1,\ldots,C.
\end{align}
States are sorted by $\mathrm{ECD}_i$, so that the heatmap displays which
category shifts contribute to the largest distributional effects.

Figure~\ref{fig:medicaid_all_category_probability_paths} provides a
category-level view of the fitted trajectories for two selected states. The
Medicaid path captures the most interesting effect of the policy, but the remaining panels
show how increases in Medicaid coverage are accompanied by changes in
e.g. uninsured and employer-sponsored. This motivates the
use of the probability-vector KL in the state-level ECD in Figures \ref{fig:medicaid_state_ecd_scatter} and \ref{fig:medicaid_heatmap}: the empirical
effect is a redistribution across the insurance composition, not only a change
in a single category.

\subsubsection{\texorpdfstring{Observed $\Delta_{\mathrm{KL}}$ Statistic}{Observed KL Statistic}}
The placebo analysis in Section~\ref{subsec:medicaid} is based on the change in average ECD
from the pre-expansion block to the post-expansion block. For the observed
Medicaid assignment, let
\begin{align}
\Omega_{\mathrm{pre}}=\{(i,j): A_i<\infty,\; j<A_i\},
\quad\Omega_{\mathrm{post}}=\{(i,j): A_i<\infty,\; j\ge A_i\}.
\end{align}
Let \(A_{\min}=\min\{A_i:A_i<\infty\}\), and define the pre-target training set
\begin{align}
\Omega_{\mathrm{pre,tgt}}=
\Omega_{\mathrm{pre}}
\cup\left[\{i:A_i=\infty\}\times\{1,\ldots,A_{\min}-1\}\right].
\end{align}
We fit four models for the observed statistic:
\begin{enumerate}
\item a post-target model fit on \(\Omega_{\mathrm{unt}}\);
\item a post-observed model fit on \(\Omega_{\mathrm{post}}\);
\item a pre-target model fit on \(\Omega_{\mathrm{pre,tgt}}\);
\item a pre-observed model fit on \(\Omega_{\mathrm{pre}}\).
\end{enumerate}
Let the resulting posterior-mean probability vectors be
$\hat{\boldsymbol{\pi}}^{\mathrm{post,tgt}}_{ij}$,
$\hat{\boldsymbol{\pi}}^{\mathrm{post,obs}}_{ij}$,
$\hat{\boldsymbol{\pi}}^{\mathrm{pre,tgt}}_{ij}$, and
$\hat{\boldsymbol{\pi}}^{\mathrm{pre,obs}}_{ij}$, respectively. We compute
\begin{align}
\mathrm{ECD}^{\mathrm{pre}}_{\mathrm{obs}}
=\frac{1}{|\Omega_{\mathrm{pre}}|}
\sum_{(i,j)\in\Omega_{\mathrm{pre}}}
\operatorname{KL}
\left(\hat{\boldsymbol{\pi}}^{\mathrm{pre,obs}}_{ij}
\,\|\,\hat{\boldsymbol{\pi}}^{\mathrm{pre,tgt}}_{ij}\right),
\end{align}
and
\begin{align}
\mathrm{ECD}^{\mathrm{post}}_{\mathrm{obs}}=
\frac{1}{|\Omega_{\mathrm{post}}|}
\sum_{(i,j)\in\Omega_{\mathrm{post}}}
\operatorname{KL}
\left(\hat{\boldsymbol{\pi}}^{\mathrm{post,obs}}_{ij}
\,\|\,\hat{\boldsymbol{\pi}}^{\mathrm{post,tgt}}_{ij}\right).
\end{align}
The observed placebo-test statistic is then $\Delta_{\mathrm{KL}}^{\mathrm{obs}}=
\mathrm{ECD}^{\mathrm{post}}_{\mathrm{obs}}
-\mathrm{ECD}^{\mathrm{pre}}_{\mathrm{obs}}.$ The pre term serves as a baseline discrepancy for treated states before expansion, while the post term measures the corresponding discrepancy after
expansion. Both terms are computed using the full five-category insurance
probability vector.

\begin{figure}
    \centering
    \includegraphics[width=0.9\linewidth]{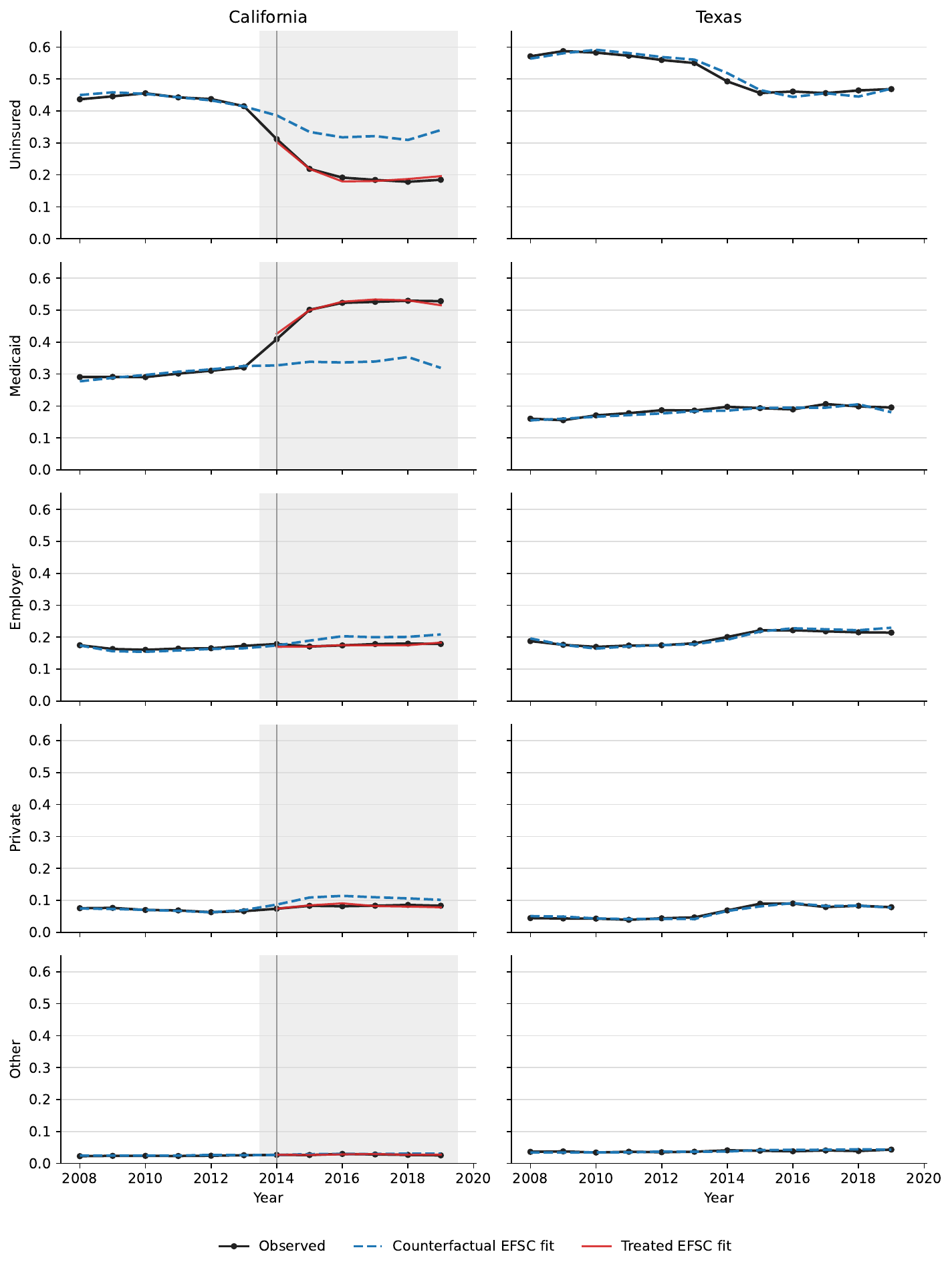}
    \caption{Estimated insurance-category probability paths for two selected states.
The figure compares observed survey-weighted category proportions with
posterior-mean fitted probabilities under the treated and counterfactual
multinomial EFSC models. Unlike Figure~\ref{fig:medicaid_probability_paths},
which focuses only on Medicaid coverage, here we illustrate how the fitted 
model reallocates probability mass across the full coverage distribution.}
    \label{fig:medicaid_all_category_probability_paths}
\end{figure}
\newpage
\subsubsection{Posterior-Predictive Placebo Procedure}
\label{app:medicaid_posterior_predictive_placebo}

\begin{algorithm}[t]
\SetAlgoNoLine
\DontPrintSemicolon
\caption{Posterior-predictive EFSC placebo test for Medicaid expansion}
\label{alg:medicaid-pp-placebo}
\KwIn{ACS/IPUMS panel $\mby$, treatment-time collection $(A_i)_{i\in\mathcal T}$, untreated cells $\Omega_{\mathrm{unt}}$, number of
synthetic units $N_{\mathrm{syn}}$, number of joint placebo replications $B$}
\KwOut{Observed statistic $\Delta_{\mathrm{KL}}^{\mathrm{obs}}$,
posterior-predictive placebo statistics
$\{\Delta_{\mathrm{KL}}^{(b)}\}_{b=1}^B$, empirical $p$-value}

Let $K=|\mathcal T|$ and
$\mathcal A=(A_i)_{i\in\mathcal T}$ denote the
indexed collection of observed treatment-time indices\;

Fit an untreated EFSC generator on $\mby_{\Omega_{\mathrm{unt}}}$ to get a
posterior approximation $q_{\mathrm{gen}}(\Theta)$\;

Compute the observed statistic $\Delta_{\mathrm{KL}}^{\mathrm{obs}}$ using the
observed staggered treatment schedule and four no-leakage EFSC fits for the
pre- and post-treatment ECDs\;

\For{$b=1,\ldots,B$}{
Draw one untreated synthetic panel $\mby_{\mathrm{syn}}^{(b)}$ from the fitted predictive generator\;

Draw a synthetic placebo treated set
$\mathcal T_{\mathrm{syn}}^{(b)}
\subseteq\{1,\ldots,N_{\mathrm{syn}}\}$ with
$|\mathcal T_{\mathrm{syn}}^{(b)}|=K$\;

Randomly permute the treatment-index multiset $\mathcal A$ across units
in $\mathcal T_{\mathrm{syn}}^{(b)}$\;

Construct the placebo pretreatment block
$\Omega_{\mathrm{pre}}^{(b)}$, placebo post-treatment block
$\Omega_{\mathrm{post}}^{(b)}$, and no-leakage post-target training set
\[
\Omega_{\mathrm{train}}^{(b)}=
(\{1,\ldots,N_{\mathrm{syn}}\}\times\{1,\ldots,T\})
\setminus\Omega_{\mathrm{post}}^{(b)}
\]
Construct the pre-target training set
\[
\Omega_{\mathrm{pre,tgt}}^{(b)}
=\Omega_{\mathrm{pre}}^{(b)}\cup
\left[\left(
\{1,\ldots,N_{\mathrm{syn}}\}
\setminus\mathcal T_{\mathrm{syn}}^{(b)}
\right)\times\{1,\ldots,\min(\mathcal A)-1\}
\right]
\]

Fit the post-target model on $\Omega_{\mathrm{train}}^{(b)}$ and the
post-observed model on $\Omega_{\mathrm{post}}^{(b)}$\;

Fit the pre-target model on $\Omega_{\mathrm{pre,tgt}}^{(b)}$ and the
pre-observed model on $\Omega_{\mathrm{pre}}^{(b)}$\;

Compute
$\Delta_{\mathrm{KL}}^{(b)}
=\mathrm{ECD}^{\mathrm{post},(b)}
-\mathrm{ECD}^{\mathrm{pre},(b)}$
}

Compute the right-tail empirical probability:
$\hat p=
({1+\sum_{b=1}^{B}
\mathbb{I}\{\Delta_{\mathrm{KL}}^{(b)}
\geq \Delta_{\mathrm{KL}}^{\mathrm{obs}}
\}
})/(B+1)$
\end{algorithm}

In the synthetic experiments, placebo assignments are obtained by selecting
control groups with the same cardinality as the treated group. In the Medicaid application, 
this direct procedure is infeasible because the number of treated expansion states is larger than the number of states remaining untreated throughout the 2008--2019 analysis window. We therefore construct a posterior-predictive approximation to the same placebo null. The idea is to learn an untreated latent factorization from cells not exposed to treatment, generate untreated synthetic state panels from that factorization, and then apply the same-cardinality no-leakage placebo logic to those synthetic panels.

Let $K=|\{i:A_i<\infty\}|$ denote the number of observed treated states, and let
\begin{align*}
\mathcal A=(A_{i_1},\ldots,A_{i_K})
\end{align*}
be the multiset of first treated-year indices among expansion states, where
$A_{i_k}\in\{1,\ldots,T\}$ for $k=1,\ldots,K$. The placebo procedure uses
synthetic panels with $N_{\mathrm{syn}}\ge K$ units and the same calendar
years as the ACS panel.

First, we fit an untreated generator on $\Omega_{\mathrm{unt}}$, obtaining a
variational posterior $q_{\mathrm{gen}}(\Theta)$. From this fitted generator,
we construct synthetic untreated panels as follows. 

For each synthetic state $i=1,\ldots,N_{\mathrm{syn}}$:
\begin{enumerate}
    \item Draw a source-state index $S_i \sim \operatorname{Unif}\{1,\ldots,N\},$ independently and with replacement.
    \item Set the synthetic unit-level effects equal to their variational posterior means, 
    \begin{align}
      \alpha^{\mathrm{syn}}_i[c]=\EE{q_{\mathrm{gen}}}{\alpha_{S_i}[c]},\;
    \boldsymbol{\theta}^{\mathrm{syn}}_i[c]=\EE{q_{\mathrm{gen}}}{\boldsymbol{\theta}_{S_i}[c]}.  
    \end{align}   
    \item For every year $j=1,\ldots,T$, combine these unit-level effects with the posterior-mean year effects and time factors to obtain the synthetic logits
    \begin{align}
    \eta^{\mathrm{syn}}_{ij}[c]=\alpha^{\mathrm{syn}}_{i}[c]
+\gamma_j[c]+
\boldsymbol{\theta}^{\mathrm{syn}}_{i}[c]^\top\boldsymbol{\beta}_{j}[c],
\; c=1,\ldots,C-1,    
    \end{align}
\end{enumerate}
where $\gamma_j[c]$ and $\boldsymbol{\beta}_j[c]$ denote their variational
posterior means. These logits are transformed into probabilities
$\boldsymbol{\pi}^{\mathrm{syn}}_{ij}$ through the inverse map, and
synthetic effective counts are drawn as
\begin{align}
\mathbf y^{\mathrm{syn}}_{ij}
\sim\operatorname{Multinomial}
(m_0,\boldsymbol{\pi}^{\mathrm{syn}}_{ij}),
\end{align}
for every synthetic state-year cell $(i,j)$. Thus every synthetic cell is generated under the learned untreated factorization.

For every placebo replicate \(b=1,\ldots,B\), we independently draw one
synthetic untreated panel and a subset
$\mathcal T_{\mathrm{syn}}^{(b)}$ of $K$ synthetic states without replacement.
We then randomly permute the observed treatment-index multiset $\mathcal A$
across the selected synthetic states. This creates a staggered placebo
assignment with the same number of treated states and the same distribution of
adoption years as the observed Medicaid assignment. Let
$\Omega^{(b)}_{\mathrm{pre}}$ and $\Omega^{(b)}_{\mathrm{post}}$ denote the resulting placebo pre and post blocks, and let
\begin{align}
\Omega^{(b)}_{\mathrm{train}}=
\left(\{1,\ldots,N_{\mathrm{syn}}\}\times\{1,\ldots,T\}\right)
\setminus\Omega^{(b)}_{\mathrm{post}}.    
\end{align}
For the pre-target fit, we additionally define
\begin{align}
\Omega^{(b)}_{\mathrm{pre,tgt}}=
\Omega^{(b)}_{\mathrm{pre}}\cup
\left[\left(\{1,\ldots,N_{\mathrm{syn}}\}
\setminus\mathcal T_{\mathrm{syn}}^{(b)}
\right)\times\{1,\ldots,\min(\mathcal A)-1\}
\right].    
\end{align}
These masks enforce the no-leakage rule: the placebo post-treatment block is
not used to train either placebo target model.

For each placebo replicate, we fit four EFSC models on the synthetic panel:
\begin{enumerate}
\item a placebo post-target model on $\Omega^{(b)}_{\mathrm{train}}$;
\item a placebo post-observed model on $\Omega^{(b)}_{\mathrm{post}}$;
\item a placebo pre-target model on $\Omega^{(b)}_{\mathrm{pre,tgt}}$;
\item a placebo pre-observed model on $\Omega^{(b)}_{\mathrm{pre}}$.
\end{enumerate}
The corresponding placebo statistic is
\begin{align}
\Delta_{\mathrm{KL}}^{(b)}=\mathrm{ECD}^{\mathrm{post},(b)}-
\mathrm{ECD}^{\mathrm{pre},(b)},
\end{align}
where both terms are computed using the same unscaled probability-vector KL
definition from Eq.~\eqref{eq:kl_vector_distributional}.

In the final implementation, we use $N_{\mathrm{syn}}=80$ synthetic units and
$B=5{,}000$ joint placebo replications. We independently draw one
posterior-predictive panel and one placebo assignment for each placebo
replicate $b=1,\ldots,B$. For each joint panel-assignment draw, we compute and
directly retain the statistic $\Delta_{\mathrm{KL}}^{(b)}$, without averaging
across panels or assignments. This construction preserves both predictive panel variation and
placebo-assignment variation. The posterior-predictive empirical placebo $p$-value is
computed with the standard finite-sample correction shown in Algorithm~\ref{alg:medicaid-pp-placebo}.

This reference distribution should be interpreted as an approximate placebo
null induced by the learned untreated factorization. It is not a literal
permutation distribution over the finite set of observed control states.
Nevertheless, it preserves the central ingredients of our
placebo design: same-cardinality treated blocks, no use of target post-treatment
cells in either target fit, and comparison of the observed statistic to
assignments generated under a no-treatment model.

\bibliography{references}

@article{abadie2003economic,
  title={The economic costs of conflict: A case study of the {Basque Country}},
  author={Abadie, Alberto and Gardeazabal, Javier},
  journal={American Economic Review},
  volume={93},
  number={1},
  pages={113--132},
  year={2003},
  publisher={American Economic Association}
}

@article{abadie2010synthetic,
  title={Synthetic control methods for comparative case studies: Estimating the effect of {California’s} tobacco control program},
  author={Abadie, Alberto and Diamond, Alexis and Hainmueller, Jens},
  journal={Journal of the American Statistical Association},
  volume={105},
  number={490},
  pages={493--505},
  year={2010},
  publisher={Taylor \& Francis}
}

@article{abadie2021using,
Author = {Abadie, Alberto},
Title = {Using Synthetic Controls: Feasibility, Data Requirements, and Methodological Aspects},
Journal = {Journal of Economic Literature},
Volume = {59},
Number = {2},
Year = {2021},
Pages = {391–425}}

@article{abadie2021penalized,
title={A Penalized Synthetic Control Estimator for Disaggregated Data},
author={Abadie, Alberto and L'Hour, Jeremy},
journal={Journal of the American Statistical Association},
volume={116},
number={536},
pages={1817--1834},
year={2021}
}

@article{athey2021matrix,
  title={Matrix completion methods for causal panel data models},
  author={Athey, Susan and Bayati, Mohsen and Doudchenko, Nikolay and Imbens, Guido and Khosravi, Khashayar},
  journal={Journal of the American Statistical Association},
  volume={116},
  number={536},
  pages={1716--1730},
  year={2021},
  publisher={Taylor \& Francis}
}

@article{blei2017variational,
  title={Variational inference: A review for statisticians},
  author={Blei, David M and Kucukelbir, Alp and McAuliffe, Jon D},
  journal={Journal of the American Statistical Association},
  volume={112},
  number={518},
  pages={859--877},
  year={2017},
  publisher={Taylor \& Francis}
}

@article{brodersen2015inferring,
title={Inferring Causal Impact Using Bayesian Structural Time-Series Models},
author={Brodersen, Kay H and Gallusser, Fabian and Koehler, Jim and Remy, Nicolas and Scott, Steven L},
journal={Annals of Applied Statistics},
volume={9},
number={1},
pages={247--274},
year={2015}
}

@book{hernan2020causal,
  author    = {Hern{\'a}n, Miguel A and Robins, James M},
  title     = {Causal Inference: What If},
  year      = {2020},
  publisher = {Chapman \& Hall/CRC}
}

@article{diaz2012population,
title={Population Intervention Causal Effects Based on Stochastic Interventions},
author={Diaz, Ivan and van der Laan, Mark J},
journal={Biometrics},
volume={68},
number={2},
pages={541--549},
year={2012}
}

@article{diaz2019causal,
title={Causal Mediation Analysis for Stochastic Interventions},
author={Diaz, Ivan and Hejazi, Nima S},
journal={Journal of the Royal Statistical Society: Series B},
volume={82},
number={3},
pages={661--683},
year={2020}
}

@book{efron2022exponential,
  title={Exponential Families in Theory and Practice},
  author={Efron, Bradley},
  year={2022},
  publisher={Cambridge University Press}
}

@book{good2005permutation,
  title     = {Permutation, Parametric, and Bootstrap Tests of Hypotheses},
  author    = {Good, Phillip I},
  edition   = {3rd},
  year      = {2005},
  publisher = {Springer-Verlag}
}

@article{jetsupphasuk2025difference,
    title={Difference-in-differences with stochastic policy shifts of a continuous treatment}, 
      author={Michael Jetsupphasuk and Chenwei Fang and Didong Li and Michael G Hudgens},
      year={2025},
      journal={arXiv:2512.00296}
}

@article{kennedy2019nonparametric,
title={Nonparametric Causal Effects Based on Incremental Propensity Score Interventions},
author={Kennedy, Edward H},
journal={Journal of the American Statistical Association},
volume={114},
number={526},
pages={645--656},
year={2019}
}

@article{klinenberg2024timevarying,
author = {Danny Klinenberg},
title = {Synthetic Control with Time Varying Coefficients: A State Space Approach with Bayesian Shrinkage},
journal = {Journal of Business \& Economic Statistics},
volume = {41},
number = {4},
pages = {1065--1076},
year = {2023}}

@inproceedings{nazaret2024misspecification,
  title={On the Misspecification of Linear Assumptions in Synthetic Controls},
  author={Nazaret, Achille and Shi, Claudia and Blei, David M},
  booktitle={International Conference on Artificial Intelligence and Statistics (AISTATS)},
  pages={3790--3798},
  year={2024},
  organization={PMLR}
}

@misc{patientprotection2010,
  author = {{U.S. Congress}},
  title        = {Patient Protection and Affordable Care Act},
  year         = {2010},
  note         = {Public Law 111-148, 124 Stat. 119},
  url = {https://www.congress.gov/111/plaws/publ148/PLAW-111publ148.pdf}
}

@inproceedings{ranganath2014black,
  title={Black box variational inference},
  author={Ranganath, Rajesh and Gerrish, Sean and Blei, David M},
  booktitle={International Conference on Artificial intelligence and statistics (AISTATS)},
  pages={814--822},
  year={2014},
  organization={PMLR}
}

@inproceedings{rho2024timeaware,
  title={Time-Aware Synthetic Control},
  author={Rho, Saeyoung and Illick, Cyrus and Narasipura, Samhitha and Abadie, Alberto and Hsu, Daniel and Misra, Vishal},
  booktitle={International Conference on Artificial Intelligence and Statistics (AISTATS)},
  year={2026},
  organization={PMLR}
}

@InProceedings{rho2025cluster,
  title = {{ClusterSC}: Advancing Synthetic Control with Donor Selection},
  author = {Rho, Saeyoung and Tang, Andrew and Bergam, Noah and Cummings, Rachel and Misra, Vishal},
  booktitle = {International Conference on Artificial Intelligence and Statistics (AISTATS)},
  pages ={109-117},
  year = {2025},
  publisher = {PMLR}
}

@misc{ruggles2025ipumsusa,
  author       = {Steven Ruggles and Sarah Flood and Matthew Sobek and
                  Daniel Backman and Grace Cooper and Julia A. Rivera Drew and
                  Stephanie Richards and Renae Rodgers and Jonathan Schroeder and
                  Kari C. W. Williams},
  title        = {{IPUMS USA}: Version 16.0 [dataset]},
  year         = {2025},
  address      = {Minneapolis, MN},
  publisher    = {IPUMS},
  doi          = {10.18128/D010.V16.0},
  url          = {https://doi.org/10.18128/D010.V16.0}
}

@article{rubin2005causal,
  title={Causal inference using potential outcomes: Design, modeling, decisions},
  author={Rubin, Donald B},
  journal={Journal of the American Statistical Association},
  volume={100},
  number={469},
  pages={322--331},
  year={2005},
  publisher={Taylor \& Francis}
}

@inproceedings{mao2024learning,
  title={Learning Identifiable Factorized Causal Representations of Cellular Responses},
  author={Mao, Haiyi and Lopez, Romain and Liu, Kai and Huetter, Jan-Christian and Richmond, David and Benos, Panayiotis V. and Qiu, Lin},
  booktitle={Advances in Neural Information Processing Systems (NeurIPS)},
  volume={37},
  year={2024}
}

@inproceedings{salakhutdinov2007pmf,
  title={Probabilistic matrix factorization},
  author={Salakhutdinov, Ruslan and Mnih, Andriy},
  booktitle={Advances in Neural Information Processing Systems (NeurIPS)},
  volume={20},
  year={2007}
}

@article{schindl2024incremental,
  title={Incremental effects for continuous exposures},
  author={Schindl, Kyle and Shen, Shuying and Kennedy, Edward H},
  journal={arXiv:2409.11967},
  year={2024}
}

@inproceedings{shao2022generalized,
  title={Generalized synthetic control method with state-space model},
  author={Shao, Junzhe and Yin, Mingzhang and Cai, Xiaoxuan and Valeri, Linda},
  booktitle={NeurIPS Workshop on Causal Machine Learning for Real-World Impact},
  year={2022}
}

@inproceedings{shi2022assumptions,
  title={On the assumptions of synthetic control methods},
  author={Shi, Claudia and Sridhar, Dhanya and Misra, Vishal and Blei, David M},
  booktitle={International Conference on Artificial Intelligence and Statistics (AISTATS)},
  pages={7163--7175},
  year={2022},
  organization={PMLR}
}

@book{strang2019linear,
  author    = {Strang, Gilbert},
  title     = {Linear Algebra and Learning from Data},
  year      = {2019},
  publisher = {Wellesley-Cambridge Press}
}

@article{viviano2023synthetic,
  title={Synthetic learner: model-free inference on treatments over time},
  author={Viviano, Davide and Bradic, Jelena},
  journal={Journal of Econometrics},
  volume={234},
  number={2},
  pages={691--713},
  year={2023}
}

@book{wager2024causal,
  author = {Wager, Stefan},
  title  = {Causal Inference: A Statistical Learning Approach},
  year   = {2025},
  note   = {Draft book},
  url    = {https://web.stanford.edu/~swager/causal_inf_book.pdf}
}

@article{wainwright2008graphical,
  title={Graphical models, exponential families, and variational inference},
  author={Wainwright, Martin J and Jordan, Michael I},
  journal={Foundations and Trends{\textregistered} in Machine Learning},
  volume={1},
  number={1-2},
  pages={1--305},
  year={2008},
  publisher={Now Publishers}
}

@article{weinstein2026hierarchical,
  title={Hierarchical causal models},
  author={Weinstein, Eli N and Blei, David M},
  journal={Journal of Machine Learning Research},
  volume={27},
  number={37},
  pages={1--73},
  year={2026}
}

\end{document}